\documentclass[11pt]{article}

\usepackage[a4paper, margin=1in]{geometry}
\usepackage[affil-it]{authblk}
\usepackage[utf8]{inputenc}
\usepackage{amsfonts}
\usepackage{amsmath}
\usepackage{amssymb}
\usepackage{amsthm}
\usepackage{bbm}
\usepackage{booktabs}
\usepackage{rotating}
\usepackage{multirow}
\usepackage{mathtools}
\usepackage{enumitem}
\usepackage{graphicx}
\usepackage{stmaryrd}
\usepackage{rotating}
\usepackage{tikz}
\usepackage{tikz-cd}
\usepackage{pgfplots}
\pgfplotsset{compat=1.18}
\usepackage{slashed}
\usepackage{setspace}
\usepackage{subcaption}
\usepackage{comment}
\usepackage[ruled,linesnumbered]{algorithm2e}
\usepackage[colorinlistoftodos]{todonotes}
\usepackage[most]{tcolorbox}
\usepackage{enumitem}
\usepackage{pifont}
\usepackage{appendix}
\usepackage{hyperref}
\hypersetup{
    colorlinks,
    linkcolor={red!50!black},
    citecolor={blue!50!black},
    urlcolor={blue!80!black}
}
\usepackage[capitalise]{cleveref}
\usepackage{xcolor}

\newcommand{\defeq}{=\vcentcolon}
\newcommand{\eqdef}{\vcentcolon=}
\newcommand{\mbb}[1]{\mathbb{#1}}

\newcommand{\mf}[1]{\mathfrak{#1}}
\newcommand{\mcA}{\mathcal{A}}
\newcommand{\mcB}{\mathcal{B}}
\newcommand{\mcC}{\mathcal{C}}

\newcommand{\mcE}{\mathcal{E}}

\newcommand{\mcL}{\mathcal{L}}

\newcommand{\mcN}{\mathcal{N}}

\newcommand{\mcV}{\mathcal{V}}

\newcommand{\lb}{\left(}
\newcommand{\rb}{\right)}
\newcommand{\mbbP}{\mathbb{P}}
\newcommand{\mbbE}{\mathbb{E}}
\newcommand{\mbbR}{\mathbb{R}}
\newcommand{\mbbC}{\mathbb{C}}
\newcommand{\mbbN}{\mathbb{N}}

\newcommand{\mix}{\textnormal{mix}}
\newcommand{\rank}{\textnormal{rank}}
\newcommand{\diag}{\textnormal{diag}}

\newcommand{\bkt}[1]{\left[#1\right]}

\def\eqcom#1{\overset{\textnormal{(\ref{#1})}}}

\DeclareMathOperator*{\argmax}{arg\,max}
\DeclareMathOperator*{\argmin}{arg\,min}

\DeclarePairedDelimiter\ceil{\lceil}{\rceil}
\DeclarePairedDelimiter\floor{\lfloor}{\rfloor}

\newlist{todolist}{itemize}{1}
\setlist[todolist]{label=$\square$}

\newcommand{\Cpp}{C\nolinebreak\hspace{-.05em}\raisebox{.4ex}{\tiny\bf +}\nolinebreak\hspace{-.10em}\raisebox{.4ex}{\tiny\bf +}}
\def\Cpp{{C\nolinebreak[4]\hspace{-.05em}\raisebox{.4ex}{\tiny\bf ++}}}
\newcommand{\OrderP}[1]{ O_{\mbbP}\left(#1\right)}
\newcommand{\orderP}[1]{ o_{\mbbP}\left(#1\right)}

\newcommand{\Vk}{\mathcal{V}}
\newcommand{\Vkhat}{\hat{\mathcal{V}}}
\newcommand{\Ktilde}{\Hat{K}^{\textnormal{spec}}}
\newcommand{\Khat}{\Hat{K}}
\newcommand{\KCI}[1]{\Hat{K}^{\textnormal{CI}}_t}

\crefname{assumption}{Assumption}{Assumptions}

\newtheorem{theorem}{Theorem}[section]
\newtheorem{corollary}[theorem]{Corollary}
\newtheorem{lemma}[theorem]{Lemma}
\newtheorem{proposition}[theorem]{Proposition}
\newtheorem{definition}[theorem]{Definition}
\newtheorem{assumption}{Assumption}
\newtheorem{example}[theorem]{Example}
\theoremstyle{definition}

\usepackage[acronym]{glossaries}

\newacronym{AA}{AA}{Applied Analysis}
\newacronym{AMI}{AMI}{Adjusted Mutual Information}
\newacronym{BMC}{BMC}{Block Markov Chain}
\newacronym[\glslongpluralkey={Block Markov Decision Processes}]{BMDP}{BMDP}{Block Markov Decision Process}
\newacronym{CASA}{CASA}{Centre for Analysis, Scientific Computing and Applications}
\newacronym{CAIC}{CAIC}{Consistent Akaike Information Criterion}
\newacronym{CPSRL}{CPSRL}{Clustered Posterior Sampling for Reinforcement Learning}
\newacronym{CQT}{CQT}{Coherence \& Quantum Technology}
\newacronym{CWI}{CWI}{National Research Institute for Mathematics and Computer Science}
\newacronym{CUDA}{CUDA}{Compute Unified Device Architecture}
\newacronym{DNA}{DNA}{Deoxyribonucleic acid}
\newacronym{DIAM}{DIAM}{Applied Mathematics}
\newacronym{ERRG}{ERRG}{Erd\"{o}s--R\'{e}nyi Random Graph}
\newacronym{EEMCS}{EEMCS}{Electrical Engineering, Mathematics \&{} Computer Science}
\newacronym{EURANDOM}{EURANDOM}{European Institute for Statistics, Probability, Stochastic Operations Research and their Applications}
\newacronym{EAISI}{EAISI}{Eindhoven Artificial Intelligence Systems Institute}
\newacronym{GPU}{GPU}{Graphics Processing Unit}
\newacronym{HMM}{HMM}{Hidden Markov Model}
\newacronym{INSY}{INSY}{Intelligent Systems}
\newacronym{LSBM}{LSBM}{Labeled Stochastic Block Model}
\newacronym{MAB}{MAB}{Multi-Armed Bandit}
\newacronym{MC}{MC}{Markov Chain}
\newacronym{MCS}{M\&{}CS}{Mathematics and Computer Science}
\newacronym[\glslongpluralkey={Markov Decision Processes}]{MDP}{MDP}{Markov Decision Process}
\newacronym{MI}{MI}{Mutual Information}
\newacronym{NAS}{NAS}{Network Architectures and Services}
\newacronym{PSRL}{PSRL}{Posterior Sampling for Reinforcement Learning}
\newacronym{PROB}{PROB}{Probability}
\newacronym[\glslongpluralkey={Partially Observable Markov Decision Processes}]{POMDP}{POMDP}{Partially Observable Markov Decision Process}
\newacronym{PyPI}{PyPI}{Python Package Index}
\newacronym{RL}{RL}{Reinforcement Learning}
\newacronym{SBM}{SBM}{Stochastic Block Model}
\newacronym{SPOR}{SPOR}{Statistics, Probability, and Operations Research}
\newacronym{SOR}{SOR}{Stochastic Operations Research}
\newacronym{ST}{ST}{Software Technology}
\newacronym{STO}{STO}{Stochastics}
\newacronym{STAT}{STAT}{Statistics}
\newacronym{SVD}{SVD}{Singular Value Decomposition}
\newacronym{TS}{TS}{Thompson Sampling}
\newacronym{TUDelft}{TU Delft}{Delft University of Technology}
\newacronym{TUe}{TU/e}{Eindhoven University of Technology}
\newacronym{UCB}{UCB}{Upper Confidence Bound}
\newacronym{UTQ}{UTQ}{University Teaching Qualification}
\newacronym{QCE}{Q\&{}CE}{Quantum \& Computer Engineering}
\newacronym{QED}{QED}{Quality--and--Efficiency--Driven}
\newacronym{WiCoS}{WiCoS}{Wireless Communication and Sensing}
\newacronym[longplural={Spatialized Random Graphs}]{SRG}{SRG}{Spatialized Random Graph}
\newacronym[longplural={Random Sequential Adsorption}]{RSA}{RSA}{Random Sequential Adsorption}

\let\oldparagraph\paragraph
\renewcommand{\paragraph}[1]{\oldparagraph{#1.}}

\excludecomment{todobox}

\title{Estimating the number of clusters of a Block Markov Chain}

\author{Thomas van Vuren\thanks{E-mail address: \texttt{t.p.a.v.vuren@tue.nl}; Corresponding author}}
\author{Thomas Cronk}
\author{Jaron Sanders}
\affil{Eindhoven University of Technology\\ Department of Mathematics \& Computer Science\\ The Netherlands}

\begin{document}

\maketitle

\begin{abstract}
    Clustering algorithms frequently require the number of clusters to be chosen in advance, but it is usually not clear how to do this.
To tackle this challenge when clustering within sequential data, we present a method for estimating the number of clusters when the data is a trajectory of a \glsentrylong{BMC}.
\glsentrylongpl{BMC} are \glsentrylongpl{MC} that exhibit a block structure in their transition matrix.
The method considers a matrix that counts the number of transitions between different states within the trajectory, and transforms this into a spectral embedding whose dimension is set via singular value thresholding.
The number of clusters is subsequently estimated via density-based clustering of this spectral embedding, an approach inspired by literature on the \glsentrylong{SBM}.
By leveraging and augmenting recent results on the spectral concentration of random matrices with Markovian dependence, we show that the method is asymptotically consistent---in spite of the dependencies between the count matrix's entries, and even when the count matrix is sparse.
We also present a numerical evaluation of our method, and compare it to alternatives.

\end{abstract}

\noindent
\textbf{Keywords:} Block Markov Chains, clustering, spectral norms, asymptotic analysis, community detection.\\
\noindent
\textbf{MSC2020:} 62H30, 60J10, 60B20, 60J20.

\glsreset{BMC}
\glsreset{SBM}

\section{Introduction}

Clustering algorithms are an essential tool for data analysis, machine learning, and network science.
They improve data interpretability by grouping data points that are similar in some manner, and they can improve performance on downstream tasks by reducing dimensionality and decreasing variance.
Reflecting on their broad scope of application, there is also a wide range of algorithms and methods that seek to discover different notions of clusters in data; see for example \cite{ezugwu2022comprehensive}.
However, a common feature of many of these algorithms is that they require the user to specify the number of clusters \textit{a priori}.
Ideally, one instead needs a method that can reliably select this number automatically, and preferably based on the data that is available.

We address this challenge in the context of clustering in \glspl{BMC} \cite{sanders2023singular}.
A \gls{BMC} is a \gls{MC} with $n$ states that admits a decomposition into $K \ll n$ clusters, and whose transition probabilities are constant on these clusters.
In other words, the \gls{MC} exhibits a block structure in its transition matrix.
The goal of this paper is to infer the number of clusters $K$ from a trajectory $X_0,X_1,\ldots,X_{\ell}$ of such \gls{MC}; see \Cref{fig:Are-there-two-or-three-clusters}.

\begin{figure}[hbtp]
    \centering
    \includegraphics[width=0.99\linewidth]{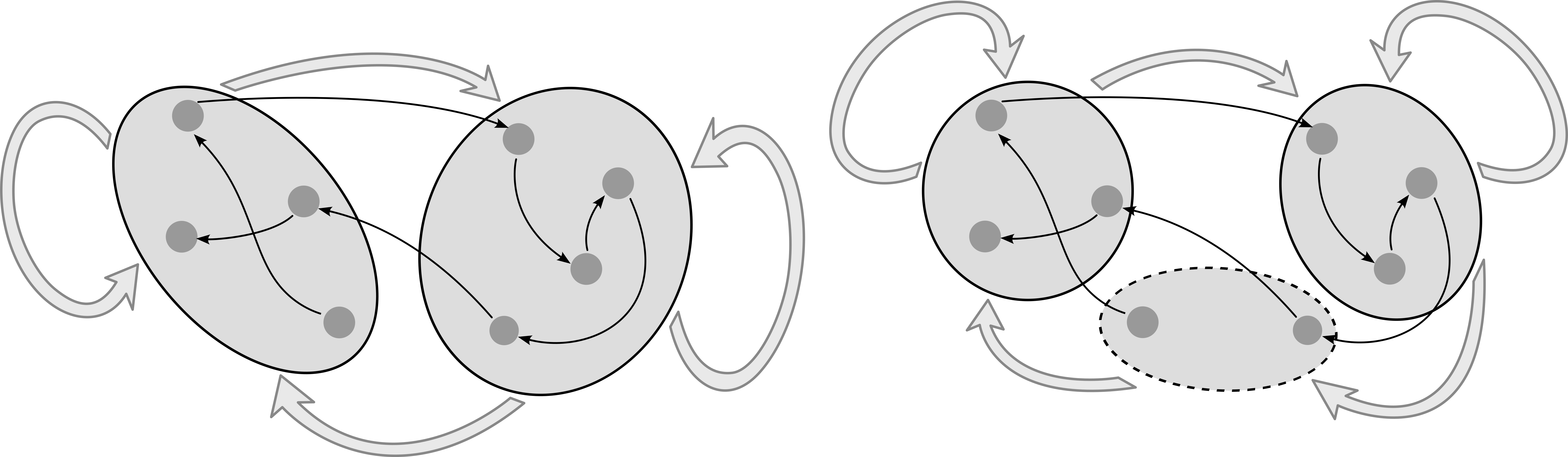}
    \caption{
        Consider the trajectory $X_0, \ldots, X_{\ell_n}$ depicted using thin black arrows, of some \gls{BMC}.
        Are there $K=2$ or perhaps $K=3$ clusters, as depicted on the left or right?
        Here, the dark gray circles represent states, the larger light gray circles represent clusters, and the thick arrows represent the (hidden) low-dimensional transition probabilities between clusters.
    }
    \label{fig:Are-there-two-or-three-clusters}
\end{figure}

\glspl{BMC} provide a model for sequential data in which there is a hidden, lower-dimensional structure.
Uncovering this hidden, lower-dimensional structure can aid with downstream tasks such as natural language processing and sequential decision taking \cite{VanWerde2022,azizzadenesheli2016reinforcement,Jedra:2022}.
\glspl{BMC} also offer a formal notion of clustering for sequential data, enabling the adoption of techniques from statistical inference.
This is analogous to the role that the \gls{SBM} plays for networks \cite{lee2019review}.

\subsection{Main results}

The problem of recovering the clusters from an observation of a trajectory $X_0, X_1,\ldots, X_{\ell_n}$ of large \glspl{BMC} with few observations (when $n \to \infty$ and $\ell_n = O(n \ln{n})$) was first addressed in \cite{ClusterBMC2017}.
There, the authors derive information--theoretic lower bounds on the error rate of any adaptive clustering algorithm, providing conditions on the path lengths $\ell_n$ and the \gls{BMC} parameters necessary for accurate cluster recovery.
It is shown in \cite{ClusterBMC2017} that a two-step clustering algorithm can asymptotically recover the clusters as the number of states $n \rightarrow \infty$, provided these conditions are satisfied.
The algorithm of \cite{ClusterBMC2017}, however, requires prior knowledge of $K$.

In this paper, we present an extension of the spectral clustering algorithm from \cite{ClusterBMC2017} that does automatically detect the number of clusters (see \Cref{alg:kpre,alg:kpost,alg:kmeans} in \Cref{sec:algorithm-description}).
Concretely, we use singular value thresholding to set the dimension of a spectral embedding of a matrix representation of the trajectory, before performing a density-based clustering.
The number of clusters is determined by counting the number of sufficiently dense regions in this embedding.
This procedure simultaneously obtains a partial clustering that can optionally be completed using a K-means clustering step.
We prove rigorously that for sufficiently large $n$, this algorithm is able to correctly identify the number of clusters whenever the conditions that allow cluster recovery are met \cite{ClusterBMC2017}.
In other words, the estimator is asymptotically consistent (see \Cref{prop:limit_kpre,thm:Khat=K,thm:kmeanserror} in \Cref{sec:algorithm-description}).
The method is inspired by \cite{NIPS2016_a8849b05}, which studies cluster recovery in the \gls{LSBM}.

Finally, we present an extensive numerical experimental section on the algorithm (see \Cref{sec:numerical-experiments}).
We investigate how the algorithm's performance is affected by its settings and different characteristics of a \gls{BMC}.
Our results showcase the algorithm's consistency, yet also highlight the considerable difficulty of the estimation problem.
We also demonstrate the limitations of singular value thresholding as a method for estimating the number of clusters numerically, and we show that the density-based clustering approach remedies these shortcomings.
We find that this requires an appropriate choice of embedding dimension, and we consider how this choice depends on the path length and the structure of the \gls{BMC}.
Stepping away from true \glspl{BMC}, we also demonstrate that the algorithm is reasonably robust against violations of the model's assumptions.
We lastly compare this algorithm to several alternative methods, including ones not specifically designed for clustering in trajectories of \glspl{BMC}.
Testing in a range of different scenarios, we identify two competitive alternative algorithms that present opportunities for future research.

\subsection{Overview of related literature}

\paragraph{Clustering in \texorpdfstring{\glspl{BMC}}{BMCs}}
As far as we know, the problem of detecting the number of clusters in a \gls{BMC} specifically has not been studied before.
However, under the assumption that the number of clusters of a \gls{BMC} is known, cluster recovery in \glspl{BMC} has previously been studied in \cite{ClusterBMC2017,duan2019state,zhang2019spectral,zhu2022learning}.
The methods in these papers all rely on spectral properties of the count matrix
\begin{equation}
    \hat{N} \allowbreak \eqdef \allowbreak (\hat{N}_{x,y})_{x,y\in [n]}
    \quad
    \textnormal{where}
    \quad
    \Hat{N}_{x,y}
    \eqdef
    \sum_{t=0}^{\ell_n-1}
    \mathbbm{1}\left[X_t=x, X_{t+1}=y\right]
    \quad \textnormal{for }
    x,y \in [n]
    ,
    \label{eqn:Transition-count-matrix-def}
\end{equation}
which counts the number of transitions between different states of the \gls{MC}.

Of particular interest to us are bounds on the spectral norm of the noise matrix $\hat{N}-\mbbE[\hat{N}]$, which are used both to perform singular value thresholding and to control the properties of the subsequent density-based clustering used in our algorithm.
Analyzing the spectrum of $\hat{N}-\mbbE[\hat{N}]$ is generally challenging because this matrix has dependent entries.
However, since the mixing time of the \gls{BMC} is sufficiently small, these dependencies are weak enough that tight asymptotic bounds are possible \cite{Bounds2023, sanders2023singular, vanwerde2023matrix}.

In \cite{Bounds2023}, asymptotically tight bounds on $\lVert\hat{N}-\mbbE[\hat{N}]\rVert$ were obtained under certain conditions on the cluster sizes and the transition probabilities between clusters, and they hold as soon as $\ell_n = \omega(n)$.
In \cite{sanders2023singular}, the limiting distribution of the bulk of the spectrum of the matrix $\hat{N} - \mbbE[\hat{N}]$ was derived.
Their results were written down for the dense regime $\ell_n = \Theta(n^2)$, which was later improved to $\ell_n = \omega(n)$ in \cite{vanwerde2023matrix}.
The latter also gives the exact limiting value of $\sqrt{n/\ell_n}\lVert\hat{N}-\mbbE[\hat{N}]\rVert$ when $\ell_n=\omega(n(\ln n)^4)$.
Note however, that these results depend on the unknown parameters of the \gls{BMC}, and can therefore not directly be used in our algorithm's design.

Especially \cite{Bounds2023} is relevant to the present paper.
We were originally hoping to use \cite[Theorem 3]{Bounds2023} to establish \Cref{thm:Khat=K}, but establishing \Cref{thm:Khat=K} required weaker restrictions on the cluster transition probabilities than \cite[Theorem 3]{Bounds2023} allowed.
We therefore decided to augment \cite[Theorem 3]{Bounds2023}'s proof and establish \Cref{thm:spectral_norm_bound} as a replacement for \cite[Theorem 3]{Bounds2023}.
\Cref{sec:proofskpre} details the necessary proof modifications.

\paragraph{Clustering in nonsynthetic, sequential data}
Reference \cite{VanWerde2022} studied the applicability of the \gls{BMC} model in real-world, nonsynthetic data.
Their results indicate that \glspl{BMC} can produce meaningful insights across a diverse collection of sequential datasets.
Reference \cite{VanWerde2022} also released an open-source Python library called \textsc{BMCToolkit}, which contains \Cpp{} implementations of the clustering algorithms described in \cite{ClusterBMC2017}.

We have expanded \textsc{BMCToolkit} in several ways, and our improvements are available on the \gls{PyPI} as of v0.10.0.
Primarily, we have incorporated \Cref{alg:kpre,alg:kpost} to detect the number of clusters in a \gls{BMC}.
The underlying \Cpp{} library also gained functionalities to e.g.\ estimate the number of singular values above given thresholds, and compute inertias of matrices.
Overall code quality also improved.
Secondarily, we have made performance improvements to the existing \Cpp{} backend.
Noteworthy is that we have accelerated parts of the library using NVidia's \gls{CUDA}.
In other words, parts of the code can now benefit from massive parallelization on certain \glspl{GPU}.
Finally, for experimental purposes, we have expanded the libraries' capabilities to generate trajectories of different \glspl{MC}.

\paragraph{Estimating the number of communities in \texorpdfstring{\glspl{SBM}}{SBMs}}
The canonical \gls{SBM} is a random graph model in which vertices $x\in \mcV$ of a graph are partitioned into $K$ communities.
Edges $(x,y)$ between vertices $x$ and $y$ are drawn randomly and independently with a probability $p_{\sigma(x),\sigma(y)}$ that depends only on the cluster memberships $\sigma(x)$ and $\sigma(y)$ of $x$ and $y$, respectively.
Given a realization of a \gls{SBM}, the community detection task is qualitatively similar to that of cluster detection from the trajectory of a \glspl{BMC}, and there exists a vast literature studying community detection in \glspl{SBM}.
We refer to e.g.\ \cite{lee2019review,JMLR:v18:16-480} for recent reviews.

Our method and proof techniques draw upon ideas developed for different \glspl{SBM}.
However, the techniques available within the literature are not directly applicable due to the model differences.
Consider, for example, the directed, concatenated nature of the trajectory and the dependencies between entries of $\hat{N}$---these are not features of typical \glspl{SBM}.
Specifically, we adapt the approach of \cite{NIPS2016_a8849b05}, whose algorithm served as the original inspiration for this work.

In \cite{NIPS2016_a8849b05}, the authors present a spectral clustering algorithm for a variant of the \gls{SBM} with labeled edges that simultaneously recovers the communities and determines the number of communities.
This is achieved through a density-based clustering of a spectral embedding of the graph, where the number of communities is determined by counting the number of sufficiently dense regions in the embedding.
Adapting this strategy to \glspl{BMC} requires accounting for the directed, concatenated, and dependent nature of the trajectory, for which we adapt proof techniques from \cite{ClusterBMC2017,Bounds2023}.

When it comes to other work on estimating the number of communities in \glspl{SBM}, we may  split the literature into research on assortative \glspl{SBM} and nonassortative \glspl{SBM}.
In the assortative case, where nodes are preferentially connected to nodes inside their own cluster, another common approach is to (approximately) maximize a criterion called the \emph{modularity} \cite{PhysRevE.69.026113,newman2006modularity,Blondel_2008}.
The modularity of a clustering is a measure for the degree of within--cluster connectedness, and can be used as a clustering heuristic without assuming that the data is generated by the \gls{SBM}.
In the nonassortative case, a variety of statistical model selection approaches based on log-likelihoods and information criteria have been explored \cite{daudin:hal-01197587,RePEc:eee:csdana:v:60:y:2013:i:c:p:12-31,come2015model,newman2016estimating,passino2020bayesian,yang2020simultaneous}.
Related, \cite{NIPS2012_d6ef5f7f,dubois2013stochastic,pmlr-v51-li16d} identify the number of communities for various \gls{SBM} extensions by optimizing the average log-likelihood on held-out validation data.

Finally, there also exist spectral methods with which the number of communities are estimated by counting outliers in the spectrum of a matrix associated to the graph \cite{krzakala2013spectral,newman2013spectral,chatterjee2015matrix}, analogous to the singular value thresholding step of our method.
We specifically highlight \cite{BudelGabriel2020}, which compares several spectral and modularity-based methods for detecting the number of communities in assortative \glspl{SBM}; and from which we adapt one such spectral method to serve as a benchmark in \Cref{sec:numerical-experiments}.

\paragraph{\texorpdfstring{\glspl{HMM}}{Hidden Markov Models (HMMs)}}
\glspl{HMM} feature a hidden \gls{MC} $C_{t \geq 0}$ that generates an observed process $X_{t \geq 0}$.
Specifically, the value of the observed process $X_t$ depends only on the current state $C_t$ of the hidden \gls{MC}.
Early works on this topic include \cite{baum1966statistical,petrie1969probabilistic}, and we refer to \cite{zucchini2009hidden} for a review.

A \gls{BMC} is a special case of a \gls{HMM}.
In \glspl{BMC}, the observed process is a finite lifted \gls{MC}, and the hidden process can be considered the lifted \gls{MC} projected on the clusters.
Moreover, the block structure of the transition matrix of $X_t$ implies that observations are sampled uniformly from states belonging to the current cluster.
On the contrary, the observed process of a \gls{HMM} can take values in an arbitrary state space, and $\mbbP[X_t = \cdot \mid C_t=j]$ need neither be uniform nor have supports that are disjoint for different $j \in [K]$.

Estimating the number of hidden states in \glspl{HMM} is typically treated as a model selection problem.
Common approaches are penalized likelihood methods and information criteria; see \cite{buckby2023finding} for a recent review.
Alternatively, cross-validation methods are suitable when data is not too sparse \cite{celeux2008selecting}.
The case of real-valued observations poses particular theoretical challenges relating to unbounded likelihood functions, meaning that consistency results are scarce; see \cite{chen2024determine} for a recent overview of such results as well as a consistency proof using marginal likelihoods.

Most relevant to our setting are finite-alphabet \glspl{HMM}, where the observed process takes values in a finite set, of which \glspl{BMC} are a special case.
There are several consistency results for maximum likelihood-based estimators; see e.g.\ the review \cite{ephraim2002hidden}.
Parameter estimation for these \glspl{HMM} when the number of hidden states is known has also been studied using spectral methods in \cite{hsu2012spectral,siddiqi2010reduced}.

\section{Preliminaries}
\label{sec:preliminaries}

The goal of this paper is to detect the number of clusters from a single trajectory $X_0,X_1,\ldots,X_{\ell_n}$ of a \gls{BMC}.
To achieve this goal, we now first define a \gls{BMC} and establish notation.

\glsreset{BMC}
\subsection{\texorpdfstring{\glspl{BMC}}{Block Markov Chains (BMCs)}}

A \gls{BMC} is a finite-state \gls{MC}, denoted $\{X_t\}_{t \geq 0}$, taking values in $[n]$, whose transition matrix $P \allowbreak = \allowbreak(P_{x,y})_{x,y\in[n]}$ exhibits a block structure.
In particular, there exists a map $\sigma:[n]\rightarrow [K]$ with $K\leq n$, and a stochastic matrix $p = (p_{i,j})_{i,j \in [K]}$ satisfying $\sum_{j \in [K]}p_{i,j}=1$ for all $i\in [K]$, such that
\begin{equation}
    \label{eq:transitionMatrix}
    P_{x,y}=\frac{p_{\sigma(x),\sigma(y)}}{|\mcV_{\sigma(y)}|}\quad\text{for all}\quad x,y\in [n].
\end{equation}
Here, $\mcV_k\eqdef \sigma^{-1}(k)$ denotes the subset of states belonging to cluster $k\in [K]$.
One checks that indeed $\sum_{y\in[n]}P_{x,y}=1$.
We will refer to $p$ as the \textit{cluster transition matrix}.
Note that a state $x\in [n]$ can only belong to one of the $K$ clusters, i.e., $\Vk_k\cap \Vk_l=\emptyset$ for $k \neq l$ and $\sigma^{-1}(k)\neq \emptyset$ for $k\in [K]$.

We will specifically be interested in the limit $n \rightarrow \infty$, and therefore add a subscript $n$ to $\ell_n$ to emphasize that the path length is allowed to depend on $n$.
To characterize the behavior of the clustering $\{\mcV_k\}_{k=1}^K$ as $n \rightarrow \infty$, we assume that there exists a fixed $\alpha=(\alpha_k)_{k\in [K]}$ with $\sum_{k=1}^K\alpha_k=1$ such that $\lim_{n \rightarrow \infty}|\mcV_k|/n=\alpha_k$ for $k\in [K]$.

The theoretical results presented in this paper rely on two assumptions on the limiting cluster ratios $\alpha$ and the cluster transition matrix $p$, respectively.

\begin{assumption}
    \label{asm:assumption1}
    The limiting cluster fractions $\alpha$ are strictly positive, i.e., $\min_{k\in [K]}\alpha_k>0$.
\end{assumption}

\begin{assumption}
    \label{asm:assumption2}
    The \gls{MC} induced by $p$ is irreducible, and admits a stationary distribution $\pi$, i.e., which satisfies $\pi^\intercal p = \pi^\intercal$.

\end{assumption}

\Cref{asm:assumption1} ensures that cluster sizes are balanced, i.e., that $|\mcV_k|=\Theta(n)$ for all $k\in [K]$.
Moreover, it follows from \cref{asm:assumption2} that $\pi$ is unique and satisfies $\pi_k>0$ for $k\in [K]$.
It also follows from \eqref{eq:transitionMatrix} that $\Pi=(\Pi_x)_{x\in[n]}$ with $\Pi_x\eqdef \pi_{\sigma(x)}/|\mcV_{\sigma(x)}|$ satisfies $\Pi^\intercal P=\Pi^\intercal$, i.e., it is the unique stationary distribution of $P$.
As an immediate consequence, $\Pi_x=\Theta(1/n)$ for all $x\in[n]$.

\subsection{Generating \texorpdfstring{\glspl{BMC}}{BMCs} at finite \texorpdfstring{$n$}{n}}

For a given pair $(\alpha,p)$ satisfying \Cref{asm:assumption1,asm:assumption2}, we can construct a \gls{BMC} for any finite $n$ such that $\min_{k\in[K]}n\alpha_k\geq 1$ by assigning $|\mcV_k|=\floor{n\alpha_k}$ states to cluster $k=2,\ldots,K$, and placing the remaining $|\mcV_1|=n-\sum_{k=2}^K|\mcV_k|$ states in cluster $1$.
Notice that $|\mcV_1|-\floor{n\alpha_1}\leq K-1$.

\subsection{Approximate cluster assignment}

The algorithms that we consider all take as input a trajectory $X_0, X_1, \ldots, X_{\ell_n}$ and then output an estimate $\hat{K}$ of the number of clusters.
Sometimes, the algorithms also return an approximate cluster assignment $\hat{\sigma}: [n] \rightarrow [\hat{K}]$.
In those cases, we let
\begin{equation}
    \{
        \hat{\mcV}_k
    \}_{k=1}^{\hat{K}}
    =
    \bigl\{
        \{
            i \in [n]
            :
            \hat{\sigma}(i) = k
        \}
    \bigr\}_{k=1}^{\hat{K}}
\end{equation}
refer to the clusters implied by $\hat{\sigma}$, and
$
    \hat{\alpha}_k
    =
    | \hat{\mathcal{V}}_k | / n
$
to the cluster size fractions implied by $\hat{\sigma}$.

\subsection{Asymptotic notation}

Given a sequence of real-valued random variables $\{X_n\}_{n=1}^{\infty}$ and a deterministic sequence $\{a_n\}_{n=1}^{\infty}$, we denote $X_n=O_{\mbbP}(a_n)$ if and only if $\forall_{\varepsilon>0}\exists_{\delta_{\varepsilon},N_{\varepsilon}}:\mbbP[|X_n/a_n|\geq \delta_{\varepsilon}]\leq\varepsilon \forall_{n>N_{\varepsilon}}$, and $X_n=o_{\mbbP}(a_n)$ if and only if $\mbbP[|X_n/a_n|\geq \delta]\rightarrow 0\forall_{\delta>0}$.
Similarly, we denote $X_n=\Omega_{\mbbP}(a_n)$ if and only if $\forall_{\varepsilon>0}\exists_{\delta_{\varepsilon},N_{\varepsilon}}:\mbbP[|X_n/a_n|\leq \delta_{\varepsilon}]\leq\varepsilon \forall_{n>N_{\varepsilon}}$, and $X_n=o_{\mbbP}(a_n)$ if and only if $\mbbP[|X_n/a_n|\leq \delta]\rightarrow 0\forall_{\delta>0}$.
We denote $X_n=\Theta_{\mbbP}(a_n)$ if and only if $X_n=O_{\mbbP}(a_n)$ and $X_n=\Omega_{\mbbP}(a_n)$.

\section{The algorithm}
\label{sec:algorithm-description}

Inspired by the two-stage procedure in \cite{NIPS2016_a8849b05} which can estimate the number of clusters in a \gls{LSBM}, we propose a similar procedure to estimate the number of clusters of a \gls{BMC}.
Contrary to the entries of adjacency matrices generated by most \glspl{SBM}, however, the entries of the count matrix $\hat{N}$ are dependent.

Proof modifications are therefore required.
Fortunately, asymptotically tight bounds on the spectral norm of $\hat{N}$ recently established in \cite{Bounds2023} now allow us to adapt the proof techniques of \cite{NIPS2016_a8849b05} and prove consistency of the two-stage procedure for \glspl{BMC} we describe next.

\subsection{Step 1: Singular value thresholding on a trimmed count matrix}

The block structure of the transition matrix induces an approximate low-rank structure in the count matrix $\hat{N}$ that we exploit to obtain an initial estimate for the number of clusters.
In order to guarantee good spectral concentration, we first need to trim the count matrix $\hat{N}$; \cite{Bounds2023} contains proof details on this.
Specifically, we set to zero all rows and columns corresponding to states in the set $\Gamma$ that contains the $\floor{n\exp(-\ell_n/n)}$ most visited states:
\begin{equation}
    \label{eqn:trimmed_matrix_def}
    \hat{N}_{\Gamma} \allowbreak \eqdef \allowbreak (\hat{N}_{\Gamma,x,y})_{x,y\in [n]}
    \quad
    \textnormal{where}
    \quad
    \hat{N}_{\Gamma,x,y}\eqdef \hat{N}_{x,y}\mathbbm{1}\{(x,y)\in\Gamma^\mathrm{c}\times\Gamma^\mathrm{c}\} \quad \text{for } x,y\in[n].
\end{equation}
We then count the number of singular values $\sigma_i(\hat{N}_{\Gamma})$ of $\hat{N}_{\Gamma}$ that exceed a threshold $\gamma_n$, leading to the following estimator:
\begin{equation}
    \label{eqn:kpredef}
    \Ktilde\eqdef |\{i\in [n]:\sigma_i(\hat{N}_{\Gamma})\geq \gamma_n\}|.
\end{equation}
\Cref{alg:kpost} can compute \eqref{eqn:kpredef} efficiently by exploiting Sylvester's law and $LDL^\intercal$-decompositions.
We explain this in \Cref{sec:On-singular-value-thresholding-via-Sylvesters-law}.

\begin{algorithm}[H]
    \caption{Pseudocode for singular value thresholding on a trimmed count matrix}
    \label{alg:kpre}
    \KwData{A trajectory $X_0,\dotsc,X_{\ell_n}$, and singular value threshold $\gamma_n$}
    \KwResult{An estimate number of clusters $\Ktilde$}
    Build transition count matrix $\hat{N}_{x,y}\gets\sum_{t=0}^{\ell_n-1}\mathbbm{1}\left\{X_t=x, X_{t+1}=y\right\}$ for $x,y\in[n]$;\label{ln:build_Nhat}\\
    Initialize $\Gamma \gets \emptyset$;\\
    \For{$i = 1, \ldots, \lfloor n \exp{(-\ell_n/n)} \rfloor$}{
        Update $\Gamma \gets \Gamma \cup \arg \max \bigl\{ y \in \mathcal{V} \backslash \Gamma : \sum_{x \in \mathcal{V}} \hat{N}_{x,y} \bigr\}$;\\
    }
    Calculate trimmed matrix $\hat{N}_\Gamma \gets \{ \hat{N}_{x,y}\mathbbm{1}\{(x,y)\in\Gamma^{\mathrm{c}}\times\Gamma^{\mathrm{c}}\}\}_{x,y\in[n]}$;\label{ln:trim_Nhat}\\
    Calculate the $LDL^\intercal$-decomposition of $\hat{N}_\Gamma^\intercal \hat{N}_\Gamma - \gamma_n^2 I_{n\times n}$;\label{ln:compute_ldl_decomposition}\\
    Perform singular value thresholding $\Ktilde\gets |\{D_{x,x}>0:x\in [n]\}|$;\label{ln:perform_sing_val_thresholding}\\
    \Return $\Ktilde$
\end{algorithm}

\subsubsection{Consistency result}

We prove the following in \Cref{sec:proofskpre}:

\begin{proposition}
    \label{prop:limit_kpre}

    Presume \cref{asm:assumption1,asm:assumption2}, and that $\ell_n = \omega(n)$.
    If
    $
        \omega(\sqrt{\ell_n/n})
        =
        \gamma_n
        =
        o(\ell_n/n)
        ,
    $
    then the output of \cref{alg:kpre}, $\Ktilde$, equals $\rank(p)$ with high probability as $n \rightarrow \infty$.
\end{proposition}

Proving \Cref{prop:limit_kpre} requires some work.
An application of the existing \cite[Corollary 4]{Bounds2023} would imply quickly that the estimator of \Cref{alg:kpre} is asymptotically consistent when $\rank(p)=K$ (see \Cref{sec:proof_of_limit_kpre_assuming_cor_singvalscalingrankp}), but the case $\rank(p) < K$ would then not be covered.
Analyzing the case that $\rank(p) < K$ is necessary to prove the implication of \Cref{prop:limit_kpre}: specifically, that \Cref{alg:kpre}'s output estimates $K$ asymptotically consistently \emph{if only if} $\rank(p)=K$.
Note also that \cite[Corollary 4]{Bounds2023} would require the stronger assumption that $p_{i,j}>0$ for all $i,j\in[K]$, which we would prefer not to make.

Therefore, to prove \Cref{prop:limit_kpre}, we opted to generalize the results of \cite{Bounds2023} in \Cref{sec:proof_of_spectral_norm_bound,sec:proof_of_cor_singvalscalingrankp} so that they hold under the weaker \cref{asm:assumption1,asm:assumption2}.
Specifically, we prove that the singular values of the trimmed count matrix $\hat{N}_{\Gamma}$ satisfy
\begin{equation}
    \label{eqn:singvalscaling}
    \sigma_i(\hat{N}_{\Gamma})=
    \begin{cases}
        \Theta_{\mathbb{P}}\lb\frac{\ell_n}{n}\rb   & i\leq \rank(p), \\
        O_{\mathbb{P}}\lb\sqrt{\frac{\ell_n}{n}}\rb & i> \rank(p)
    \end{cases}
\end{equation}
whenever $\ell_n = \omega(n)$, and \cref{asm:assumption1,asm:assumption2} hold.
The proof of \Cref{prop:limit_kpre} given \eqref{eqn:singvalscaling} is straightforward, and is presented in \Cref{sec:proof_of_limit_kpre_assuming_cor_singvalscalingrankp}.

\subsubsection{Limitations}

Observe that \Cref{prop:limit_kpre} has now established that \Cref{alg:kpre} estimates $K$ consistently if and only if $p$ is full-rank.
However, \glspl{BMC} may exhibit meaningful clusters even when $p$ is rank-deficient.
Similar observations were also made in the \gls{SBM}, see for example \cite{10.1214/15-AOS1370}.
We illustrate this for \glspl{BMC} in \Cref{ex:dot_product_model}:

\begin{example}[Dot-product model \cite{10.5555/1777879.1777890,athreya2018statistical}]
    \label{ex:dot_product_model}

    Presume \Cref{asm:assumption1}.
    For each $i\in 1,\ldots,K$, let $v_i\in \mbbR_{\geq 0}^d$ with $d<K$ and define $p_{i,j}= \langle v_i,v_j\rangle/\sum_{j=1}^K\langle v_i,v_j\rangle$.
    Then, $\rank(p) = \textnormal{dim span}(v_1, \allowbreak \ldots, \allowbreak v_K) \leq d < K$.
    For instance, let $K=3$, $d=2$, $v_1=a\cdot(1,0)^\intercal$, $v_2=b\cdot (1,1)^\intercal$, and $v_3=a\cdot(0,1)^\intercal$ with $a,b>0$.
    The corresponding transition matrix and stationary distribution on the clusters are given by
    \begin{equation}
        p
        =
        \frac{1}{2(a+b)}
        \begin{psmallmatrix}
            2 a & 2 b & 0 \\
            a & 2 b & a \\
            0 & 2 b & 2 a \\
        \end{psmallmatrix}
        ,
        \qquad
        \pi
        =
        \frac{1}{2(a+b)}
        \begin{psmallmatrix}
            a \\  2b \\  a \\
        \end{psmallmatrix}
        .
    \end{equation}
    Note that the first and third cluster are indeed distinguishable: states belonging to the first cluster will never transition to states belonging to the third and \emph{vice versa}.
\end{example}

Consequently, a different algorithm is needed to obtain an estimator that can be asymptotically consistent whenever $\rank(p) < K$.
We describe such algorithm in \Cref{sec:Step-2}, remedying this shortcoming of \Cref{alg:kpre}.

\subsection{Step 2: Density-based clustering on a spectral embedding}
\label{sec:Step-2}

\Cref{alg:kpost} first passes to a $2r$-dimensional spectral embedding of the observations
\begin{align}
    \hat{X}
    \eqdef
    [
        \sigma_1 U_{\cdot,1}
        ,
        \ldots
        ,
        \sigma_r U_{\cdot,r}
        ,
        \sigma_1 V_{\cdot,1}
        ,
        \ldots
        ,
        \sigma_r V_{\cdot,r}
    ]
    ,
    \label{eqn:spectral_embedding_def}
\end{align}
which is constructed out of $\hat{N}_{\Gamma}$'s \gls{SVD} $\hat{N}_{\Gamma}\defeq U \Sigma V^\intercal$ with $\Sigma=\diag(\sigma_1,\ldots,\sigma_n)$ and $\sigma_1\geq \sigma_2\geq\ldots\geq\sigma_n$ the singular values of $\hat{N}_\Gamma$.
Similar embeddings have been used for directed graphs with community structure \cite{sussman2012consistent}.
Loosely speaking, because $\hat{N}_{\Gamma}$ is not symmetric, both the left- and right-singular vectors are necessary to capture information about transitions out of and into a given state $x\in[n]$, respectively.
We then define the neighbourhoods of a state $x\in[n]$ by
\begin{equation}
    \label{eqn:nbrhooddef}
    \mcN_x\eqdef \left\{ y\in[n] \big| \lVert \hat{X}_{x,\cdot}-\hat{X}_{y,\cdot}\rVert_2\leq h_n \right\}.
\end{equation}
Here, $h_n$ is a threshold controlling the neighborhood radius.
By counting the number of neighborhoods larger than a threshold $\rho_n$, taking care to avoid double counting, we can obtain a refined estimate $\Khat$ of $K$ by choosing $r = \hat{K}^{\mathrm{spec}}$.

\begin{algorithm}[H]
    \caption{Pseudocode for density-based estimate}\label{alg:kpost}
    \KwData{A trimmed count matrix $\hat{N}_\Gamma$, embedding dimension $2r$, and thresholds $h_n, \rho_n$}
    \KwResult{An estimate number of clusters $\Khat$, partial clustering $\{\hat{V}_k\}_{k=1,\ldots,\Khat}$}
    Calculate \gls{SVD} $U\Sigma V^\intercal$ of $\hat{N}_{\Gamma}$;\\
    Construct spectral embedding $\hat{X}\gets [\sigma_1U_{\cdot,1},\ldots,\sigma_rU_{\cdot,r},\sigma_1V_{\cdot,1},\ldots,\sigma_rV_{\cdot,r}]$;\\
    \For{$x\in [n]$}{
        \label{ln:compute_neighborhoods}Construct neighborhood $\mathcal{N}_x\gets\{y\in [n]|\|\hat{X}_{x,\cdot}-\hat{X}_{y,\cdot}\|_2\leq h_n
            \};$
    }
    Initialize $\hat{\mcV}_0\gets[n],\quad k\gets0$;\label{ln:init_max_neighborhood}\\
    \While{$|\hat{\mcV}_k|\geq \rho_n$}{\label{ln:find_largest_neighborhoods_1}
        Update $k\gets k+1$;\label{ln:find_largest_neighborhoods_2}\\
        Find center with largest neighborhood $z_k^*\gets\argmax_{x\in [n]}\{\mathcal{N}_x\backslash\bigcup^{k-1}_{l=1}\Vkhat_l\}$;\label{ln:find_largest_neighborhoods_3}\\
        Store neighborhood $\Vkhat_k\gets\mathcal{N}_{z_k^*}\backslash\bigcup^{k-1}_{l=1}\Vkhat_l$;\label{ln:find_largest_neighborhoods_4}\\
    }\label{ln:find_largest_neighborhoods_5}
    Set $\Khat\gets k-1$;\label{ln:set_final_value_of_k_hat}\\
    \Return $\Khat, \{z_k^*\}_{k=1,\ldots,\Khat}, \{\Vkhat_k\}_{k=1,\dotsc,\Khat}$
\end{algorithm}

\subsubsection{Consistency result}

In \cite{ClusterBMC2017}, the authors identify an information theoretic quantity,
\begin{equation}
    I(\alpha, p)
    \eqdef
    \min_{i,j \in [K]}
    \sum_{k=1}^K
    \frac{1}{\alpha_i}
    \Bigl(
        \pi_ip_{i,k}
        \ln{ \frac{p_{i,k}}{p_{j,k}} }
        +
        \pi_kp_{k,i}
        \ln{ \frac{p_{k,i}\alpha_j}{p_{k,j}\alpha_i}}
    \Bigr)
    +
    \Bigl(
        \frac{\pi_j}{\alpha_j}
        -
        \frac{\pi_i}{\alpha_i}
    \Bigr)
    ,
    \label{eqn:I(a,p)}
\end{equation}
the precise value of which is shown to affect cluster recoverability in the limit $n \rightarrow \infty$.
Under the condition $I(\alpha,p)>0$, we prove the following in \Cref{sec:proofskpost}:

\begin{theorem}
    \label{thm:Khat=K}

    Presume \cref{asm:assumption1,asm:assumption2}, $I(\alpha, p) > 0$, and $\ell_n = \omega(n)$.
    Upon choosing settings $\gamma_n$, $h_n$, and $\rho_n$ that satisfy
    \begin{equation}
        \label{eqn:parameter_conditions}
        \omega\lb \sqrt{\ell_n/n}\rb=\gamma_n =o(\ell_n/n),\quad \omega(\ell_n/n)=nh^2_n=o((\ell_n/n)^2),\quad \omega((\ell_n/n)/h_n^2)=\rho_n=o(n),
    \end{equation}
    then the output $\hat{K}$ of \cref{alg:kpost} initialized with embedding dimension $r=\Ktilde$ obtained from \cref{alg:kpre} satisfies $\Khat=K$ with high probability as $n \rightarrow \infty$.
\end{theorem}

As more and more transitions are observed, the rows of $\hat{X}$ start to concentrate to their respective expectations.
As points in $\mathbb{R}^{2r}$, one can show that these are constant on each cluster, and well-separated in $L^2$-distance between different clusters provided that $I(\alpha,p)>0$ (see \Cref{lem:equivalent_spectral_embedding} and \Cref{lem:separability_property}).
Rows of $\hat{X}$ that are close to each other are thus likely to belong to the same cluster and \emph{vice versa}.

The proof then relies on a careful analysis of how the rows $\hat{X}_{x,\cdot}$ concentrate in $\mbbR^{2r}$.
In particular, the radius threshold $h_n$ should be chosen large enough to ensure that a state's neighborhood captures most of its cluster, but small enough to avoid overlap with other clusters (cf.\ \eqref{eqn:parameter_conditions}).
Similarly, the neighborhood size threshold $\rho_n$ should be chosen large enough to ensure that the number of outliers that remain after $K$ iterations is less than $\rho_n$, yet small enough to ensure that we detect the smallest cluster (cf.\ \eqref{eqn:parameter_conditions}).
\Cref{sec:equivalence_of_norms,sec:concentration_of_low_rank_approximation} provide results relating the concentration of $\hat{X}$ to the concentration of $\hat{N}_{\Gamma}$ in spectral norm, which we have already analyzed in \Cref{sec:proofskpre}.
These concentration properties are then used in \Cref{sec:density-based_clustering_recovers_the_true_number_of_clusters} to derive the appropriate asymptotic scalings of the thresholds $h_n$ and $\rho_n$ that ensure consistency of the estimator $\Khat$.

\subsection{(Optional) Step 3: Completion via K-means clustering}

While executing \cref{alg:kpost}, one simultaneously obtains a partial clustering that can readily be completed using an optional K-means clustering step.
Although clustering is not the main focus of this paper, we describe the completion step for completeness next.

Observe that \cref{alg:kpost} calculates for $k=1, \ldots, \Khat$ a neighborhood center $z_k^*\in[n]$ and a neighborhood $\hat{\mcV}_k\subset\mcV$.
These neighborhoods form an incomplete cover of the state space $[n]$ and can be viewed as a partial clustering.
Points not in $\cup_{k\in [\Khat]}\hat{\mcV}_k$ can either be treated as outliers, or assigned to the cluster whose center is closest.
The latter is done in \cref{alg:kmeans}:

\begin{algorithm}[H]
    \caption{Pseudocode for $K$-means completion}
    \label{alg:kmeans}
    \KwData{Embedding $\hat{X}$, cluster centers $\{z_k^*\}_{k=1,\ldots\Khat}$,\\
    partial clustering $\{\hat{\mcV}_{k}\}_{k=1,\ldots,\Khat}$}
    \KwResult{Approximate clustering $\{\hat{\mcV}_{k}\}_{k= 1,\ldots,\Khat}$}
    \For{$x\in[n]\backslash\bigcup^{\Khat}_{l=1}\Vkhat_l$}{
        Find closest center $k_*\gets\argmin_{k\in{1,\dotsc,\Khat}}\{\|\Hat{X}_{x,\cdot}-\Hat{X}_{z_k^*,\cdot}\|_2\}$;\\
        Assign to cluster $\Vkhat_{k_*}\gets\Vkhat_{k_*}\cup\{x\}$;
    }
    \Return $\{\Vkhat_k\}_{k=1,\ldots,\Khat}$
\end{algorithm}

The combination of \Cref{alg:kpre,alg:kpost,alg:kmeans} effectively results in an augmented version of the spectral clustering algorithm described in \cite{ClusterBMC2017}---one that does not rely on prior knowledge of $K$.

\subsubsection{Fraction of misclassified states}
Since \Cref{thm:Khat=K} ensures that \Cref{alg:kpost} recovers the true number of clusters with high probability, we can characterize the performance of \Cref{alg:kmeans} in terms of its set of misclassified states, which we define as follows:
\begin{equation}
    \mcE
    \eqdef
    \bigcup_{k=1}^{K} \hat{\mcV}_{\gamma^{\mathrm{opt}}(k)}\setminus\mcV_k
    \quad
    \textnormal{where}
    \quad
    \gamma^{\mathrm{opt}}\in \argmin_{\gamma\in\mathrm{Perm}(K)}\biggl|\bigcup_{k=1}^{K} \hat{\mcV}_{\gamma^{\mathrm{opt}}(k)}\setminus\mcV_k\biggr|
    .
    \label{eqn:misclassified_states}
\end{equation}
Relying on the analysis of \cite{ClusterBMC2017}, the following can be established for \Cref{alg:kmeans}:

\begin{theorem}[{Adapted from \cite[Theorem 2]{ClusterBMC2017}}]
    \label{thm:kmeanserror}
    Presume \cref{asm:assumption1,asm:assumption2}, $I(\alpha, p) > 0$, and $\ell_n = \omega(n)$.
    Then the proportion of misclassified states \eqref{eqn:misclassified_states} after \cref{alg:kmeans} when initialized with the output of \cref{alg:kpre,alg:kpost} using parameters satisfying \eqref{eqn:parameter_conditions} satisfies:
    \begin{equation}
        \frac{|\mathcal{E}|}{n}
        =
        \OrderP{\frac{n}{\ell_n}}=\orderP{1}.
        \label{eqn:Misclassification-error}
    \end{equation}
\end{theorem}

Contrary to \cite[Theorem 2]{ClusterBMC2017}, \Cref{thm:kmeanserror} does not assume prior knowledge of $K$.
Also, compared to \cite[Theorem 2]{ClusterBMC2017}, the error rate in \eqref{eqn:Misclassification-error} is tighter by factor $\ln{(\ell_n/n)}$.
This is achieved by leveraging \Cref{thm:spectral_norm_bound} (or, with slightly more restrictions, \cite[Theorem 3]{Bounds2023}), rather than the analogous \cite[Proposition 7]{ClusterBMC2017}.

\subsection{Discussion}

\subsubsection{Thresholding gaps between singular values}

Observe that \eqref{eqn:singvalscaling} implies that also the gaps between successive singular values typically exhibit the same scaling as the singular values themselves.
This suggests that replacing $\sigma_i$ by $\delta_i\eqdef \sigma_{i}-\sigma_{i-1}$ in \eqref{eqn:kpredef} also leads to a consistent estimator.
However, because there is an abundance of singular values near the spectral edge \cite[Proposition D.10]{vanwerde2023matrix}, the gaps for $i>\rank(p)$ are typically actually much smaller than $\sqrt{\ell_n/n}$, i.e., these are typically $o_{\mbbP}(\sqrt{\ell_n/n})$.
Consequently, the separation in scale between the first $K$ gaps and the gaps in the bulk is much greater than that between the first $K$ singular values and those in the bulk.
To exploit this observation in full one requires the precise asymptotic scaling of the gaps to choose an appropriate threshold that is $o(\sqrt{\ell_n/n})$, but we leave this for future work.

\subsubsection{On the conditions of \texorpdfstring{\Cref{thm:Khat=K}}{Theorem 3.3}}
\label{sec:Discussion-on-the-necessity-of-Ialphap-positive-and-elln-omega-n}

Under mild conditions on the clustering algorithms involved, \cite{ClusterBMC2017} showed that no clustering algorithm can achieve \textit{asymptotically accurate recovery}, i.e., $\mathbb{E}[|\mcE|]=o(n)$, unless $I(\alpha, p) > 0$ and $\ell_n = \omega(n)$.
Moreover, under the same conditions, \cite{ClusterBMC2017} showed that no clustering algorithm can achieve \textit{asymptotically exact recovery}, i.e., $\mathbb{E}[|\mcE|]=o(1)$, unless $I(\alpha, p) > 0$ and $\ell_n-n\ln n/I(\alpha, p) = \omega(1)$.
Related, \cite{ClusterBMC2017} also showed that larger values of $I(\alpha, p)$ positively affect the performance of an asymptotically optimal two-stage clustering algorithm that they simultaneously proposed.
Thus, intuitively, smaller values of $I(\alpha, p)$ correspond to more difficult clustering problems.

\Cref{thm:Khat=K} demonstrates that it is possible to estimate $K$ whenever asymptotically accurate recovery is possible ($I(\alpha, p) > 0$ and $\ell_n = \omega(n)$).
This is a consequence of \Cref{asm:assumption1}.
Even if exact recovery is not possible, the assumption that clusters contain a positive fraction of the states allows our algorithm to detect the number of clusters without being able to correctly assign all states to these clusters.
The condition $I(\alpha,p)>0$ ensures that states in different clusters have asymptotically distinct transition probabilities, thus avoiding issues surrounding identifiability; see \cite[Lemma 1]{ClusterBMC2017}.
Investigating in further detail for which asymptotic scalings of $\ell_n$ it would still be possible to recover the number of clusters accurately with decent probability, or to what extent detectability of the number of clusters is affected by the precise value of $I(\alpha, p)$, remains for future work.

\subsubsection{Relations to other density-based clustering algorithms}

\Cref{alg:kpost} is similar to other density-based clustering algorithms such as DBSCAN \cite{ester1996density} and variants thereof.
Indeed, similarly to DBSCAN, Step 2 is also restricted to states whose $\rho_n$-nearest neighbour is at most $h_n$ distance away.
Our particular prescription in \Cref{alg:kpost} is chosen to facilitate the proof of \cref{thm:Khat=K} as it allows us to build on related work \cite{ClusterBMC2017,NIPS2016_a8849b05}, enabling our original goal of deriving settings that ensure asymptotic consistency in \glspl{BMC}.
In practice, other density-based clustering algorithms that include additional (heuristic) optimizations might perform better in practice, and we include a direct numerical comparison of \Cref{alg:kpost} to HDBSCAN in \Cref{sec:results_alternative_methods}.

\subsubsection{Randomization can improve computational complexity}

Lastly, the computational complexity of \Cref{alg:kpost} can be improved by randomly sampling $\ceil{\ln n}$ states and computing neighborhoods only for these states.
It can be shown using similar techniques as in \cite{NIPS2016_a8849b05} that \Cref{thm:Khat=K} still holds when \Cref{alg:kpost} is modified in this way.
We do not pursue this here both for the sake of brevity, and because numerical experiments showed that the nonasymptotic performance of the algorithm suffers in sparse regimes.

\section{Numerical experiments}
\label{sec:numerical-experiments}

In this section, we investigate various properties of our method numerically.
For a more detailed description of the key performance indices that we are typically interested in, see \Cref{sec:Numerical-key-performance-indices}.

Recall that \Cref{alg:kpre,alg:kpost} require specification of the different thresholds $h_n$, $\rho_n$, and $\gamma_n$.
In our experiments, we parameterize these thresholds as follows: given some $a \in (0,1)$, $b \in (0,a)$, and $c \in (0,1)$, we set
\begin{gather}
    \label{eqn:paramter_parameterization}
    nh_n^2
    =
    (\ell_n/n)^{1+a}
    ,
    \quad
    \textnormal{``radius threshold''},
    \\
    \rho_n
    = n/(\ell_n/n)^{a-b}
    ,
    \quad
    \textnormal{``neighborhood size threshold'', and}
    \\
    \gamma_n
    =
    ( \ell_n/n )^{c}
    ,
    \quad
    \textnormal{``singular-value threshold''}
    .
\end{gather}
These choices satisfy \eqref{eqn:parameter_conditions} whenever $\ell_n = \omega(n)$.\footnote{%
Verification of the conditions on $\gamma_n$ and $h_n$ is immediate.
The lower bound on $\rho_n$ is verified by checking that whenever $\ell_n/n=\omega(1)$, the following holds for $a\in (0,1)$, $b\in (0,a)$:
$
    (\ell_n/n)/h_n^2  = \ell_n/(\ell_n/n)^{1+a} = n/(\ell_n/n)^{a} = o(\rho_n)
$.
}

Based on the numerical experiment in \Cref{sec:Relative_accuracy_vs_parameter_settings} that looks for suitable exponents, we decided to present the experiments in this section typically with exponents $a = 9/10, b = 1/10, c = 3/4$.
Whenever we deviate from this choice, we mention this explicitly.

\subsection{Illustration of the method}

We start with a few direct implementations, so as to establish proof--of--concept in synthetic environments.

\subsubsection{Proof--of--concept on synthetic data: From easy to hard scenarios}
\label{sec:Proof-of-concept-on-synthetic-data}

In this experiment, we apply our algorithm to four \glspl{BMC} that have just a few clusters.
Each \gls{BMC} has different parameters $(\alpha_a, p_a)$, $\cdots$, $(\alpha_d, p_d)$, and we have chosen these parameters such that the information quantity $I(\alpha, p)$ becomes progressively smaller.
Plots containing the estimated number of clusters $\hat{K}$ as a function of system size $n$ are shown in \Cref{fig:Proof_of_concept__From_easy_to_hard}(a)--(d).

\input{images/TikZ__Proof_of_concept_plots_from_easy_to_hard.tex}

\paragraph{More difficult problems (smaller $\ell_n$, $I(\alpha,p)$) require larger $n$}
From plots (a) to (d) in \Cref{fig:Proof_of_concept__From_easy_to_hard}, the clustering problem becomes progressively more difficult (recall \Cref{sec:Discussion-on-the-necessity-of-Ialphap-positive-and-elln-omega-n}).
Observe firstly that the minimum system size from which onward our method estimates the number of clusters correctly increases as the clustering problem becomes more difficult.
Observe secondly that when the trajectory length $\ell_n = n ( \ln{n} )^\beta$ is longer ($\beta = 2.00$ versus $\beta = 1.75$ or $1.50$), the method performs better.
Both observations are natural because effectively, the implications of laws of large numbers kick in earlier.

\paragraph{Sparse regimes involve numerically difficult asymptotics}
Finally, let us note that we felt that we were encountering the difficult numerical nature of the asymptotics of the underlying mathematical problem throughout most  experiments.
That is, while it may be theoretically true that the algorithm should work for all $\ell_n = \omega(n)$ for large enough $n$, the necessary $n$ is prohibitively large from a practical point of view.
In fact, the necessary $n$ is exponentially large.
The path lengths in these experiments were actually still quite short: for example, in configuration $(\beta = 2.00, n = 1000)$, the average state is visited only approximately $47$ times; see also \Cref{fig:observation_matrix_visualization}.

\subsubsection{Illustration for low rank example}

Next, we demonstrate the necessity of \Cref{alg:kpost} when dealing with a \gls{BMC} in which $p$ is rank deficient.

\begin{figure}[hbtp]
    \centering
    \begin{subfigure}{0.495\linewidth}
        % This file was created with tikzplotlib v0.10.1.
\begin{tikzpicture}[scale=0.9,]
    \pgfplotsset{compat=1.18}
    \usepgfplotslibrary{fillbetween}
    \begin{axis}[
        name = third plot,
        % at={(second plot.below south west)},%yshift=-0.2\textwidth,
        % yshift=0.4\textwidth,
        width=240pt,height=99pt,
        % width = \textwidth, height=0.4125\textwidth,
        tick pos=both,
        xlabel={Relative accuracy},
        xmin=-1, xmax=1,
        xtick style={color=black},
        y grid style={darkgray176},
        ymin=0, ymax=0.55,
        ytick style={color=black}
        ]

        \path [draw=black, semithick, dashed]
        (axis cs:0,0)
        --(axis cs:0,1);

        \path [draw=black, semithick, dashed]
        (axis cs:-0.5,0)
        --(axis cs:-0.5,1);

        \node[anchor=north east] at (axis cs: 0.95,.5225) {$\ell_n = n (\ln n)^{2.0}$};
        % \addplot[ybar, draw=blue, fill = blue!30!white, fill opacity=0.5] coordinates{
        %     (-0.9, 0.94)	(-0.8, 0.038)	(-0.7, 0.018)	(-0.5, 0.004)
        % };
        % \addplot[ybar, draw=red, fill = red!30!white, fill opacity=0.5] coordinates{
        %     (-0.8, 0.016)	(-0.7, 0.008)	(-0.6, 0.012)	(-0.4, 0.004)	(-0.3, 0.016)	(-0.2, 0.018)	(-0.1, 0.01)	(0.0, 0.012)	(0.1, 0.01)	(0.2, 0.024)	(0.3, 0.026)	(0.4, 0.014)	(0.5, 0.046)	(0.6, 0.04)	(0.7, 0.044)	(0.8, 0.026)	(0.9, 0.044)	(1.0, 0.034)	(1.1, 0.044)	(1.2, 0.032)	(1.3, 0.044)	(1.4, 0.046)	(1.5, 0.044)	(1.6, 0.04)	(1.7, 0.048)	(1.8, 0.046)	(1.9, 0.042)	(2.0, 0.036)	(2.1, 0.028)	(2.2, 0.044)	(2.3, 0.024)	(2.4, 0.024)	(2.5, 0.018)	(2.6, 0.008)	(2.7, 0.008)	(2.8, 0.014)	(2.9, 0.006)
        % };
        \addplot[ybar, semithick, draw=blue, fill = blue!30!white, draw opacity = 0.7, fill opacity=0.4] coordinates{
            (-1,0)  (-0.9, 0.94)	(-0.8, 0.038)	(-0.7, 0.018)	(-0.5, 0.004)
        };
        \addplot[ybar, semithick, draw=red, fill = red!30!white, draw opacity = 0.7, fill opacity=0.4] coordinates{
            (-0.8, 0.002)	(-0.7, 0.006)	(-0.6, 0.014)	(-0.5, 0.008)	(-0.4, 0.008)	(-0.3, 0.01)	(-0.2, 0.018)	(-0.1, 0.018)	(0.0, 0.018)	(0.1, 0.014)	(0.2, 0.018)	(0.3, 0.02)	(0.4, 0.03)	(0.5, 0.024)	(0.6, 0.034)	(0.7, 0.028)	(0.8, 0.044)	(0.9, 0.056)	(1.0, 0.038)	(1.1, 0.038)	(1.2, 0.04)	(1.3, 0.058)	(1.4, 0.048)	(1.5, 0.056)	(1.6, 0.034)	(1.7, 0.032)	(1.8, 0.034)	(1.9, 0.046)	(2.0, 0.042)	(2.1, 0.034)	(2.2, 0.03)	(2.3, 0.032)	(2.4, 0.012)	(2.5, 0.014)	(2.6, 0.008)	(2.7, 0.012)	(2.8, 0.008)	(2.9, 0.004)	(3.0, 0.004)	(3.1, 0.004)	(3.5, 0.002)
        };

    \end{axis}

    \begin{axis}[
        name = second plot,
        % at={(third plot.above south west)},%yshift=-0.2\textwidth,
        width=240pt,height=99pt,
        yshift=64pt,
        % width=\axisdefaultwidth,height=0.4\textwidth,
        % tick align=inside,
        tick pos=both,
        xmin=-1, xmax=1,
        xtick style={color=black},
        xticklabel=\empty,
        y grid style={darkgray176},
        ylabel={Rel. occurence},
        ymin=0, ymax=0.55,
        % ytick distance=0.5,
        ytick style={color=black}
        ]

        \path [draw=black, semithick, dashed]
        (axis cs:0,0)
        --(axis cs:0,1);
        \path [draw=black, semithick, dashed]
        (axis cs:-0.5,0)
        --(axis cs:-0.5,1);
        \node[anchor=north east] at (axis cs: 0.95,.5225) {$\ell_n = n (\ln n)^{3.5}$};
        % \addplot[ybar, draw=blue, fill = blue!30!white, fill opacity=0.5] coordinates{
        %     (-0.9, 0.002)	(-0.8, 0.006)	(-0.7, 0.012)	(-0.6, 0.058)	(-0.5, 0.164)	(-0.4, 0.174)	(-0.3, 0.2)	(-0.2, 0.12)	(-0.1, 0.066)	(0.0, 0.05)	(0.1, 0.046)	(0.2, 0.014)	(0.3, 0.032)	(0.4, 0.024)	(0.5, 0.01)	(0.6, 0.014)	(0.7, 0.004)	(0.8, 0.004)
        % };
        % \addplot[ybar, draw=red, fill = red!30!white, fill opacity=0.5] coordinates{
        %     (-0.9, 0.004)	(-0.8, 0.108)	(-0.7, 0.264)	(-0.6, 0.1)	(-0.5, 0.124)	(-0.4, 0.102)	(-0.3, 0.122)	(-0.2, 0.078)	(-0.1, 0.052)	(0.0, 0.022)	(0.1, 0.012)	(0.2, 0.004)	(0.3, 0.002)	(0.4, 0.004)	(0.5, 0.002)
        % };
        % \path[name path=axis] (axis cs:-1,0) -- (axis cs:1,0);
        \addplot[ybar, semithick, draw = blue, fill = blue!30!white, draw opacity = 0.7, fill opacity = 0.4] coordinates{
            (-0.6, 0.002)	(-0.5, 0.008)	(-0.4, 0.026)	(-0.3, 0.114)	(-0.2, 0.244)	(-0.1, 0.29)	(0.0, 0.174)	(0.1, 0.04)	(0.2, 0.026)	(0.3, 0.034)	(0.4, 0.02)	(0.5, 0.008)	(0.6, 0.006)	(0.7, 0.006)	(0.8, 0.002)
        };
        \addplot[ybar, semithick, draw = red, fill = red!30!white, draw opacity = 0.7, fill opacity=0.4] coordinates{
            (-0.8, 0.008)	(-0.7, 0.144)	(-0.6, 0.368)	(-0.5, 0.284)	(-0.4, 0.104)	(-0.3, 0.072)	(-0.2, 0.018)	(-0.1, 0.002)
        };
    \end{axis}

    \begin{axis}[
    name = first plot,
    % at={(second plot.above south west)},%yshift=-0.2\textwidth,
    % width=\axisdefaultwidth,height=0.4\textwidth,
    width=240pt,height=99pt,
    yshift=128pt,
    tick pos=both,
    xmin=-1, xmax=1,
    xtick style={color=black},
    xticklabel=\empty,
    y grid style={darkgray176},
    % ytick distance=0.5,
    ymin=0, ymax=0.55,
    ytick style={color=black},
    legend style={at={(0.975, 0.05)}, anchor=south east},
    area legend
    ]

    \path [draw=black, semithick, dashed]
    (axis cs:0,0)
    --(axis cs:0,1);
    \path [draw=black, semithick, dashed]
        (axis cs:-0.5,0)
        --(axis cs:-0.5,1);
    \node[anchor=north east] at (axis cs: 0.95,.5225) {$\ell_n = n (\ln n)^{5.0}$};
        % \addplot[ybar, draw=blue, fill = blue!30!white, draw opacity = 0.7, fill opacity=0.4] coordinates{
        %     (-0.8, 0.002)	(-0.6, 0.012)	(-0.5, 0.042)	(-0.4, 0.096)	(-0.3, 0.226)	(-0.2, 0.286)	(-0.1, 0.2)	(0.0, 0.064)	(0.1, 0.028)	(0.2, 0.022)	(0.3, 0.014)	(0.4, 0.004)	(0.5, 0.004)
        % };
        % \addlegendentry{$\Khat$}
        % \addplot[ybar, draw=red, fill = red!30!white, draw opacity = 0.7, fill opacity=0.4] coordinates{
        %     (-0.8, 0.056)	(-0.7, 0.45)	(-0.6, 0.41)	(-0.5, 0.078)	(-0.4, 0.006)
        % };
        % \addlegendentry{$\Ktilde$}
        \addplot[ybar, semithick, draw=blue, fill = blue!30!white, draw opacity = 0.7, fill opacity=0.4] coordinates{
            (-0.4, 0.004)	(-0.3, 0.002)	(-0.2, 0.066)	(-0.1, 0.26)	(0.0, 0.488)	(0.1, 0.102)	(0.2, 0.062)	(0.3, 0.008)	(0.4, 0.006)	(0.5, 0.002)
        };
        \addlegendentry{Alg. 2}
        \addplot[ybar, semithick, draw=red, fill = red!30!white, draw opacity = 0.7, fill opacity=0.4] coordinates{
            (-0.7, 0.044)	(-0.6, 0.43)	(-0.5, 0.526)
        };
        \addlegendentry{Alg. 1}
    \end{axis}

\end{tikzpicture}
        \caption{}
        \label{fig:Low_rank_example_relative_accuracy}
    \end{subfigure}
    \begin{subfigure}{0.495\linewidth}
        \input{images/TikZ__Low_rank_example_singular_values.tex}
        \caption{}
        \label{fig:Low_rank_example_singular_values}
    \end{subfigure}
\caption{
    (a) Histograms of the relative accuracy of \cref{alg:kpre,alg:kpost} for \glspl{BMC} as in \cref{ex:dot_product_model} with $K=10$, $d=5$, and $v_i\in\mbbR^d$ and $\alpha$ sampled as described in \Cref{sec:low_rank_BMCs}.
    Here, $n=1000$, and the path length $\ell_n=n(\ln n)^{\beta}$ is varied with $\beta= 5.0, 3.5, 2.0$ from top to bottom.
    Each histogram is the result of $500$ independent repetitions.
    (b) Histograms of the empirical singular value distribution of $\hat{N} / \gamma_n$ where $\gamma_n = (\ell_n/n)^{3/4}$ for some random \gls{BMC} from \Cref{fig:Low_rank_example_relative_accuracy} while varying the path length as in (a).
    The red dots indicate the location of the five largest singular values.
    The solid red curve represents the theoretic prediction for the limiting distribution from \cite[Proposition 7.6]{vanwerde2023matrix}.
    The dashed vertical line indicates the location of the threshold $\gamma_n$ which remains fixed due to the rescaling.
}
\label{fig:Low_rank_example}
\end{figure}

\paragraph{Singular value thresholding detects the rank, not the number of clusters}
\Cref{fig:Low_rank_example_relative_accuracy} shows histograms of the empirical distribution of $\Ktilde$ from \Cref{alg:kpre} and $\Khat$ from \Cref{alg:kpost} when randomly sampling \glspl{BMC} from the dot-product model in \cref{ex:dot_product_model}.
\cref{sec:low_rank_BMCs} contains details on the sampling procedure, and in particular, \Cref{tab:ensemble_characterizations_low_rank} indicates that \glspl{BMC} sampled in this way typically have small values of $I(\alpha, p)$.
We are thus dealing with relatively difficult clustering problems.
Observe now from \Cref{fig:Low_rank_example_relative_accuracy} that as $\beta = 2.0, 3.5, 5.0$ increases from bottom to top, and thus the associated path length $\ell_n = n(\ln n)^\beta$ increases, the two estimators start to concentrate around $\rank(p) = 5$ and $K = 10$, respectively, slowly.

\Cref{fig:Low_rank_example_singular_values} also shows how the distribution of singular values of $\hat{N}$ changes as the ratio $\ell_n/n$ grows.
It shows the empirical distribution of singular values for a single realization of $\hat{N} / \gamma_n$, where $\gamma_n = (\ell_n/n)^{3/4}$, and at the three path lengths considered in \Cref{fig:Low_rank_example_relative_accuracy}.
Observe that with the short path length $(\beta=2)$, only the largest singular value is separated well from the bulk.
Moreover, the edge of the bulk exceeds the threshold value.
On the contrary, at longer path lengths $(\beta=3.5,5)$, the location of the spectral edge decreases relative to the threshold.
Up to $\rank(p) = 5$ singular values are seen to emerge from the bulk.

\subsection{Sensitivity of the algorithm to the embedding dimension}
\label{sec:dominance_of_singular_value_thresholding}

\Cref{alg:kpost} allows specifying the embedding dimension $2r$, and we will next investigate its effect.
Note that \Cref{thm:Khat=K} shows that \Cref{alg:kpost} is consistent when setting $r = \Ktilde$ obtained from \Cref{alg:kpre}, but here we uncouple from this choice.
\Cref{fig:Relative_accuracy_vs_rank} shows histograms of the empirical distribution of $\Khat$ when we run \cref{alg:kpost} using $r \in \{ 5, 10,15 \}$ on randomly sampled \glspl{BMC} with $K = 10$.

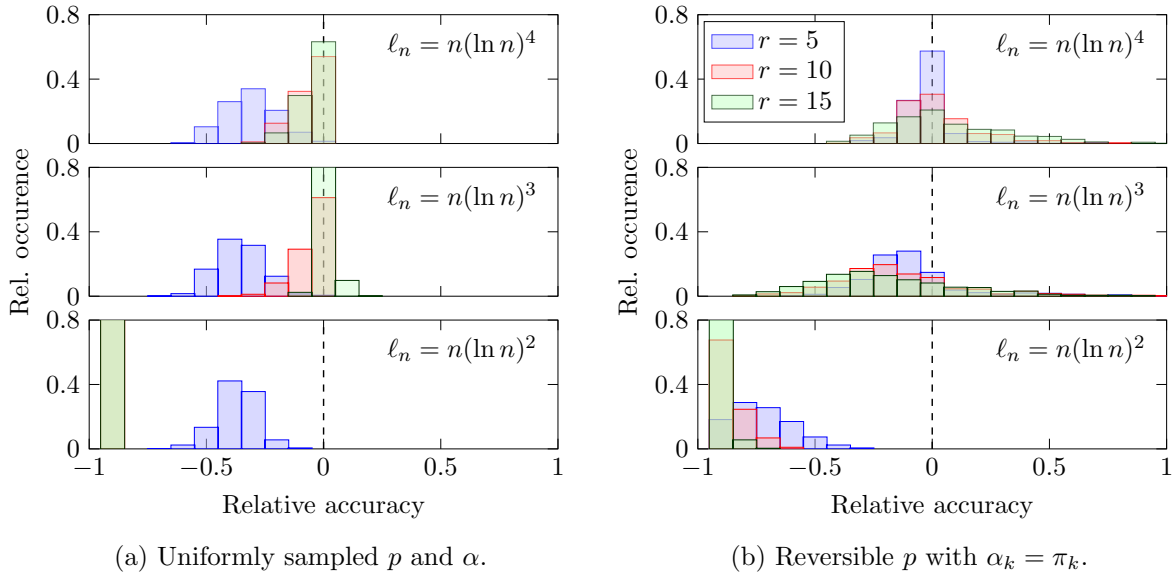
\begin{figure}[htb]
    \centering
    \begin{subfigure}{0.495\linewidth}
        % This file was created with tikzplotlib v0.10.1.
\begin{tikzpicture}[scale=0.9,]
    \begin{axis}[
        name = third plot,
        % at={(second plot.below south west)},%yshift=-0.2\textwidth,       
        % yshift=0.4\textwidth, 
        width=240pt,height=99pt,
        % width=240pt,height=0.4\textwidth,
        % tick align=inside,
        tick pos=both,
        xlabel={Relative accuracy},
        xmin=-1, xmax=1,
        xtick style={color=black},
        y grid style={darkgray176}, 
        ymin=0, ymax=0.8,
        ytick distance=0.4,
        ytick style={color=black}
        ]
        
        \path [draw=black, semithick, dashed]
        (axis cs:0,0)
        --(axis cs:0,1);

        \node[anchor=north east] at (axis cs: 0.95,.76) {$\ell_n = n (\ln n)^2$};
        \addplot[ybar, draw=blue, fill = blue!30!white, fill opacity=0.5] coordinates{
            (-0.7, 0.002)	(-0.6, 0.024)	(-0.5, 0.134)	(-0.4, 0.422)	(-0.3, 0.356)	(-0.2, 0.056)	(-0.1, 0.006)	
        };
        \addplot[ybar, draw=red, fill = red!30!white, fill opacity=0.5] coordinates{
            (-0.9, 1.0)	
        };
        \addplot[ybar, draw=green!20!black, fill = green!30!white, fill opacity=0.5] coordinates{
            (-0.9, 1.0)	
        };  
           
    \end{axis}

    \begin{axis}[
        name = second plot,
        % at={(third plot.above south west)},%yshift=-0.2\textwidth,
        width=240pt,height=99pt,
        yshift=64pt,
        % width=\axisdefaultwidth,height=0.4\textwidth,
        % tick align=inside,
        tick pos=both,
        xmin=-1, xmax=1,
        xtick style={color=black},
        xticklabel=\empty,
        y grid style={darkgray176},
        ylabel={Rel. occurence},
        ymin=0, ymax=0.8,
        ytick distance=0.4,
        ytick style={color=black}
        ]

        \path [draw=black, semithick, dashed]
        (axis cs:0,0)
        --(axis cs:0,1);
        \node[anchor=north east] at (axis cs: 0.95,.76) {$\ell_n = n (\ln n)^3$};
        \addplot[ybar, draw=blue, fill = blue!30!white, fill opacity=0.5] coordinates{
            (-0.7, 0.002)	(-0.6, 0.016)	(-0.5, 0.168)	(-0.4, 0.354)	(-0.3, 0.316)	(-0.2, 0.124)	(-0.1, 0.014)	(0.0, 0.006)	
        };
        \addplot[ybar, draw=red, fill = red!30!white, fill opacity=0.5] coordinates{
            (-0.4, 0.002)	(-0.3, 0.012)	(-0.2, 0.082)	(-0.1, 0.292)	(0.0, 0.612)	
        };
        \addplot[ybar, draw=green!20!black, fill = green!30!white, fill opacity=0.3] coordinates{
            (-0.1, 0.024)	(0.0, 0.876)	(0.1, 0.098)	(0.2, 0.002)	
        };
        % \addplot[ybar, bar width=5pt, draw=red, fill opacity=0.5, dashed] coordinates{
        %     (-0.2, 0.024)	(-0.1, 0.202)	(0.0, 0.77)	(0.1, 0.004)	
        % };
    \end{axis}

    \begin{axis}[
    name = first plot,
    % at={(second plot.above south west)},%yshift=-0.2\textwidth,
    % width=\axisdefaultwidth,height=0.4\textwidth,
    width=240pt,height=99pt,
    yshift=128pt,
    tick pos=both,
    xmin=-1, xmax=1,
    xtick style={color=black},
    xticklabel=\empty,
    y grid style={darkgray176},
    ymin=0, ymax=0.8,
    ytick distance=0.4,
    ytick style={color=black},
    legend style={at={(0.025, 0.95)}, anchor=north west},
    area legend
    ]
    
    \path [draw=black, semithick, dashed]
    (axis cs:0,0)
    --(axis cs:0,1);

    \node[anchor=north east] at (axis cs: 0.95,.76) {$\ell_n = n (\ln n)^4$};
        \addplot[ybar, draw=blue, fill = blue!30!white, draw opacity = 0.7, fill opacity=0.4] coordinates{
            (-0.6, 0.006)	(-0.5, 0.104)	(-0.4, 0.26)	(-0.3, 0.34)	(-0.2, 0.206)	(-0.1, 0.07)	(0.0, 0.014)	
        };
        \addplot[ybar, draw=red, fill = red!30!white, draw opacity = 0.7, fill opacity=0.4] coordinates{
            (-0.3, 0.01)	(-0.2, 0.126)	(-0.1, 0.324)	(0.0, 0.54)	
        };
        \addplot[ybar, draw=green!20!black, fill = green!30!white, draw opacity = 0.7, fill opacity=0.4] coordinates{
            (-0.3, 0.004)	(-0.2, 0.066)	(-0.1, 0.298)	(0.0, 0.632)	
        };
    \end{axis}
    
\end{tikzpicture}
    
        \caption{Uniformly sampled $p$ and $\alpha$.}\label{fig:Relative_accuracy_vs_rank_uniform}
    \end{subfigure}
    \begin{subfigure}{0.495\linewidth}
        % This file was created with tikzplotlib v0.10.1.
\begin{tikzpicture}[scale=0.9,]
    \begin{axis}[
        name = third plot,
        % at={(second plot.below south west)},%yshift=-0.2\textwidth,       
        % yshift=0.4\textwidth, 
        width=240pt,height=99pt,
        % width=240pt,height=0.4\textwidth,
        % tick align=inside,
        tick pos=both,
        xlabel={Relative accuracy},
        xmin=-1, xmax=1,
        xtick style={color=black},
        y grid style={darkgray176}, 
        ymin=0, ymax=0.8,
        ytick distance=0.4,
        ytick style={color=black}
        ]
        
        \path [draw=black, semithick, dashed]
        (axis cs:0,0)
        --(axis cs:0,1);

        \node[anchor=north east] at (axis cs: 0.95,.76) {$\ell_n = n (\ln n)^2$};
        \addplot[ybar, draw=blue, fill = blue!30!white, fill opacity=0.5] coordinates{
            (-0.9, 0.182)	(-0.8, 0.288)	(-0.7, 0.256)	(-0.6, 0.17)	(-0.5, 0.074)	(-0.4, 0.024)	(-0.3, 0.006)	
        };
        \addplot[ybar, draw=red, fill = red!30!white, fill opacity=0.5] coordinates{
            (-0.9, 0.676)	(-0.8, 0.246)	(-0.7, 0.068)	(-0.6, 0.01)	
        };
        \addplot[ybar, draw=green!20!black, fill = green!30!white, fill opacity=0.5] coordinates{
            (-0.9, 0.94)	(-0.8, 0.056)	(-0.7, 0.004)	
        };  
           
    \end{axis}

    \begin{axis}[
        name = second plot,
        % at={(third plot.above south west)},%yshift=-0.2\textwidth,
        width=240pt,height=99pt,
        yshift=64pt,
        % width=\axisdefaultwidth,height=0.4\textwidth,
        % tick align=inside,
        tick pos=both,
        xmin=-1, xmax=1,
        xtick style={color=black},
        xticklabel=\empty,
        y grid style={darkgray176},
        ylabel={Rel. occurence},
        ymin=0, ymax=0.8,
        ytick distance=0.4,
        ytick style={color=black}
        ]

        \path [draw=black, semithick, dashed]
        (axis cs:0,0)
        --(axis cs:0,1);
        \node[anchor=north east] at (axis cs: 0.95,.76) {$\ell_n = n (\ln n)^3$};
        \addplot[ybar, draw=blue, fill = blue!30!white, fill opacity=0.5] coordinates{
            (-0.6, 0.002)	(-0.5, 0.012)	(-0.4, 0.054)	(-0.3, 0.104)	(-0.2, 0.256)	(-0.1, 0.28)	(0.0, 0.148)	(0.1, 0.038)	(0.2, 0.022)	(0.3, 0.012)	(0.4, 0.02)	(0.5, 0.018)	(0.6, 0.012)	(0.7, 0.004)	(0.8, 0.008)	(0.9, 0.004)	(1.0, 0.002)	(1.1, 0.004)	
        };
        \addplot[ybar, draw=red, fill = red!30!white, fill opacity=0.5] coordinates{
            (-0.8, 0.004)	(-0.7, 0.014)	(-0.6, 0.022)	(-0.5, 0.056)	(-0.4, 0.094)	(-0.3, 0.172)	(-0.2, 0.196)	(-0.1, 0.138)	(0.0, 0.116)	(0.1, 0.046)	(0.2, 0.044)	(0.3, 0.014)	(0.4, 0.03)	(0.5, 0.016)	(0.6, 0.012)	(0.7, 0.006)	(0.8, 0.004)	(0.9, 0.006)	(1.0, 0.004)	(1.1, 0.002)	(1.2, 0.002)	(1.3, 0.002)	
        };
        \addplot[ybar, draw=green!20!black, fill = green!30!white, fill opacity=0.5] coordinates{
            (-0.8, 0.008)	(-0.7, 0.028)	(-0.6, 0.06)	(-0.5, 0.092)	(-0.4, 0.136)	(-0.3, 0.154)	(-0.2, 0.128)	(-0.1, 0.102)	(0.0, 0.082)	(0.1, 0.056)	(0.2, 0.054)	(0.3, 0.03)	(0.4, 0.032)	(0.5, 0.012)	(0.6, 0.006)	(0.7, 0.006)	(0.8, 0.006)	(0.9, 0.004)	(1.1, 0.004)	
        };
    \end{axis}

    \begin{axis}[
    name = first plot,
    % at={(second plot.above south west)},%yshift=-0.2\textwidth,
    % width=\axisdefaultwidth,height=0.4\textwidth,
    width=240pt,height=99pt,
    yshift=128pt,
    tick pos=both,
    xmin=-1, xmax=1,
    xtick style={color=black},
    xticklabel=\empty,
    y grid style={darkgray176},
    ymin=0, ymax=0.8,
    ytick distance=0.4,
    ytick style={color=black},
    legend style={at={(0.0125, 0.95)}, anchor=north west},
    legend cell align={left},
    area legend
    ]
    
    \path [draw=black, semithick, dashed]
    (axis cs:0,0)
    --(axis cs:0,1);

    \node[anchor=north east] at (axis cs: 0.95,.76) {$\ell_n = n (\ln n)^4$};
        \addplot[ybar, draw=blue, fill = blue!30!white, draw opacity = 0.7, fill opacity=0.4] coordinates{
            (-0.3, 0.018)	(-0.2, 0.036)	(-0.1, 0.268)	(0.0, 0.574)	(0.1, 0.062)	(0.2, 0.012)	(0.3, 0.008)	(0.4, 0.01)	(0.5, 0.002)	(0.6, 0.004)	(0.7, 0.006)	
        };
        \addlegendentry{$r=5$}
        \addplot[ybar, draw=red, fill = red!30!white, draw opacity = 0.7, fill opacity=0.4] coordinates{
            (-0.4, 0.002)	(-0.3, 0.036)	(-0.2, 0.066)	(-0.1, 0.266)	(0.0, 0.306)	(0.1, 0.154)	(0.2, 0.062)	(0.3, 0.056)	(0.4, 0.018)	(0.5, 0.018)	(0.6, 0.006)	(0.7, 0.004)	(0.8, 0.006)	
        };
        \addlegendentry{$r=10$}
        \addplot[ybar, draw=green!20!black, fill = green!30!white, draw opacity = 0.7, fill opacity=0.4] coordinates{
            (-0.4, 0.014)	(-0.3, 0.052)	(-0.2, 0.128)	(-0.1, 0.166)	(0.0, 0.208)	(0.1, 0.12)	(0.2, 0.088)	(0.3, 0.084)	(0.4, 0.046)	(0.5, 0.044)	(0.6, 0.026)	(0.7, 0.01)	(0.9, 0.008)	(1.0, 0.002)	(1.1, 0.004)
        };
        \addlegendentry{$r=15$}
    \end{axis}
    
\end{tikzpicture}
    
        \caption{Reversible $p$ with $\alpha_k=\pi_k$.}\label{fig:Relative_accuracy_vs_rank_reversible}
    \end{subfigure}
    \caption{%
        Histograms of relative accuracy when $r \in \{ 5, 10, 15 \}$ and $\ell_n = n (\ln n)^{\beta}$ with $\beta \in \{ 2, 3, 4 \}$, for uniformly sampled \glspl{BMC} and reversible \glspl{BMC}; see \Cref{sec:uniform_BMCs,sec:reversible_BMCs}.
        Here, $n=1000$, $K=10$, and each histogram is the result of $500$ independent replications.
    }
    \label{fig:Relative_accuracy_vs_rank}
\end{figure}

\paragraph{Choosing \texorpdfstring{$r$}{r} too large deteriorates performance}
Observe from \Cref{fig:Relative_accuracy_vs_rank_uniform}, in which both $p$ and $\alpha$ are sampled uniformly from the simplex, that performance deteriorates when $r$ exceeds the dimensionality of the signal $\rank(N)=10$.
Setting $r$ too large risks adding noninformative singular vectors to the embedding $\hat{X}$, which perturbs the density-based clustering done by \Cref{alg:kpost} by decreasing the average density of the clusters in $\mbbR^{2r}$; see also \Cref{sec:cluster-density-vs-rank}.
In sparse regimes $(\beta=2)$, the result is that few clusters achieve sufficiently high density to be detected by \Cref{alg:kpost}, and consequently then, $\Khat<K$.
In dense regimes $(\beta=4)$, $\hat{K}$ sometimes overestimates $K$ instead.
In this case, adding noise to $\hat{X}$ disrupts the true clusters, but smaller areas of high density remain, which \Cref{alg:kpost} detects as distinct clusters.

\paragraph{Choosing \texorpdfstring{$r$}{r} too small can be beneficial, even for long path lengths}
\Cref{fig:Relative_accuracy_vs_rank_uniform} also shows that choosing $r$ smaller than the dimensionality of the signal can be beneficial.
In sparse regimes $(\beta=2)$, the signal-to-noise ratio of additional features might not be sufficient for these to have a positive impact on performance, even if these features are informative.
Such a phenomenon is well-known in the classification literature (see e.g., \cite[Section 3]{jain2000statistical}), where it is understood in terms of a bias--variance tradeoff.

Perhaps more surprisingly, this behavior persists in relatively dense regimes $(\beta=4)$.
Here, \Cref{alg:kpost} often recovers the correct number of clusters even when $r=5$.
We suspect that this is because we use both left- and right-singular vectors, which effectively doubles the embedding dimension.
We also suspect that a uniformly sampled \gls{BMC} exhibits sufficient variation between clusters to distinguish these using fewer than $K$ singular vectors.
Indeed, in the extreme case, clustering could still be performed based solely on the visitation rate of each state, provided that $\Pi_x$ varies strongly accross different clusters and there are sufficient observations.

\paragraph{Choosing \texorpdfstring{$r$}{r} too small is not always beneficial}
To confirm our suspicions, we repeat the experiment from \Cref{fig:Relative_accuracy_vs_rank_uniform} using randomly sampled reversible \glspl{BMC} while setting $\alpha_k = \pi_k$ for $k\in [K]$.
For such reversible \glspl{BMC}, the left- and right-singular vectors of $N$ coincide as it is symmetric; and $\Pi_x = 1/n$ for $x \in [n]$, so that clusters cannot be detected based on their visitation rates alone.
The results of \Cref{fig:Relative_accuracy_vs_rank_reversible} show that while $r=5$ still yields the best performance in sparse regimes ($\beta=2$), it is outperformed by the other choices of $r$ when the path length is increased ($\beta=3,4$).

\subsection{Sparse data and the effect of mixing times}
\label{sec:sparse_example}

To demonstrate that \Cref{alg:kpost} recovers $K$ for sufficiently large $n$ even in the sparse regime $\omega(n)=\ell_n=o(n\ln n)$, we consider the one-parameter family of \glspl{BMC} with $K=2$, $p=(p_0, 1-p_0; 1-p_0, p_0)$, and $\alpha=(1/2, 1/2)$, and examine the performance for path lengths $\ell_n=n (\ln n)^{\beta}$ with exponent $\beta\leq 1$.
Note that the clusters are indistinguishable ($I(\alpha,p)=0$) when $p_0=1/2$, and that the \gls{MC} is not irreducible when $p_0=1$.
We therefore vary $p_0\in (1/2, 1)$ without loss of generality.
We also set $\gamma_n=(\ell_n/n)^c$ with $c=0.95$, which is necessary to obtain nontrivial results at reasonable values of $n$.

\Cref{fig:Sparse_example} shows scatter plots of pairs $(\beta,p_0)$ for which the correct number of clusters is detected at least half of the time.
To disentangle the results of \Cref{alg:kpre,alg:kpost}, we consider three different methods for estimating $K$: (a) \Cref{alg:kpost} with $r$ set using \Cref{alg:kpre}, (b) \Cref{alg:kpost} with $r$ fixed at 2, and (c) \Cref{alg:kpre}, corresponding to \Cref{fig:Sparse_example_a,fig:Sparse_example_b,fig:Sparse_example_c}, respectively.
For each method $i\in\{\mathrm{a},\mathrm{b},\mathrm{c}\}$, we denote the set of pairs $(\beta,p_0)$ for which the correct number of clusters is detected at least half of the time by $\hat{\mcC}^{(i)}_n\subset (0,1)\times(1/2,1)$, and denote its restriction to fixed $\beta\in(0,1)$ by $\hat{\mcC}^{(i)}_n|_{\beta}$.

\input{images/TikZ__Sparse_example.tex}

\paragraph{Numerically also, \Cref{alg:kpost} appears asymptotically consistent}
Observe from \Cref{fig:Sparse_example_a} that the sets $\hat{\mcC}_n^{(\mathrm{a})}$ increase in size with increasing $n$, reflecting the asymptotic consistency of \Cref{alg:kpost}.
Similar trends are observed in \Cref{fig:Sparse_example_b,fig:Sparse_example_c}.

\paragraph{Here, \Cref{alg:kpost} performs no better than \Cref{alg:kpre}}
\Cref{fig:Sparse_example_c} shows that the set $\hat{\mcC}_n^{(\mathrm{a})}$ is a subset of $\hat{\mcC}_n^{(\mathrm{c})}$, indicating that \Cref{alg:kpost} offers no advantage in this particular scenario.
This may be explained as follows.
Note from \Cref{fig:Sparse_example_c} that \Cref{alg:kpre} returns values greater than 2 when $\beta$ is small, and recall from \cref{sec:dominance_of_singular_value_thresholding} that if \Cref{alg:kpost} runs with an embedding dimension that is too large, it performs poorly when observations are sparse.
Finally, note from \Cref{fig:Sparse_example_c} that \Cref{alg:kpre} outputs values less than 2 when $p_0$ is small, and recall that \Cref{alg:kpost} can underperform when $r$ is set too small if $p$ is reversible and $\Pi_x$ is constant (which is the case here).

\paragraph{Shorter path lengths require larger values of $I(\alpha,p)$}
Closer inspection of \Cref{fig:Sparse_example_a,fig:Sparse_example_b} shows that for each $n$, the sets $\hat{\mcC}_n^{(\mathrm{a})}|_{\beta}$ and $\hat{\mcC}_n^{(\mathrm{b})}|_{\beta}$ shrink in size as $\beta$ decreases.
This is expected because sparser observations (smaller $\beta$) demand easier clustering problems (larger $p_0$) for our algorithm to succeed; note that here, $I(\alpha, p)=4(p_0-1/2)\ln p_0/(1-p_0)$ is increasing for $p_0\in(1/2,1)$.
This trend is more pronounced for (b), where $r$ is fixed, than for (a).

\paragraph{Singular value thresholding causes deviations from the trend}
Observe from \Cref{fig:Sparse_example_c} that $\hat{\mcC}_n^{(\mathrm{c})}|_{\beta}$ shows the opposite trend, growing in size as $\beta$ decreases before vanishing abruptly.
This is further investigated in \Cref{sec:singular-values-in-sparse-regimes}, where we plot the first 3 singular values of $\hat{N}_{\Gamma}$.
We find that at the finite values of $n$ considered here, the singular values deviate considerably from the asymptotic scaling predicted by \eqref{eqn:singvalscaling}.
Because the performance of \Cref{alg:kpost} is limited by that of \Cref{alg:kpre} when we use the latter to set $r$, this also leads to deviations from the expected trend that is found when we fix $r=2$.

\paragraph{Divergent mixing times perturb the spectral estimator}
We finally highlight a subtle yet interesting effect.
Specifically, observe in the upper left corner of \Cref{fig:Sparse_example_a,fig:Sparse_example_c}, but not \Cref{fig:Sparse_example_b}, that the estimators fail at least half of the time when $\beta$ is small and $p_0$ is large.
We suspect that here, the dependencies within the trajectory affect the spectrum of $\hat{N}_{\Gamma}$ strongly.
Indeed, the mixing time of the \gls{BMC}, here given by
$
    t_{\mix}
    \eqdef
    \inf\{t \geq 1:\max_{x\in[n]}d_{TV}(\Pi,P^t_{x,\cdot})<1/4\}
    =
    \ceil{\ln(2)/\ln(1/(2p_0-1))}
$,
diverges as $p_0\rightarrow 1$.

\subsection{Relations between relative performance and different characteristics of \texorpdfstring{\glspl{BMC}}{BMCs}}

\subsubsection{Relative performance vs information quantity}

In this experiment, we investigate whether the relative accuracy of our method is correlated with $I(\alpha, p)$.
\Cref{fig:Relative_accuracy_vs_information_quantity_uniform} shows plots of different quality measures as a function of $I(\alpha, p)$, for \glspl{BMC} sampled uniformly at random.

\begin{figure}[htb]
    \centering
    \begin{subfigure}{0.495\linewidth}
        \begin{tikzpicture}[scale = .9]
    \begin{axis}[
        xmin = 0.2, xmax = 1.2,
        ymin = -1, ymax = 1,
        width = \linewidth,%*0.65
        height = 0.74 * \linewidth,%*0.65
        ytick distance=0.5,
        xlabel = {$I(\alpha, p)$}, ylabel = {Relative accuracy},
        legend cell align={left},
        ]
        
        \addplot[forget plot] graphics [xmin=0.2,xmax=1.2,ymin=-.94,ymax=1.025] {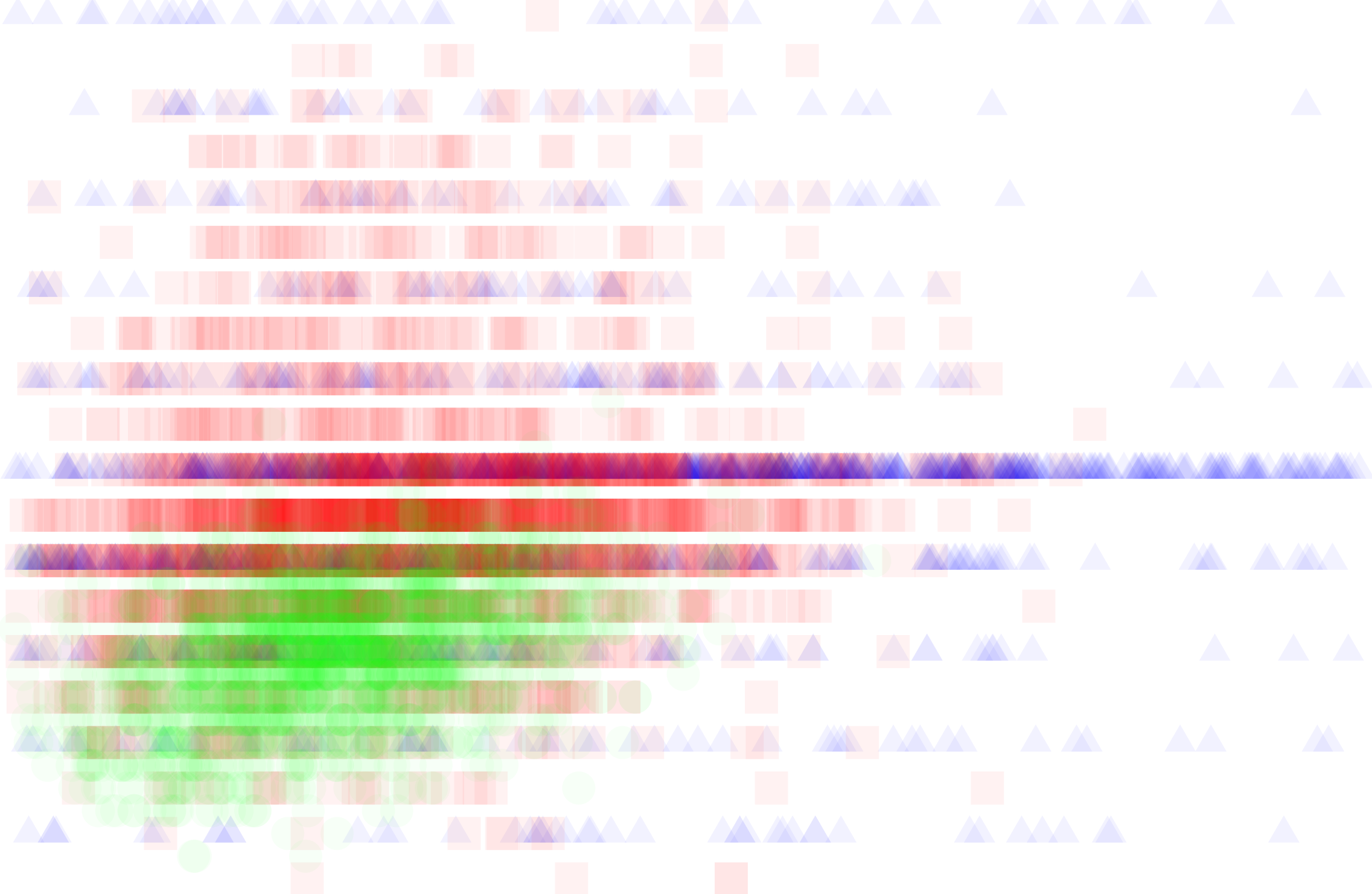};
        
        \addplot[smooth, mark=triangle*, mark options={scale=1}, error bars/.cd, y dir=both, y explicit] plot coordinates {
            (0.28727579167508666, 0.11330472103004308) += (0.28727579167508666, 0.08610255574655477) -= (0.28727579167508666, 0.08610255574655477)
            (0.4136644049796612, 0.1392265193370166) += (0.4136644049796612, 0.06294253816831756) -= (0.4136644049796612, 0.06294253816831756)
            (0.5400530182842358, 0.022777777777777775) += (0.5400530182842358, 0.04973648914567375) -= (0.5400530182842358, 0.04973648914567375)
            (0.6664416315888104, 0.12105263157894741) += (0.6664416315888104, 0.055321751024918195) -= (0.6664416315888104, 0.055321751024918195)
            (0.7928302448933849, 0.08395061728395065) += (0.7928302448933849, 0.06515255141581736) -= (0.7928302448933849, 0.06515255141581736)
            (0.9192188581979595, 0.005161290322580636) += (0.9192188581979595, 0.061632509601157986) -= (0.9192188581979595, 0.061632509601157986)
            (1.045607471502534, 0.10129870129870133) += (1.045607471502534, 0.10957249074825011) -= (1.045607471502534, 0.10957249074825011)
            (1.1719960848071087, 0.07272727272727272) += (1.1719960848071087, 0.13324282427665837) -= (1.1719960848071087, 0.13324282427665837)
        };
        \addlegendentry{$K=5$, blue}

        \addplot[smooth, mark=square*, mark options={scale=1}, error bars/.cd, y dir=both, y explicit] plot coordinates {
            (0.305447184958787, -0.1676646706586825) += (0.305447184958787, 0.03914688473117906) -= (0.305447184958787, 0.03914688473117906)
            (0.3640294273584843, -0.1197718631178708) += (0.3640294273584843, 0.03640071664334112) -= (0.3640294273584843, 0.03640071664334112)
            (0.42261166975818154, -0.04729729729729726) += (0.42261166975818154, 0.03508699629695592) -= (0.42261166975818154, 0.03508699629695592)
            (0.4811939121578788, -0.07808641975308639) += (0.4811939121578788, 0.031278534177487036) -= (0.4811939121578788, 0.031278534177487036)
            (0.5397761545575761, -0.06300366300366303) += (0.5397761545575761, 0.03661984477353615) -= (0.5397761545575761, 0.03661984477353615)
            (0.5983583969572733, -0.09647577092511009) += (0.5983583969572733, 0.036507301131517886) -= (0.5983583969572733, 0.036507301131517886)
            (0.6569406393569707, -0.04967741935483871) += (0.6569406393569707, 0.036926217763575256) -= (0.6569406393569707, 0.036926217763575256)
            (0.7155228817566679, -0.050526315789473676) += (0.7155228817566679, 0.05724091567380778) -= (0.7155228817566679, 0.05724091567380778)
        };
        \addlegendentry{$K=10$, red}

        \addplot[smooth, mark=*, mark options={scale=1}, error bars/.cd, y dir=both, y explicit] plot coordinates {
            (0.3036289703028093, -0.48694029850746257) += (0.3036289703028093, 0.02539623864543812) -= (0.3036289703028093, 0.02539623864543812)
            (0.3442016910973566, -0.4573394495412845) += (0.3442016910973566, 0.019111071489278143) -= (0.3442016910973566, 0.019111071489278143)
            (0.3847744118919039, -0.4117132867132865) += (0.3847744118919039, 0.01689583627099578) -= (0.3847744118919039, 0.01689583627099578)
            (0.42534713268645125, -0.39709302325581375) += (0.42534713268645125, 0.013806833646369742) -= (0.42534713268645125, 0.013806833646369742)
            (0.46591985348099857, -0.39334470989761117) += (0.46591985348099857, 0.014400943074975425) -= (0.46591985348099857, 0.014400943074975425)
            (0.506492574275546, -0.34220532319391644) += (0.506492574275546, 0.01715933931837463) -= (0.506492574275546, 0.01715933931837463)
            (0.5470652950700932, -0.33749999999999997) += (0.5470652950700932, 0.020790173241673614) -= (0.5470652950700932, 0.020790173241673614)
            (0.5876380158646406, -0.300952380952381) += (0.5876380158646406, 0.0246577040774353) -= (0.5876380158646406, 0.0246577040774353)
            };
        \addlegendentry{$K=20$, green}

        \addplot[dashed, domain = 0.2:1.4, samples = 2] {0.0};

    \end{axis}
\end{tikzpicture}
        \caption{}\label{fig:Relative_accuracy_vs_information_quantity_uniform_scatter}
    \end{subfigure}
    \begin{subfigure}{0.495\linewidth}
        \begin{tikzpicture}[scale = .9]
    \begin{axis}[
        xmin = 0.2, xmax = 1.2,
        ymin = 0.5, ymax = 1.025,
        width = \linewidth,%*0.65
        height = 0.74 * \linewidth,%*0.65
        xlabel = {$I(\alpha, p)$}, ylabel = {AMI},
        ]
        
        \addplot[forget plot] graphics [xmin=0.2,xmax=1.2,ymin=0.5,ymax=1.01] {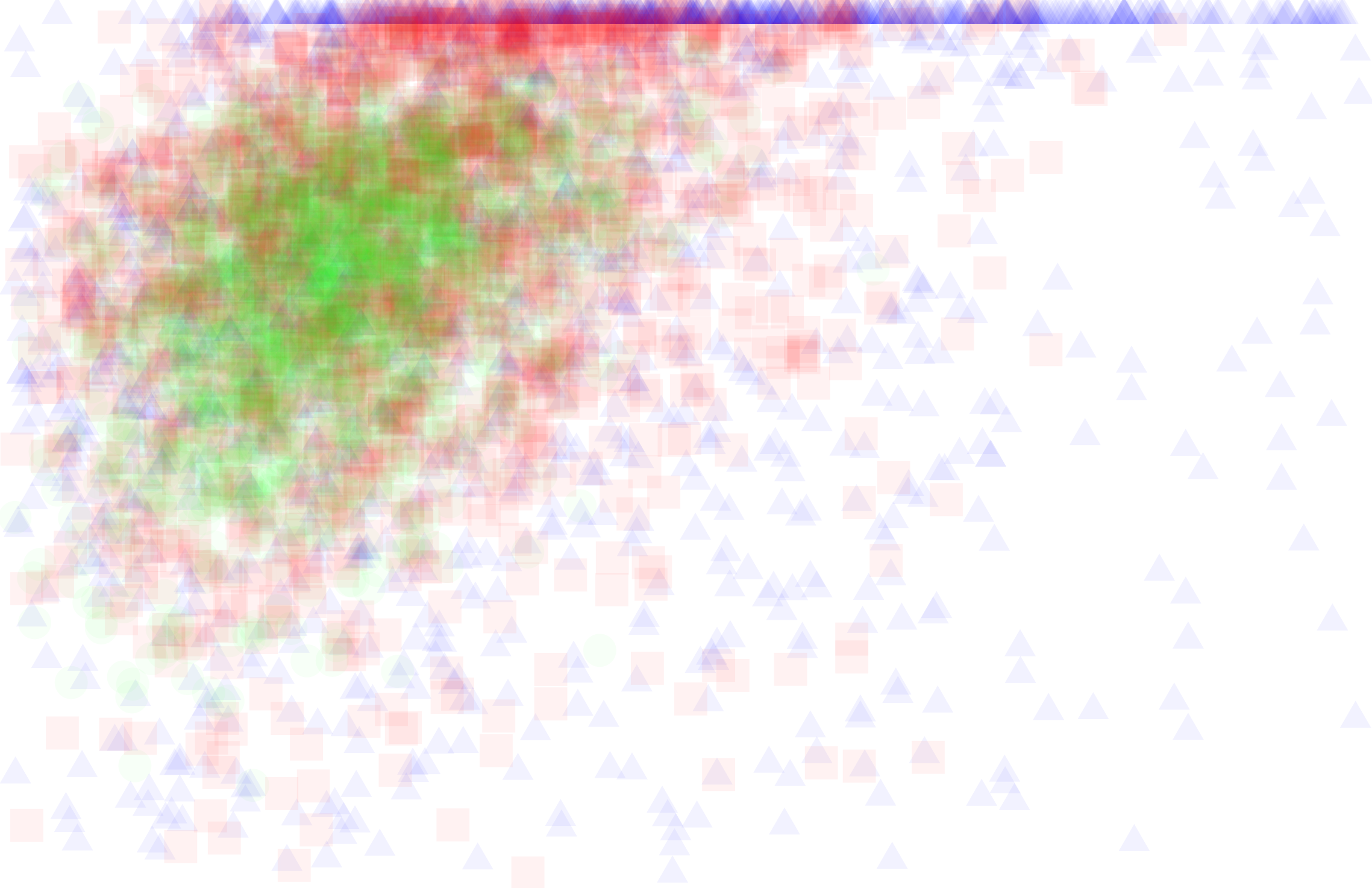};
        
        \addplot[smooth, mark=triangle*, mark options={scale=1}, x, error bars/.cd, y dir=both, y explicit] plot coordinates {
            (0.2908918426122678, 0.7517894165353962) += (0.2908918426122678, 0.024063973630038423) -= (0.2908918426122678, 0.024063973630038423)
            (0.38170394119146306, 0.7983790319007933) += (0.38170394119146306, 0.02742085964484935) -= (0.38170394119146306, 0.02742085964484935)
            (0.4573352858426776, 0.8366889407904772) += (0.4573352858426776, 0.02246016555691746) -= (0.4573352858426776, 0.02246016555691746)
            (0.5313821352593849, 0.8392211394614907) += (0.5313821352593849, 0.02651886429549034) -= (0.5313821352593849, 0.02651886429549034)
            (0.6100672916377536, 0.87150061296833) += (0.6100672916377536, 0.026046269461098028) -= (0.6100672916377536, 0.026046269461098028)
            (0.6935651017314375, 0.8772235768773897) += (0.6935651017314375, 0.02552727941312928) -= (0.6935651017314375, 0.02552727941312928)
            (0.795875707939719, 0.857817311137545) += (0.795875707939719, 0.03027141395084488) -= (0.795875707939719, 0.03027141395084488)
            (0.9587206407702156, 0.8726554921270026) += (0.9587206407702156, 0.02965188296992154) -= (0.9587206407702156, 0.02965188296992154)
        };
        \addplot[smooth, mark=square*, mark options={scale=1}, error bars/.cd, y dir=both, y explicit] plot coordinates {
            (0.3090298685891905, 0.8146113407593975) += (0.3090298685891905, 0.013630163109637891) -= (0.3090298685891905, 0.013630163109637891)
            (0.3704973569328202, 0.830517094521941) += (0.3704973569328202, 0.013606308419210743) -= (0.3704973569328202, 0.013606308419210743)
            (0.4166753042536836, 0.8432010450228018) += (0.4166753042536836, 0.013574642697917763) -= (0.4166753042536836, 0.013574642697917763)
            (0.457448449623995, 0.8543259148097532) += (0.457448449623995, 0.012610965515843193) -= (0.457448449623995, 0.012610965515843193)
            (0.4945123314069585, 0.8692694910418501) += (0.4945123314069585, 0.013919766435617787) -= (0.4945123314069585, 0.013919766435617787)
            (0.5352226544754635, 0.8858641816462646) += (0.5352226544754635, 0.014286862832427144) -= (0.5352226544754635, 0.014286862832427144)
            (0.5818022333729282, 0.8814983999956431) += (0.5818022333729282, 0.01309456186240625) -= (0.5818022333729282, 0.01309456186240625)
            (0.6477512235615687, 0.8989628234047604) += (0.6477512235615687, 0.015175832415003454) -= (0.6477512235615687, 0.015175832415003454)
        };
        \addplot[smooth, mark=*, mark options={scale=1}, error bars/.cd, y dir=both, y explicit] plot coordinates {
            (0.3187736800418465, 0.8018726401116827) += (0.3187736800418465, 0.009383851902165672) -= (0.3187736800418465, 0.009383851902165672)
            (0.36184480572665423, 0.8169000143939025) += (0.36184480572665423, 0.009664353161665288) -= (0.36184480572665423, 0.009664353161665288)
            (0.39187096887461814, 0.8229968111744902) += (0.39187096887461814, 0.009147977878055786) -= (0.39187096887461814, 0.009147977878055786)
            (0.41890054620456807, 0.8415961802683846) += (0.41890054620456807, 0.008474517047641354) -= (0.41890054620456807, 0.008474517047641354)
            (0.4439310659051068, 0.8393564035009619) += (0.4439310659051068, 0.00849795779823968) -= (0.4439310659051068, 0.00849795779823968)
            (0.4705751273374412, 0.8508798192615301) += (0.4705751273374412, 0.008670258571180799) -= (0.4705751273374412, 0.008670258571180799)
            (0.5009045902821977, 0.8591474303920298) += (0.5009045902821977, 0.008370090563388437) -= (0.5009045902821977, 0.008370090563388437)
            (0.5441010790037029, 0.8724968266884788) += (0.5441010790037029, 0.007815961480975084) -= (0.5441010790037029, 0.007815961480975084)
            };
        \addplot[dashed, domain = 0.0:1.2, samples = 2] {1.0};
        \coordinate (insetPosition) at (axis cs:0.65, 0.505);
    \end{axis}
    % inset
    \fill [white] (insetPosition) rectangle ++(3cm,1.9cm); % white box below insert
    \begin{axis}[
            at={(insetPosition)},
            reverse legend,
            xmin = 0.2, xmax = 1.0,
            ymin = 0.85, ymax = 1.05,
            ytick distance=0.1,
            xticklabel=\empty,
            label style={font=\small},
            ylabel = {$\textnormal{AMI}_{\hat{K}}/\textnormal{AMI}_{K}$},
            y label style={at={(axis description cs:-0.17,0.5)},anchor=south},
            width = 4.5cm, height = 3.5cm
        ]

        \addplot[forget plot] graphics [xmin=0.2,xmax=1.0,ymin=.85,ymax=1.05] {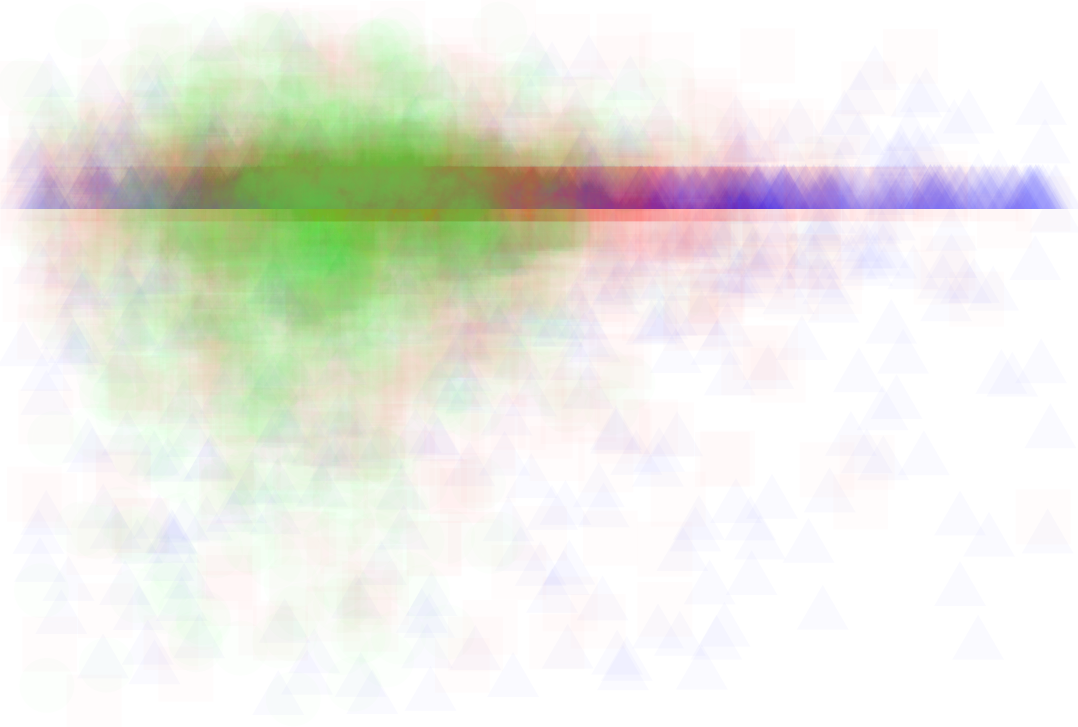};
        
        \addplot[smooth, mark=triangle*, mark options={scale=1}, error bars/.cd, y dir=both, y explicit] plot coordinates {
            (0.2952992558784813, 0.9633499923064263) += (0.2952992558784813, 0.019990830401052914) -= (0.2952992558784813, 0.019990830401052914)
            (0.38257492077133826, 0.9839770810156736) += (0.38257492077133826, 0.030548257524063614) -= (0.38257492077133826, 0.030548257524063614)
            (0.45868532160870645, 0.9794258698390675) += (0.45868532160870645, 0.018131899879869493) -= (0.45868532160870645, 0.018131899879869493)
            (0.5334039990249458, 0.9708018655874602) += (0.5334039990249458, 0.02154512330123887) -= (0.5334039990249458, 0.02154512330123887)
            (0.6060441386037293, 0.9818360020518264) += (0.6060441386037293, 0.014656769345658245) -= (0.6060441386037293, 0.014656769345658245)
            (0.6916961757829225, 0.9945595748771299) += (0.6916961757829225, 0.028794459277006666) -= (0.6916961757829225, 0.028794459277006666)
            (0.7936819388900826, 0.9982372777124425) += (0.7936819388900826, 0.02872389802877455) -= (0.7936819388900826, 0.02872389802877455)
            (0.9386900684355579, 0.9500564693591009) += (0.9386900684355579, 0.037228394680792726) -= (0.9386900684355579, 0.037228394680792726)
            };
        \addplot[smooth, mark=square*, mark options={scale=1}, error bars/.cd, y dir=both, y explicit] plot coordinates {
            (0.30962952734652555, 0.9845892762779427) += (0.30962952734652555, 0.011251656684678672) -= (0.30962952734652555, 0.011251656684678672)
            (0.3740434933892527, 0.9865637893823984) += (0.3740434933892527, 0.006370035172493659) -= (0.3740434933892527, 0.006370035172493659)
            (0.4183034249077825, 0.9902980855671457) += (0.4183034249077825, 0.004199317266752652) -= (0.4183034249077825, 0.004199317266752652)
            (0.4569751558752452, 0.9861698353878955) += (0.4569751558752452, 0.005826959481681366) -= (0.4569751558752452, 0.005826959481681366)
            (0.49412541402503585, 0.9932252872242001) += (0.49412541402503585, 0.005039266069856323) -= (0.49412541402503585, 0.005039266069856323)
            (0.5319086452929576, 0.9885712156782853) += (0.5319086452929576, 0.00513116705751647) -= (0.5319086452929576, 0.00513116705751647)
            (0.5773354223519136, 0.9875053156131992) += (0.5773354223519136, 0.005472870330655755) -= (0.5773354223519136, 0.005472870330655755)
            (0.6429484720742813, 0.9915186794879712) += (0.6429484720742813, 0.0041142323739086425) -= (0.6429484720742813, 0.0041142323739086425)
            };
        \addplot[smooth, mark=*, mark options={scale=1}, error bars/.cd, y dir=both, y explicit] plot coordinates {
            (0.31391068477543504, 0.9752066928574413) += (0.31391068477543504, 0.00654332419141197) -= (0.31391068477543504, 0.00654332419141197)
            (0.3602668692126042, 0.9827525890787693) += (0.3602668692126042, 0.004718189038269056) -= (0.3602668692126042, 0.004718189038269056)
            (0.39374179905253603, 0.9848555594017615) += (0.39374179905253603, 0.005422062616952656) -= (0.39374179905253603, 0.005422062616952656)
            (0.4199930303768238, 0.9857356506201602) += (0.4199930303768238, 0.004646844156767007) -= (0.4199930303768238, 0.004646844156767007)
            (0.44529149142794416, 0.9826016119548013) += (0.44529149142794416, 0.0048780538783057825) -= (0.44529149142794416, 0.0048780538783057825)
            (0.472819506269701, 0.9847264388940455) += (0.472819506269701, 0.005356390451123959) -= (0.472819506269701, 0.005356390451123959)
            (0.5042057976545918, 0.9953614108538854) += (0.5042057976545918, 0.003912475056640676) -= (0.5042057976545918, 0.003912475056640676)
            (0.5449687122733665, 0.9898670589502971) += (0.5449687122733665, 0.003452356853445727) -= (0.5449687122733665, 0.003452356853445727)
            };
        \addplot[dashed, domain = 0.2:1.4, samples = 2] {1.0};
    \end{axis}
\end{tikzpicture}
        \caption{}\label{fig:Mutual_information_vs_information_quantity_uniform_scatter}
    \end{subfigure}
    \begin{subfigure}{\linewidth}
        \begin{tikzpicture}[scale=0.8,]
    \begin{axis}[
        name = third plot,
        % at={(second plot.below south west)},%yshift=-0.2\textwidth,       
        % yshift=0.4\textwidth, 
        width=200pt,height=99pt,
        % width=240pt,height=0.4\textwidth,
        % tick align=inside,
        tick pos=both,
        xlabel={Relative accuracy},
        xmin=-1, xmax=1,
        bar width=3.33333334ex,
        xtick style={color=black},
        y grid style={darkgray176}, 
        % ytick distance=0.5,
        ylabel={Rel. occurence},
        ymin=0, ymax=0.65,
        ytick style={color=black},
        legend style={at={(0.9875, 0.95)}, anchor=north east},
        area legend
        ]
        
        \path [draw=black, semithick, dashed]
        (axis cs:0,0)
        --(axis cs:0,1);

        \node[anchor=north west] at (axis cs: -0.975,.61) {$K=5$};
        \addplot[ybar, semithick, draw=blue, fill = blue!30!white, draw opacity = 0.7, fill opacity=0.4] coordinates{
            (-0.8, 0.02252252252252252)	(-0.6, 0.04054054054054054)	(-0.4, 0.08558558558558559)	(-0.2, 0.2912912912912913)	(0.0, 0.2897897897897898)	(0.2, 0.06306306306306306)	(0.4, 0.024024024024024024)	(0.6, 0.03003003003003003)	(0.8, 0.028528528528528527)	(1.0, 0.02702702702702703)	(1.2, 0.021021021021021023)	(1.4, 0.02252252252252252)	(1.6, 0.015015015015015015)	(1.8, 0.010510510510510511)	(2.0, 0.0075075075075075074)	(2.2, 0.006006006006006006)	(2.4, 0.006006006006006006)	(2.6, 0.003003003003003003)	(2.8, 0.0)	(3.0, 0.003003003003003003)	(3.2, 0.003003003003003003)	
        };
        % \addlegendentry{First 3-quantile}
        \addplot[ybar, semithick, draw=red, fill = red!30!white, draw opacity = 0.7, fill opacity=0.4] coordinates{
            (-0.8, 0.021021021021021023)	(-0.6, 0.025525525525525526)	(-0.4, 0.057057057057057055)	(-0.2, 0.15915915915915915)	(0.0, 0.5405405405405406)	(0.2, 0.057057057057057055)	(0.4, 0.03453453453453453)	(0.6, 0.021021021021021023)	(0.8, 0.016516516516516516)	(1.0, 0.013513513513513514)	(1.2, 0.006006006006006006)	(1.4, 0.013513513513513514)	(1.6, 0.012012012012012012)	(1.8, 0.010510510510510511)	(2.0, 0.006006006006006006)	(2.2, 0.0)	(2.4, 0.0)	(2.6, 0.003003003003003003)	(2.8, 0.003003003003003003)	(3.0, 0.0)	(3.2, 0.0)	
        };
        % \addlegendentry{Second 3-quantile}
        \addplot[ybar, semithick, draw=green!20!black, fill = green!30!white, draw opacity = 0.7, fill opacity=0.4] coordinates{
            (-0.8, 0.04661654135338346)	(-0.6, 0.03308270676691729)	(-0.4, 0.03458646616541353)	(-0.2, 0.10977443609022557)	(0.0, 0.6135338345864662)	(0.2, 0.0481203007518797)	(0.4, 0.01804511278195489)	(0.6, 0.019548872180451128)	(0.8, 0.009022556390977444)	(1.0, 0.01804511278195489)	(1.2, 0.013533834586466165)	(1.4, 0.012030075187969926)	(1.6, 0.0030075187969924814)	(1.8, 0.007518796992481203)	(2.0, 0.0030075187969924814)	(2.2, 0.0015037593984962407)	(2.4, 0.0030075187969924814)	(2.6, 0.004511278195488722)	(2.8, 0.0)	(3.0, 0.0)	(3.2, 0.0015037593984962407)	
        };
        % \addlegendentry{Third 3-quantile}
    \end{axis}

    \begin{axis}[
        name = second plot,
        % at={(third plot.above south west)},%yshift=-0.2\textwidth,
        width=200pt,height=99pt,
        xshift=182pt,
        % width=\axisdefaultwidth,height=0.4\textwidth,
        % tick align=inside,
        tick pos=both,
        bar width=1.66666667ex,
        xlabel={Relative accuracy},
        xmin=-1, xmax=1,
        % xtick style={color=black},
        % xticklabel=\empty,
        y grid style={darkgray176},
        ymin=0, ymax=0.4,
        % ytick distance=0.5,
        ytick style={color=black}
        ]

        \path [draw=black, semithick, dashed]
        (axis cs:0,0)
        --(axis cs:0,1);
        \node[anchor=north west] at (axis cs: -0.975,.375) {$K=10$};
        \addplot[ybar, semithick, draw=blue, fill = blue!30!white, draw opacity = 0.7, fill opacity=0.4] coordinates{
            (-0.9, 0.003003003003003003)	(-0.8, 0.003003003003003003)	(-0.7, 0.015015015015015015)	(-0.6, 0.03753753753753754)	(-0.5, 0.052552552552552555)	(-0.4, 0.0990990990990991)	(-0.3, 0.13363363363363365)	(-0.2, 0.1921921921921922)	(-0.1, 0.18468468468468469)	(0.0, 0.08408408408408409)	(0.1, 0.04504504504504504)	(0.2, 0.03753753753753754)	(0.3, 0.03903903903903904)	(0.4, 0.015015015015015015)	(0.5, 0.024024024024024024)	(0.6, 0.012012012012012012)	(0.7, 0.013513513513513514)	(0.8, 0.0075075075075075074)	(0.9, 0.0015015015015015015)	(1.0, 0.0)	(1.1, 0.0)	(1.3, 0.0)	
        };
        \addplot[ybar, semithick, draw=red, fill = red!30!white, draw opacity = 0.7, fill opacity=0.4] coordinates{
            (-0.9, 0.0)	(-0.8, 0.0015015015015015015)	(-0.7, 0.013513513513513514)	(-0.6, 0.013513513513513514)	(-0.5, 0.03303303303303303)	(-0.4, 0.04504504504504504)	(-0.3, 0.1021021021021021)	(-0.2, 0.16366366366366367)	(-0.1, 0.25675675675675674)	(0.0, 0.15315315315315314)	(0.1, 0.046546546546546545)	(0.2, 0.03903903903903904)	(0.3, 0.02702702702702703)	(0.4, 0.028528528528528527)	(0.5, 0.01951951951951952)	(0.6, 0.025525525525525526)	(0.7, 0.018018018018018018)	(0.8, 0.006006006006006006)	(0.9, 0.006006006006006006)	(1.0, 0.0)	(1.1, 0.0)	(1.3, 0.0015015015015015015)	
        };
        \addplot[ybar, semithick, draw=green!20!black, fill = green!30!white, draw opacity = 0.7, fill opacity=0.4] coordinates{
            (-0.9, 0.0045045045045045045)	(-0.8, 0.006006006006006006)	(-0.7, 0.0045045045045045045)	(-0.6, 0.013513513513513514)	(-0.5, 0.03303303303303303)	(-0.4, 0.03903903903903904)	(-0.3, 0.06456456456456457)	(-0.2, 0.1906906906906907)	(-0.1, 0.23423423423423423)	(0.0, 0.24174174174174173)	(0.1, 0.03303303303303303)	(0.2, 0.03153153153153153)	(0.3, 0.025525525525525526)	(0.4, 0.018018018018018018)	(0.5, 0.021021021021021023)	(0.6, 0.012012012012012012)	(0.7, 0.0075075075075075074)	(0.8, 0.012012012012012012)	(0.9, 0.003003003003003003)	(1.0, 0.003003003003003003)	(1.1, 0.0015015015015015015)	(1.3, 0.0)	
        };
    \end{axis}

    \begin{axis}[
    name = first plot,
    % at={(second plot.above south west)},%yshift=-0.2\textwidth,
    % width=\axisdefaultwidth,height=0.4\textwidth,
    width=200pt,height=99pt,
    xshift=364pt,
    tick pos=both,
    xmin=-1, xmax=1,
    % xtick style={color=black},
    % xticklabel=\empty,
    y grid style={darkgray176},
    xlabel={Relative accuracy},
    ymin=0, ymax=0.25,
    bar width=0.833333333ex,
    ytick style={color=black},
    legend style={at={(0.9875, 0.95)}, anchor=north east},
    area legend
    ]
    
    \path [draw=black, semithick, dashed]
    (axis cs:0,0)
    --(axis cs:0,1);

    \node[anchor=north west] at (axis cs: -0.975,.234) {$K=20$};
        \addplot[ybar, semithick, draw=blue, fill = blue!30!white, draw opacity = 0.7, fill opacity=0.4] coordinates{
            (-0.85, 0.003003003003003003)	(-0.8, 0.0)	(-0.75, 0.02702702702702703)	(-0.7, 0.046546546546546545)	(-0.65, 0.06606606606606606)	(-0.6, 0.08558558558558559)	(-0.55, 0.11561561561561562)	(-0.5, 0.12162162162162163)	(-0.45, 0.12612612612612611)	(-0.4, 0.12012012012012012)	(-0.35, 0.1036036036036036)	(-0.3, 0.07657657657657657)	(-0.25, 0.046546546546546545)	(-0.2, 0.028528528528528527)	(-0.15, 0.021021021021021023)	(-0.1, 0.0045045045045045045)	(-0.05, 0.0045045045045045045)	(0.0, 0.0015015015015015015)	(0.05, 0.0)	(0.1, 0.0015015015015015015)	(0.15, 0.0)	
        };
        \addlegendentry{$Q_{1/3}(I)$}
        \addplot[ybar, semithick, draw=red, fill = red!30!white, draw opacity = 0.7, fill opacity=0.4] coordinates{
            (-0.85, 0.0015015015015015015)	(-0.8, 0.003003003003003003)	(-0.75, 0.003003003003003003)	(-0.7, 0.009009009009009009)	(-0.65, 0.024024024024024024)	(-0.6, 0.046546546546546545)	(-0.55, 0.06606606606606606)	(-0.5, 0.08408408408408409)	(-0.45, 0.13213213213213212)	(-0.4, 0.1966966966966967)	(-0.35, 0.14264264264264265)	(-0.3, 0.11861861861861862)	(-0.25, 0.08108108108108109)	(-0.2, 0.05555555555555555)	(-0.15, 0.02252252252252252)	(-0.1, 0.006006006006006006)	(-0.05, 0.0015015015015015015)	(0.0, 0.006006006006006006)	(0.05, 0.0)	(0.1, 0.0)	(0.15, 0.0)	
        };
        \addlegendentry{$Q_{2/3}(I)$}
        \addplot[ybar, semithick, draw=green!20!black, fill = green!30!white, draw opacity = 0.7, fill opacity=0.4] coordinates{
            (-0.85, 0.0)	(-0.8, 0.0)	(-0.75, 0.0015015015015015015)	(-0.7, 0.009009009009009009)	(-0.65, 0.010510510510510511)	(-0.6, 0.03153153153153153)	(-0.55, 0.036036036036036036)	(-0.5, 0.05855855855855856)	(-0.45, 0.09759759759759759)	(-0.4, 0.12312312312312312)	(-0.35, 0.14264264264264265)	(-0.3, 0.14864864864864866)	(-0.25, 0.11411411411411411)	(-0.2, 0.0990990990990991)	(-0.15, 0.06006006006006006)	(-0.1, 0.04054054054054054)	(-0.05, 0.015015015015015015)	(0.0, 0.009009009009009009)	(0.05, 0.0015015015015015015)	(0.1, 0.0)	(0.15, 0.0015015015015015015)	
        };
        \addlegendentry{$Q_{3/3}(I)$}
    \end{axis}
\end{tikzpicture}
        \caption{}\label{fig:Relative_accuracy_vs_information_quantity_uniform_histogram}
    \end{subfigure}
    \caption{
    Figures (a) and (b) show scatter plots of different quality measures for \Cref{alg:kpost} as a function of $I(\alpha,p)$, where we fixed $n=1000$ and $\ell_n/n=(\ln n)^3$ and sampled $2000$ \glspl{BMC} uniformly at random for each value of $K \in \{5,10,20\}$; see \Cref{sec:uniform_BMCs}.
    The curves summarize the data for each $K$ by plotting binned averages for different ranges of $I(\alpha, p)$.
    The inset in (b) shows the ratio of the \gls{AMI} with the \gls{AMI} when the true number of clusters is revealed.
    Figure (c) shows histograms of the relative accuracy, with each histogram collecting all data points with $I(\alpha, p)\in Q_{i/3}(I)$ where $Q_{i/3}(I)$ is the $i$'th 3-quantile.
    }\label{fig:Relative_accuracy_vs_information_quantity_uniform}
\end{figure}

\paragraph{Positive correlation of relative accuracy with $I(\alpha,p)$}
Observe in \Cref{fig:Relative_accuracy_vs_information_quantity_uniform_scatter} that both the relative accuracy and the information quantity decrease as $K$ increases.
Moreover, the relative accuracy shows a clear positive correlation when $K=20$, but this is less pronounced at smaller values ($K=5,10$).
Still, while the average relative accuracy does not depend on $I(\alpha, p)$ in this case, the histograms in \Cref{fig:Relative_accuracy_vs_information_quantity_uniform_histogram} show that the estimator $\Khat$ concentrates more tightly around $K$ for larger $I(\alpha, p)$.

\paragraph{Improved performance is not due to spurious clusters}
Recall that our method also produces a partial clustering that can be completed using \cref{alg:kmeans}.
To assess the quality of these clusters, we calculated the \gls{AMI} score described in \Cref{sec:Adjusted-Mutual-Information} and plotted the result as a function of the information quantity in \Cref{fig:Mutual_information_vs_information_quantity_uniform_scatter}.
We observe a distinct positive correlation between the \gls{AMI} score and the information quantity.
Together with the results in \Cref{fig:Relative_accuracy_vs_information_quantity_uniform_scatter}, this supports the interpretation that the improved accuracy of $\Khat$ at large $I(\alpha, p)$ is due to a genuine improvement of the clustering rather than the detection of spurious clusters.

\paragraph{Revealing $K$ does not significantly improve the clustering}
The inset in \Cref{fig:Mutual_information_vs_information_quantity_uniform_scatter} further shows the ratio of the \gls{AMI} score of our algorithm to the \gls{AMI} score when $K$ is revealed.
For the latter, we set $r=\Ktilde$ in \cref{alg:kpost} to isolate the performance of the density-based clustering, and iterate lines \Crefrange{ln:find_largest_neighborhoods_1}{ln:find_largest_neighborhoods_5} until precisely $K$ clusters are identified.
We observe that the ratio of \gls{AMI} scores is then consistently close to 1.

\subsubsection{Relative performance vs cluster sizes}

In this experiment, we consider instead how the relative accuracy correlates with variability in the cluster sizes.
To measure cluster size variability, we use the normalized entropy of a clustering, as defined in \Cref{sec:Entropy-of-a-partition}.

\begin{figure}[htbp]
    \centering
    \begin{subfigure}{0.495\linewidth}
        \begin{tikzpicture}[scale = .9]
    \begin{axis}[
        xmin = 0.6, xmax = 1.0,
        ymin = -1, ymax = 1,
        width = \linewidth,%*0.65
        height = 0.74 * \linewidth,%*0.65
        ytick distance=0.5,
        xlabel = {$\bar{\textnormal{H}}(\{\mathcal{V}_k\}_{k=1}^K)$}, ylabel = {Relative accuracy},
        legend cell align={left},
        ]

        \addplot[forget plot] graphics [xmin=0.6,xmax=1.0,ymin=-.94,ymax=1.025] {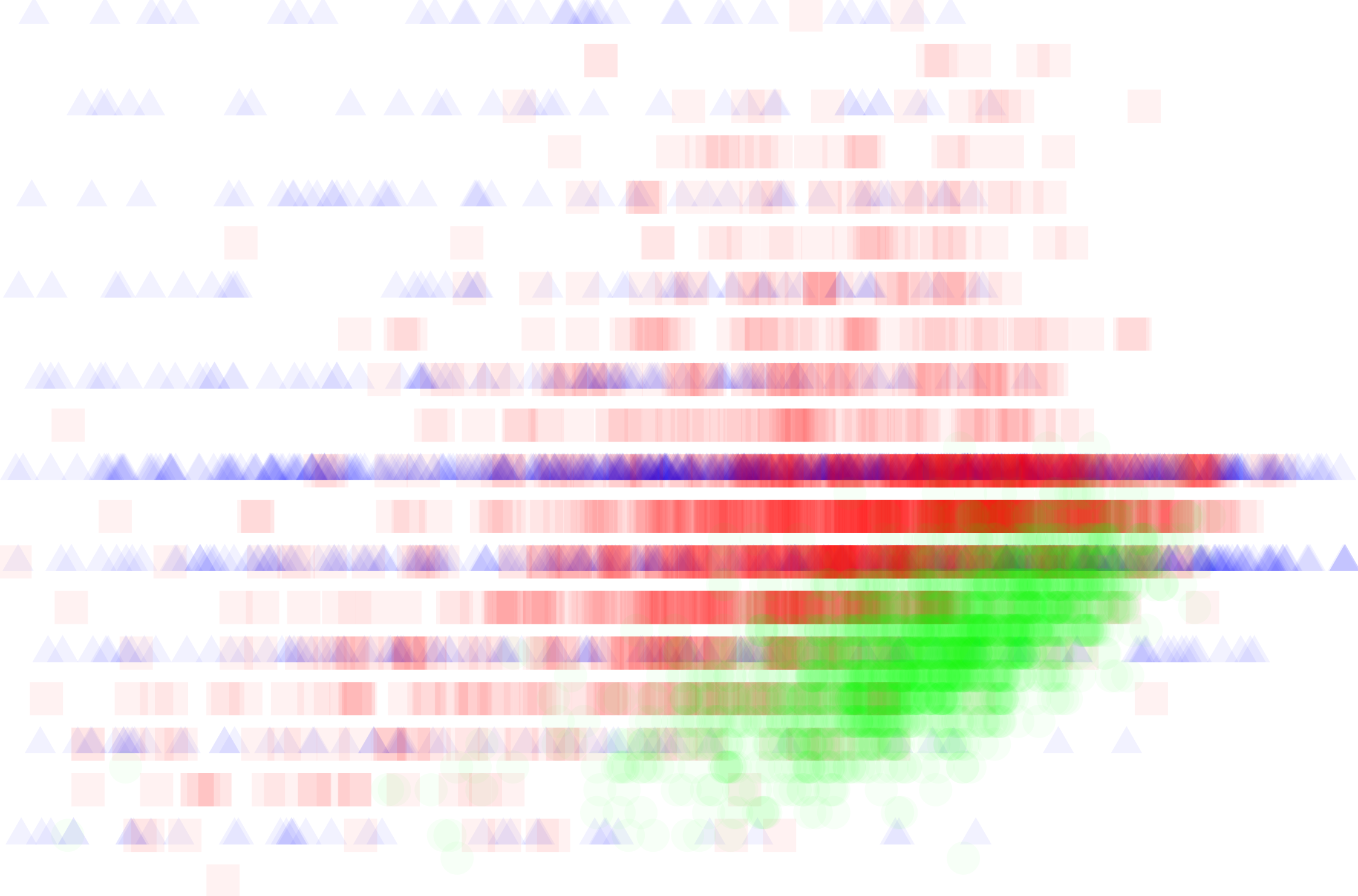};

        \addplot[smooth, mark=triangle*, mark options={scale=1}, x, error bars/.cd, y dir=both, y explicit] plot coordinates {
            (0.6400674749641049, 0.1417840375586854) += (0.6400674749641049, 0.10361744486298355) -= (0.6400674749641049, 0.10361744486298355)
            (0.7117880065262115, 0.3018867924528302) += (0.7117880065262115, 0.11660763074839915) -= (0.7117880065262115, 0.11660763074839915)
            (0.7590058028836215, 0.20281690140845055) += (0.7590058028836215, 0.09177843902511952) -= (0.7590058028836215, 0.09177843902511952)
            (0.7957669016831388, 0.11132075471698119) += (0.7957669016831388, 0.07046630154958361) -= (0.7957669016831388, 0.07046630154958361)
            (0.8270418354567272, 0.1679245283018868) += (0.8270418354567272, 0.0880303821037049) -= (0.8270418354567272, 0.0880303821037049)
            (0.8553579128152524, 0.1549295774647889) += (0.8553579128152524, 0.08277973994876119) -= (0.8553579128152524, 0.08277973994876119)
            (0.882352409262593, 0.014150943396226415) += (0.882352409262593, 0.03679291279573003) -= (0.882352409262593, 0.03679291279573003)
            (0.9133179960946263, -0.08262910798122058) += (0.9133179960946263, 0.015577362579161681) -= (0.9133179960946263, 0.015577362579161681)
        };
        \addlegendentry{$K=5$, blue}
        \addplot[smooth, mark=square*, mark options={scale=1}, error bars/.cd, y dir=both, y explicit] plot coordinates {
            (0.7502214274481256, -0.2098591549295774) += (0.7502214274481256, 0.04097605209871901) -= (0.7502214274481256, 0.04097605209871901)
            (0.7957111189963813, -0.14764150943396231) += (0.7957111189963813, 0.039103089022423834) -= (0.7957111189963813, 0.039103089022423834)
            (0.8189758634663178, -0.09999999999999994) += (0.8189758634663178, 0.041561532374923896) -= (0.8189758634663178, 0.041561532374923896)
            (0.8371395602183827, -0.09245283018867931) += (0.8371395602183827, 0.03864033329967393) -= (0.8371395602183827, 0.03864033329967393)
            (0.8524005955059732, -0.09715639810426542) += (0.8524005955059732, 0.03607029691478222) -= (0.8524005955059732, 0.03607029691478222)
            (0.8677811524889507, -0.06650943396226414) += (0.8677811524889507, 0.032175998393413674) -= (0.8677811524889507, 0.032175998393413674)
            (0.8840853633093189, 0.012264150943396201) += (0.8840853633093189, 0.03668888170804409) -= (0.8840853633093189, 0.03668888170804409)
            (0.9028361202229702, -0.017840375586854473) += (0.9028361202229702, 0.030591513547690113) -= (0.9028361202229702, 0.030591513547690113)
        };
        \addlegendentry{$K=10$, red}
        \addplot[smooth, mark=*, mark options={scale=1}, error bars/.cd, y dir=both, y explicit] plot coordinates {
            (0.8234632763863274, -0.5112676056338028) += (0.8234632763863274, 0.017148503155062635) -= (0.8234632763863274, 0.017148503155062635)
            (0.8456182269096604, -0.4632075471698113) += (0.8456182269096604, 0.015561266623739007) -= (0.8456182269096604, 0.015561266623739007)
            (0.8584289436739254, -0.44504716981132053) += (0.8584289436739254, 0.01575039051141651) -= (0.8584289436739254, 0.01575039051141651)
            (0.8690339069374075, -0.40213270142180146) += (0.8690339069374075, 0.014274841515711692) -= (0.8690339069374075, 0.014274841515711692)
            (0.8786371799369683, -0.39764150943396226) += (0.8786371799369683, 0.01622229942547881) -= (0.8786371799369683, 0.01622229942547881)
            (0.8878124414005903, -0.346244131455399) += (0.8878124414005903, 0.014555251470689801) -= (0.8878124414005903, 0.014555251470689801)
            (0.8975751534302852, -0.3195754716981131) += (0.8975751534302852, 0.014709189475493262) -= (0.8975751534302852, 0.014709189475493262)
            (0.9096696087424017, -0.2767605633802815) += (0.9096696087424017, 0.01346893677994666) -= (0.9096696087424017, 0.01346893677994666)
            };
            \addlegendentry{$K=20$, green}
        \addplot[dashed, domain = 0.0:1.2, samples = 2] {0.0};
    \end{axis}
\end{tikzpicture}
        \caption{}\label{fig:Relative_accuracy_vs_cluster_entropy_uniform_scatter}
    \end{subfigure}
    \begin{subfigure}{0.495\linewidth}
        \begin{tikzpicture}[scale = .9]
    \begin{axis}[
        xmin = 0.6, xmax = 1.0,
        ymin = 0.4, ymax = 1.025,
        width = \linewidth,%*0.65
        height = 0.74 * \linewidth,%*0.65
        xlabel = {$\bar{\textnormal{H}}(\{\mathcal{V}\}_{k=1}^K)$}, ylabel = {AMI},
        ]

        \addplot graphics [xmin=0.6,xmax=1.0,ymin=0.4,ymax=1.01] {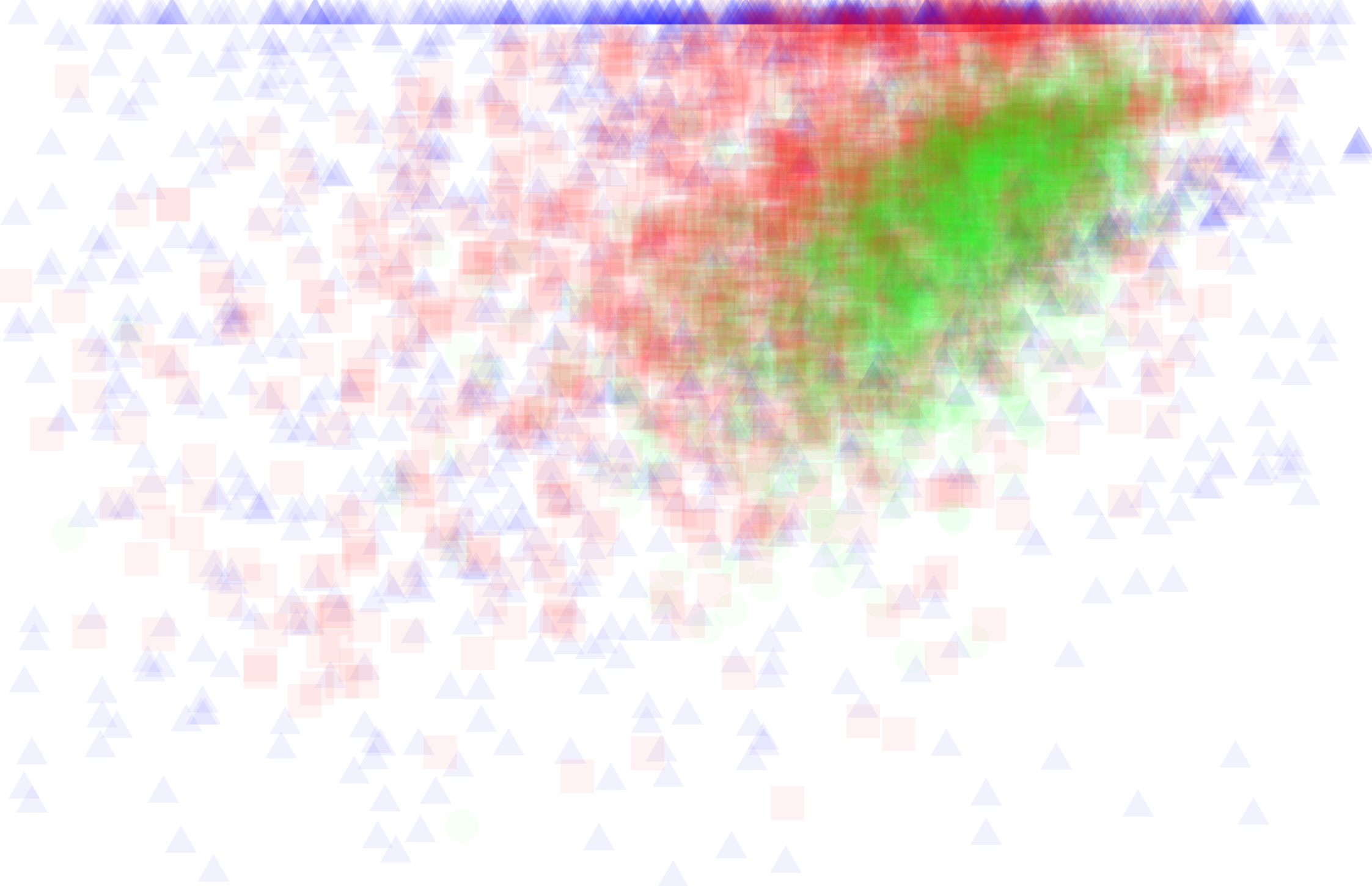};

        \addplot[smooth, mark=triangle*, mark options={scale=0.5}, x, error bars/.cd, y dir=both, y explicit] plot coordinates {
            (0.6400674749641049, 0.7021633618627334) += (0.6400674749641049, 0.0381945833953324) -= (0.6400674749641049, 0.0381945833953324)
            (0.7117880065262115, 0.7893877355867199) += (0.7117880065262115, 0.02808244488496058) -= (0.7117880065262115, 0.02808244488496058)
            (0.7590058028836215, 0.8068454510516784) += (0.7590058028836215, 0.03041186562421922) -= (0.7590058028836215, 0.03041186562421922)
            (0.7957669016831388, 0.8742078831666344) += (0.7957669016831388, 0.02320204491159619) -= (0.7957669016831388, 0.02320204491159619)
            (0.8270418354567272, 0.8724400496624782) += (0.8270418354567272, 0.021377833212493296) -= (0.8270418354567272, 0.021377833212493296)
            (0.8553579128152524, 0.8906761231739411) += (0.8553579128152524, 0.020842546658809993) -= (0.8553579128152524, 0.020842546658809993)
            (0.882352409262593, 0.9079367515341465) += (0.882352409262593, 0.019749904805021302) -= (0.882352409262593, 0.019749904805021302)
            (0.9133179960946263, 0.9115157908099351) += (0.9133179960946263, 0.013336315399328787) -= (0.9133179960946263, 0.013336315399328787)
        };
        \addplot[smooth, mark=square*, mark options={scale=0.5}, error bars/.cd, y dir=both, y explicit] plot coordinates {
            (0.7502214274481256, 0.8094923016276235) += (0.7502214274481256, 0.014452252758005324) -= (0.7502214274481256, 0.014452252758005324)
            (0.7957111189963813, 0.8290496629595043) += (0.7957111189963813, 0.012241655686651079) -= (0.7957111189963813, 0.012241655686651079)
            (0.8189758634663178, 0.8432757374627678) += (0.8189758634663178, 0.012939486172192859) -= (0.8189758634663178, 0.012939486172192859)
            (0.8371395602183827, 0.868322941457026) += (0.8371395602183827, 0.011131480700057556) -= (0.8371395602183827, 0.011131480700057556)
            (0.8524005955059732, 0.8684739325184063) += (0.8524005955059732, 0.011645186431599624) -= (0.8524005955059732, 0.011645186431599624)
            (0.8677811524889507, 0.8825745962823422) += (0.8677811524889507, 0.011704032363432561) -= (0.8677811524889507, 0.011704032363432561)
            (0.8840853633093189, 0.907044143064203) += (0.8840853633093189, 0.009967785049157515) -= (0.8840853633093189, 0.009967785049157515)
            (0.9028361202229702, 0.9124120663021577) += (0.9028361202229702, 0.00829838915922154) -= (0.9028361202229702, 0.00829838915922154)
        };
        \addplot[smooth, mark=*, mark options={scale=0.5}, error bars/.cd, y dir=both, y explicit] plot coordinates {
            (0.8234632763863274, 0.7900747428189087) += (0.8234632763863274, 0.008912482590544415) -= (0.8234632763863274, 0.008912482590544415)
            (0.8456182269096604, 0.8147307556761145) += (0.8456182269096604, 0.008511276616907976) -= (0.8456182269096604, 0.008511276616907976)
            (0.8584289436739254, 0.8221858033647977) += (0.8584289436739254, 0.008211242246381223) -= (0.8584289436739254, 0.008211242246381223)
            (0.8690339069374075, 0.8428488829988163) += (0.8690339069374075, 0.008200421504442763) -= (0.8690339069374075, 0.008200421504442763)
            (0.8786371799369683, 0.8518456434234873) += (0.8786371799369683, 0.008462484702019227) -= (0.8786371799369683, 0.008462484702019227)
            (0.8878124414005903, 0.8700919201169913) += (0.8878124414005903, 0.006691926080187446) -= (0.8878124414005903, 0.006691926080187446)
            (0.8975751534302852, 0.875914148485148) += (0.8975751534302852, 0.007500175919741692) -= (0.8975751534302852, 0.007500175919741692)
            (0.9096696087424017, 0.8946250382484069) += (0.9096696087424017, 0.006046013010765493) -= (0.9096696087424017, 0.006046013010765493)
            };
        \addplot[dashed, domain = 0.0:1.2, samples = 2] {1.0};
        \coordinate (insetPosition) at (axis cs:0.775, 0.405);
    \end{axis}
    % inset
    \fill [white] (insetPosition) rectangle ++(3cm,1.9cm); % white box below insert
    \begin{axis}[
            at={(insetPosition)},
            reverse legend,
            xmin = 0.6, xmax = 1.0,
            ymin = 0.85, ymax = 1.05,
            ytick distance=0.1,
            xticklabel=\empty,
            label style={font=\small},
            ylabel = {$\textnormal{AMI}_{\hat{K}}/\textnormal{AMI}_{K}$},
            y label style={at={(axis description cs:-0.17,0.5)},anchor=south},
            % xlabel = {$H(\alpha)$}, 
            % ylabel = {Adjusted mutual information (AMI)},
            % label style = {font = \normalsize},
            % tick label style = {font = \large}
            width = 4.5cm, height = 3.5cm
        ]

        \addplot[forget plot] graphics [xmin=0.6,xmax=1.0,ymin=.85,ymax=1.05] {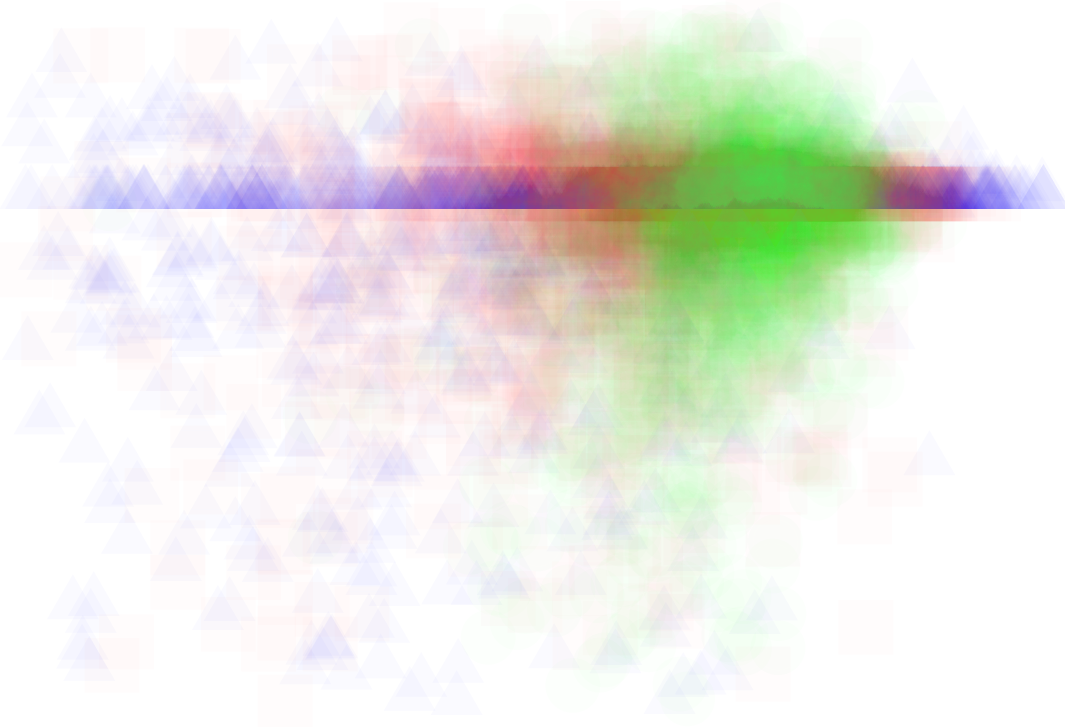};
        
        \addplot[smooth, mark=triangle*, mark options={scale=1}, error bars/.cd, y dir=both, y explicit] plot coordinates {
            (0.6400674749641049, 0.9068060031947139) += (0.6400674749641049, 0.04261230386797106) -= (0.6400674749641049, 0.04261230386797106)
            (0.7117880065262115, 0.9802557407175888) += (0.7117880065262115, 0.028850718085707328) -= (0.7117880065262115, 0.028850718085707328)
            (0.7590058028836215, 0.9607184666870329) += (0.7590058028836215, 0.030621911025388526) -= (0.7590058028836215, 0.030621911025388526)
            (0.7957669016831388, 0.9789001919255927) += (0.7957669016831388, 0.015975224177333113) -= (0.7957669016831388, 0.015975224177333113)
            (0.8270418354567272, 0.9808896914380428) += (0.8270418354567272, 0.014124161426044909) -= (0.8270418354567272, 0.014124161426044909)
            (0.8553579128152524, 0.9747223692424868) += (0.8553579128152524, 0.015069395032567674) -= (0.8553579128152524, 0.015069395032567674)
            (0.882352409262593, 0.9926855354169138) += (0.882352409262593, 0.010443205549906546) -= (0.882352409262593, 0.010443205549906546)
            (0.9133179960946263, 1.0001429914445588) += (0.9133179960946263, 0.0011467585362025063) -= (0.9133179960946263, 0.0011467585362025063)
            };
        \addplot[smooth, mark=square*, mark options={scale=1}, error bars/.cd, y dir=both, y explicit] plot coordinates {
            (0.7502214274481256, 0.9830217330414311) += (0.7502214274481256, 0.006762299491272788) -= (0.7502214274481256, 0.006762299491272788)
            (0.7957111189963813, 0.9898886818580835) += (0.7957111189963813, 0.004457793750815515) -= (0.7957111189963813, 0.004457793750815515)
            (0.8189758634663178, 0.9873354226883911) += (0.8189758634663178, 0.00624211730316855) -= (0.8189758634663178, 0.00624211730316855)
            (0.8371395602183827, 0.991855511065827) += (0.8371395602183827, 0.003863796727151663) -= (0.8371395602183827, 0.003863796727151663)
            (0.8524005955059732, 0.9894026330141887) += (0.8524005955059732, 0.004670568383099614) -= (0.8524005955059732, 0.004670568383099614)
            (0.8677811524889507, 0.9919701107918123) += (0.8677811524889507, 0.004954475596177547) -= (0.8677811524889507, 0.004954475596177547)
            (0.8840853633093189, 0.9959952938295905) += (0.8840853633093189, 0.0032051181049933897) -= (0.8840853633093189, 0.0032051181049933897)
            (0.9028361202229702, 0.9946362962807266) += (0.9028361202229702, 0.0026915935755122515) -= (0.9028361202229702, 0.0026915935755122515)
            };
        \addplot[smooth, mark=*, mark options={scale=1}, error bars/.cd, y dir=both, y explicit] plot coordinates {
            (0.8234632763863274, 0.9710306119171883) += (0.8234632763863274, 0.006606542483739048) -= (0.8234632763863274, 0.006606542483739048)
            (0.8456182269096604, 0.9798952960186604) += (0.8456182269096604, 0.005245731690509437) -= (0.8456182269096604, 0.005245731690509437)
            (0.8584289436739254, 0.9836543235938221) += (0.8584289436739254, 0.005236370202884394) -= (0.8584289436739254, 0.005236370202884394)
            (0.8690339069374075, 0.9876951685793142) += (0.8690339069374075, 0.004586161738918391) -= (0.8690339069374075, 0.004586161738918391)
            (0.8786371799369683, 0.9836910031896329) += (0.8786371799369683, 0.005686027919522127) -= (0.8786371799369683, 0.005686027919522127)
            (0.8878124414005903, 0.9929593002134676) += (0.8878124414005903, 0.0032677515235261145) -= (0.8878124414005903, 0.0032677515235261145)
            (0.8975751534302852, 0.9928020067553573) += (0.8975751534302852, 0.003254384752958206) -= (0.8975751534302852, 0.003254384752958206)
            (0.9096696087424017, 0.9945828583853746) += (0.9096696087424017, 0.0029572599138279047) -= (0.9096696087424017, 0.0029572599138279047)
            };
        \addplot[dashed, domain = 0.2:1.0, samples = 2] {1.0};
    \end{axis}
\end{tikzpicture}
        \caption{}\label{fig:Mutual_information_vs_cluster_entropy_uniform_scatter}
    \end{subfigure}
    \begin{subfigure}{\linewidth}
        \begin{tikzpicture}[scale=0.8,]
    \begin{axis}[
        name = third plot,
        % at={(second plot.below south west)},%yshift=-0.2\textwidth,       
        % yshift=0.4\textwidth, 
        width=200pt,height=99pt,
        % width=240pt,height=0.4\textwidth,
        % tick align=inside,
        tick pos=both,
        xlabel={Relative accuracy},
        xmin=-1, xmax=1,
        bar width=3.33333334ex,
        xtick style={color=black},
        y grid style={darkgray176}, 
        % ytick distance=0.5,
        ylabel={Rel. occurence},
        ymin=0, ymax=0.65,
        ytick style={color=black},
        legend style={at={(0.9875, 0.95)}, anchor=north east},
        area legend
        ]
        
        \path [draw=black, semithick, dashed]
        (axis cs:0,0)
        --(axis cs:0,1);

        \node[anchor=north west] at (axis cs: -0.975,.61) {$K=5$};
        \addplot[ybar, semithick, draw=blue, fill = blue!30!white, draw opacity = 0.7, fill opacity=0.4] coordinates{
            (-0.8, 0.07207207207207207)	(-0.6, 0.06456456456456457)	(-0.4, 0.0915915915915916)	(-0.2, 0.14714714714714713)	(0.0, 0.26876876876876876)	(0.2, 0.07507507507507508)	(0.4, 0.03753753753753754)	(0.6, 0.04504504504504504)	(0.8, 0.02702702702702703)	(1.0, 0.03153153153153153)	(1.2, 0.02702702702702703)	(1.4, 0.02702702702702703)	(1.6, 0.010510510510510511)	(1.8, 0.016516516516516516)	(2.0, 0.018018018018018018)	(2.2, 0.010510510510510511)	(2.4, 0.006006006006006006)	(2.6, 0.0045045045045045045)	(2.8, 0.009009009009009009)	(3.0, 0.003003003003003003)	(3.2, 0.003003003003003003)	(3.4, 0.0015015015015015015)	(3.6, 0.0015015015015015015)	(3.8, 0.0015015015015015015)	
        };
        % \addlegendentry{First 3-quantile}
        \addplot[ybar, semithick, draw=red, fill = red!30!white, draw opacity = 0.7, fill opacity=0.4] coordinates{
            (-0.8, 0.009009009009009009)	(-0.6, 0.015015015015015015)	(-0.4, 0.04954954954954955)	(-0.2, 0.14564564564564564)	(0.0, 0.5435435435435435)	(0.2, 0.07057057057057058)	(0.4, 0.036036036036036036)	(0.6, 0.024024024024024024)	(0.8, 0.016516516516516516)	(1.0, 0.02252252252252252)	(1.2, 0.012012012012012012)	(1.4, 0.010510510510510511)	(1.6, 0.003003003003003003)	(1.8, 0.016516516516516516)	(2.0, 0.006006006006006006)	(2.2, 0.0015015015015015015)	(2.4, 0.0045045045045045045)	(2.6, 0.0015015015015015015)	(2.8, 0.003003003003003003)	(3.0, 0.0015015015015015015)	(3.2, 0.0015015015015015015)	(3.4, 0.003003003003003003)	(3.6, 0.0)	(3.8, 0.003003003003003003)	
        };
        % \addlegendentry{Second 3-quantile}
        \addplot[ybar, semithick, draw=green!20!black, fill = green!30!white, draw opacity = 0.7, fill opacity=0.4] coordinates{
            (-0.8, 0.0045045045045045045)	(-0.6, 0.0075075075075075074)	(-0.4, 0.04054054054054054)	(-0.2, 0.2912912912912913)	(0.0, 0.6201201201201201)	(0.2, 0.009009009009009009)	(0.4, 0.006006006006006006)	(0.6, 0.006006006006006006)	(0.8, 0.0045045045045045045)	(1.0, 0.003003003003003003)	(1.2, 0.0015015015015015015)	(1.4, 0.0015015015015015015)	(1.6, 0.0015015015015015015)	(1.8, 0.0015015015015015015)	(2.0, 0.0)	(2.2, 0.0)	(2.4, 0.0)	(2.6, 0.0)	(2.8, 0.0)	(3.0, 0.0)	(3.2, 0.0015015015015015015)	(3.4, 0.0)	(3.6, 0.0)	(3.8, 0.0)	
        };
        % \addlegendentry{Third 3-quantile}
    \end{axis}

    \begin{axis}[
        name = second plot,
        % at={(third plot.above south west)},%yshift=-0.2\textwidth,
        width=200pt,height=99pt,
        xshift=182pt,
        % width=\axisdefaultwidth,height=0.4\textwidth,
        % tick align=inside,
        tick pos=both,
        bar width=1.66666667ex,
        xlabel={Relative accuracy},
        xmin=-1, xmax=1,
        % xtick style={color=black},
        % xticklabel=\empty,
        y grid style={darkgray176},
        ymin=0, ymax=0.4,
        % ytick distance=0.5,
        ytick style={color=black}
        ]

        \path [draw=black, semithick, dashed]
        (axis cs:0,0)
        --(axis cs:0,1);
        \node[anchor=north west] at (axis cs: -0.975,.375) {$K=10$};
        \addplot[ybar, semithick, draw=blue, fill = blue!30!white, draw opacity = 0.7, fill opacity=0.4] coordinates{
            (-0.9, 0.0015015015015015015)	(-0.8, 0.015015015015015015)	(-0.7, 0.03303303303303303)	(-0.6, 0.05405405405405406)	(-0.5, 0.10810810810810811)	(-0.4, 0.12312312312312312)	(-0.3, 0.14114114114114115)	(-0.2, 0.14264264264264265)	(-0.1, 0.12762762762762764)	(0.0, 0.07357357357357357)	(0.1, 0.042042042042042045)	(0.2, 0.04504504504504504)	(0.3, 0.03303303303303303)	(0.4, 0.016516516516516516)	(0.5, 0.010510510510510511)	(0.6, 0.012012012012012012)	(0.7, 0.013513513513513514)	(0.8, 0.0045045045045045045)	(0.9, 0.003003003003003003)	(1.0, 0.0)	(1.1, 0.0)	(1.2, 0.0)	(1.4, 0.0)	
        };
        \addplot[ybar, semithick, draw=red, fill = red!30!white, draw opacity = 0.7, fill opacity=0.4] coordinates{
            (-0.9, 0.0)	(-0.8, 0.0015037593984962407)	(-0.7, 0.0)	(-0.6, 0.012030075187969926)	(-0.5, 0.03759398496240601)	(-0.4, 0.06015037593984962)	(-0.3, 0.14887218045112782)	(-0.2, 0.20902255639097744)	(-0.1, 0.19699248120300752)	(0.0, 0.1368421052631579)	(0.1, 0.04661654135338346)	(0.2, 0.042105263157894736)	(0.3, 0.02556390977443609)	(0.4, 0.02857142857142857)	(0.5, 0.01804511278195489)	(0.6, 0.015037593984962405)	(0.7, 0.010526315789473684)	(0.8, 0.004511278195488722)	(0.9, 0.0)	(1.0, 0.0030075187969924814)	(1.1, 0.0030075187969924814)	(1.2, 0.0)	(1.4, 0.0)	
        };
        \addplot[ybar, semithick, draw=green!20!black, fill = green!30!white, draw opacity = 0.7, fill opacity=0.4] coordinates{
            (-0.9, 0.0)	(-0.8, 0.0)	(-0.7, 0.0)	(-0.6, 0.0)	(-0.5, 0.0030075187969924814)	(-0.4, 0.007518796992481203)	(-0.3, 0.0406015037593985)	(-0.2, 0.1699248120300752)	(-0.1, 0.362406015037594)	(0.0, 0.26766917293233083)	(0.1, 0.02706766917293233)	(0.2, 0.031578947368421054)	(0.3, 0.02406015037593985)	(0.4, 0.012030075187969926)	(0.5, 0.010526315789473684)	(0.6, 0.012030075187969926)	(0.7, 0.007518796992481203)	(0.8, 0.009022556390977444)	(0.9, 0.009022556390977444)	(1.0, 0.0)	(1.1, 0.0030075187969924814)	(1.2, 0.0015037593984962407)	(1.4, 0.0015037593984962407)
        };
    \end{axis}

    \begin{axis}[
    name = first plot,
    % at={(second plot.above south west)},%yshift=-0.2\textwidth,
    % width=\axisdefaultwidth,height=0.4\textwidth,
    width=200pt,height=99pt,
    xshift=364pt,
    tick pos=both,
    xmin=-1, xmax=1,
    % xtick style={color=black},
    % xticklabel=\empty,
    y grid style={darkgray176},
    xlabel={Relative accuracy},
    ymin=0, ymax=0.25,
    bar width=0.833333333ex,
    ytick style={color=black},
    legend style={at={(0.9875, 0.95)}, anchor=north east},
    area legend
    ]
    
    \path [draw=black, semithick, dashed]
    (axis cs:0,0)
    --(axis cs:0,1);

    \node[anchor=north west] at (axis cs: -0.975,.234) {$K=20$};
        \addplot[ybar, semithick, draw=blue, fill = blue!30!white, draw opacity = 0.7, fill opacity=0.4] coordinates{
            (-0.85, 0.0015015015015015015)	(-0.8, 0.010510510510510511)	(-0.75, 0.021021021021021023)	(-0.7, 0.036036036036036036)	(-0.65, 0.08108108108108109)	(-0.6, 0.11261261261261261)	(-0.55, 0.14264264264264265)	(-0.5, 0.16666666666666666)	(-0.45, 0.13363363363363365)	(-0.4, 0.11561561561561562)	(-0.35, 0.08558558558558559)	(-0.3, 0.03903903903903904)	(-0.25, 0.03753753753753754)	(-0.2, 0.0075075075075075074)	(-0.15, 0.006006006006006006)	(-0.1, 0.0015015015015015015)	(-0.05, 0.0015015015015015015)	(0.0, 0.0)	(0.05, 0.0)	
        };
        \addlegendentry{$Q_{1/3}(\bar{\textnormal{H}})$}
        \addplot[ybar, semithick, draw=red, fill = red!30!white, draw opacity = 0.7, fill opacity=0.4] coordinates{
            (-0.85, 0.0015015015015015015)	(-0.8, 0.0)	(-0.75, 0.003003003003003003)	(-0.7, 0.0015015015015015015)	(-0.65, 0.013513513513513514)	(-0.6, 0.028528528528528527)	(-0.55, 0.057057057057057055)	(-0.5, 0.11561561561561562)	(-0.45, 0.16666666666666666)	(-0.4, 0.1966966966966967)	(-0.35, 0.17117117117117117)	(-0.3, 0.0960960960960961)	(-0.25, 0.07057057057057058)	(-0.2, 0.03903903903903904)	(-0.15, 0.018018018018018018)	(-0.1, 0.015015015015015015)	(-0.05, 0.003003003003003003)	(0.0, 0.0015015015015015015)	(0.05, 0.0015015015015015015)	
        };
        \addlegendentry{$Q_{2/3}(\bar{\textnormal{H}})$}
        \addplot[ybar, semithick, draw=green!20!black, fill = green!30!white, draw opacity = 0.7, fill opacity=0.4] coordinates{
            (-0.85, 0.0)	(-0.8, 0.0)	(-0.75, 0.0)	(-0.7, 0.0)	(-0.65, 0.0)	(-0.6, 0.0015015015015015015)	(-0.55, 0.010510510510510511)	(-0.5, 0.02702702702702703)	(-0.45, 0.057057057057057055)	(-0.4, 0.08408408408408409)	(-0.35, 0.15165165165165165)	(-0.3, 0.16516516516516516)	(-0.25, 0.1921921921921922)	(-0.2, 0.14864864864864866)	(-0.15, 0.09009009009009009)	(-0.1, 0.04354354354354354)	(-0.05, 0.021021021021021023)	(0.0, 0.0045045045045045045)	(0.05, 0.003003003003003003)	
        };
        \addlegendentry{$Q_{3/3}(\bar{\textnormal{H}})$}
    \end{axis}
\end{tikzpicture}
        \caption{}\label{fig:Relative_accuracy_vs_cluster_entropy_uniform_histogram}
    \end{subfigure}
    \caption{
    Same data as \Cref{fig:Relative_accuracy_vs_information_quantity_uniform}, but the results are now plotted as a function of the normalized cluster entropy $\bar{\text{H}}(\{\mcV_k\}_{k=1}^K)$, rather than the information quantity $I(\alpha, p)$.
    }\label{fig:Relative_accuracy_vs_cluster_entropy_uniform}
\end{figure}

\paragraph{Positive correlation of relative accuracy with $\bar{\text{H}}$}
\Cref{fig:Relative_accuracy_vs_cluster_entropy_uniform} shows the results in the same format as \Cref{fig:Relative_accuracy_vs_information_quantity_uniform}.
We observe similar trends, indicating that the clustering entropy is also predictive of our algorithm's performance.
Specifically, our algorithm performs better when cluster sizes are more uniform.
We have also checked (not shown here) that the information quantity and clustering entropy are not significantly correlated, and that we are therefore detecting two distinct factors that affect \Cref{alg:kpost}'s performance.

\subsection{Estimating the number of clusters in noisy environments}
\label{sec:noisy_environments}

Up to here, we have only considered \Cref{alg:kpre,alg:kpost}'s performance when the trajectory was generated by a pure \gls{BMC}.
In this next experiment, we investigate the robustness of the method in a noisy, synthetic environment.

The idea is to generate trajectories of \emph{perturbed \glspl{BMC}}.
Specifically, we will generate trajectories of \glspl{MC} with transition matrices of the form
\begin{equation}
    \label{eqn:noisy_transition_matrix}
    P
    =
    (1-\varepsilon)
    P_{\textnormal{\gls{BMC} with } K \textnormal{ clusters}}
    +
    \varepsilon
    P_{\textnormal{Uniform \gls{MC} of size } n}
    .
\end{equation}
In this experiment, the cluster transition matrix $p$ and cluster fractions $\alpha$ of the \gls{BMC} component will be sampled uniformly at random from the simplex; see \Cref{sec:uniform_BMCs}.

\Cref{fig:Relative-accuracy-vs-perturbation-probability} displays the relative accuracy attained for random uniform \glspl{BMC} with $K$ clusters, that are perturbed by uniformly distributed random \gls{MC} with probability $\varepsilon \in [0,1]$.

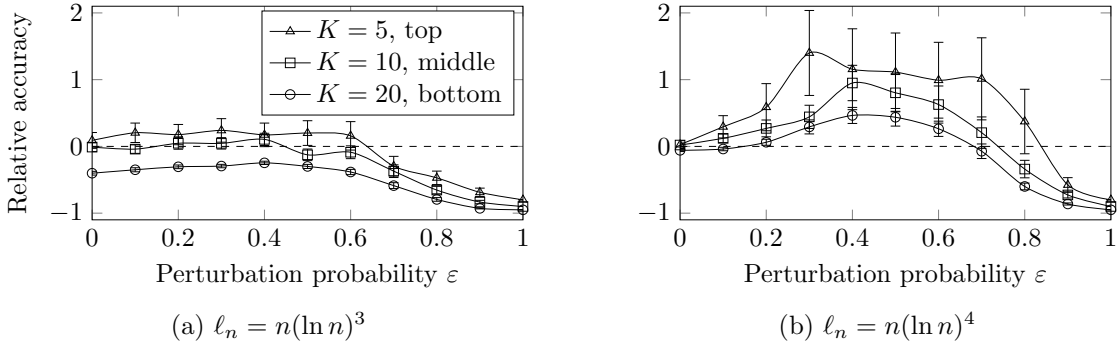
\begin{figure}[hbtp]
    \centering
    \begin{subfigure}{0.495\linewidth}
        \centering
        \begin{tikzpicture}[scale = 0.9]
            \begin{axis}[
                reverse legend,
                width = \linewidth,
                height = 0.6 * \linewidth,
                xmin = 0, xmax = 1.0,
                ymin = -1.1, ymax = 2.1,
                xlabel = {Perturbation probability $\varepsilon$},
                ylabel = {Relative accuracy},
                legend cell align={left},
            ]
            \addplot[dashed, domain = 0.0:1.0, samples = 2] {0.0};

            \addplot[smooth, mark=o, error bars/.cd, y dir=both, y explicit] plot coordinates {
                (0, -0.401563) += (0, 0.028048) -= (0, 0.028048)
                (0.1, -0.35) += (0.1, 0.0302828) -= (0.1, 0.0302828)
                (0.2, -0.305729) += (0.2, 0.0275008) -= (0.2, 0.0275008)
                (0.3, -0.294271) += (0.3, 0.0256581) -= (0.3, 0.0256581)
                (0.4, -0.246354) += (0.4, 0.0247221) -= (0.4, 0.0247221)
                (0.5, -0.3) += (0.5, 0.0307269) -= (0.5, 0.0307269)
                (0.6, -0.383333) += (0.6, 0.0392103) -= (0.6, 0.0392103)
                (0.7, -0.5875) += (0.7, 0.0441444) -= (0.7, 0.0441444)
                (0.8, -0.794271) += (0.8, 0.0273488) -= (0.8, 0.0273488)
                (0.9, -0.928125) += (0.9, 0.00814271) -= (0.9, 0.00814271)
                (1, -0.95) += (1, 3.33136e-16) -= (1, 3.33136e-16)
            };
            \addlegendentry{$K = 20$, bottom}

            \addplot[smooth, mark=square, error bars/.cd, y dir=both, y explicit] plot coordinates {
                (0, -0.0104167) += (0, 0.0676824) -= (0, 0.0676824)
                (0.1, -0.040625) += (0.1, 0.0636907) -= (0.1, 0.0636907)
                (0.2, 0.046875) += (0.2, 0.0893193) -= (0.2, 0.0893193)
                (0.3, 0.0458333) += (0.3, 0.0856) -= (0.3, 0.0856)
                (0.4, 0.105208) += (0.4, 0.0916419) -= (0.4, 0.0916419)
                (0.5, -0.126042) += (0.5, 0.0868246) -= (0.5, 0.0868246)
                (0.6, -0.0875) += (0.6, 0.0959909) -= (0.6, 0.0959909)
                (0.7, -0.370833) += (0.7, 0.0952352) -= (0.7, 0.0952352)
                (0.8, -0.654167) += (0.8, 0.0559071) -= (0.8, 0.0559071)
                (0.9, -0.833333) += (0.9, 0.0254459) -= (0.9, 0.0254459)
                (1, -0.9) += (1, 1.33255e-16) -= (1, 1.33255e-16)
            };
            \addlegendentry{$K = 10$, middle}

            \addplot[smooth, mark=triangle, error bars/.cd, y dir=both, y explicit] plot coordinates {
                (0, 0.0895833) += (0, 0.118469) -= (0, 0.118469)
                (0.1, 0.204167) += (0.1, 0.145516) -= (0.1, 0.145516)
                (0.2, 0.177083) += (0.2, 0.151785) -= (0.2, 0.151785)
                (0.3, 0.239583) += (0.3, 0.176164) -= (0.3, 0.176164)
                (0.4, 0.172917) += (0.4, 0.176636) -= (0.4, 0.176636)
                (0.5, 0.2) += (0.5, 0.184971) -= (0.5, 0.184971)
                (0.6, 0.158333) += (0.6, 0.213032) -= (0.6, 0.213032)
                (0.7, -0.295833) += (0.7, 0.146771) -= (0.7, 0.146771)
                (0.8, -0.470833) += (0.8, 0.100847) -= (0.8, 0.100847)
                (0.9, -0.6875) += (0.9, 0.0624129) -= (0.9, 0.0624129)
                (1, -0.8) += (1, 3.10927e-16) -= (1, 3.10927e-16)
            };
            \addlegendentry{$K = 5$, top}
            
            \end{axis}
            \end{tikzpicture}
            \caption{$\ell_n = n (\ln{n})^3$}\label{fig:Relative-accuracy-vs-perturbation-probability-sparse}
    \end{subfigure}
    \begin{subfigure}{0.495\linewidth}
        \centering
        \begin{tikzpicture}[scale = 0.9]
            \begin{axis}[
                reverse legend,
                xmin = 0, xmax = 1.0,
                ymin = -1.1, ymax = 2.1,
                width = \linewidth,
                height = 0.6 * \linewidth,
                xlabel = {Perturbation probability $\varepsilon$},
            ]
            \addplot[dashed, domain = 0.0:1.0, samples = 2] {0.0};

            \addplot[smooth, mark=o, error bars/.cd, y dir=both, y explicit] plot coordinates {
                (0, -0.0588542) += (0, 0.0180889) -= (0, 0.0180889)
                (0.1, -0.0390625) += (0.1, 0.0258738) -= (0.1, 0.0258738)
                (0.2, 0.0619792) += (0.2, 0.0451883) -= (0.2, 0.0451883)
                (0.3, 0.292188) += (0.3, 0.105055) -= (0.3, 0.105055)
                (0.4, 0.464062) += (0.4, 0.120153) -= (0.4, 0.120153)
                (0.5, 0.436458) += (0.5, 0.132011) -= (0.5, 0.132011)
                (0.6, 0.264063) += (0.6, 0.114148) -= (0.6, 0.114148)
                (0.7, -0.0723958) += (0.7, 0.109147) -= (0.7, 0.109147)
                (0.8, -0.6) += (0.8, 0.0501976) -= (0.8, 0.0501976)
                (0.9, -0.8625) += (0.9, 0.0155489) -= (0.9, 0.0155489)
                (1, -0.95) += (1, 3.33136e-16) -= (1, 3.33136e-16)
            };

            \addplot[smooth, mark=square, error bars/.cd, y dir=both, y explicit] plot coordinates {
                (0, 0.0208333) += (0, 0.0295748) -= (0, 0.0295748)
                (0.1, 0.120833) += (0.1, 0.0637783) -= (0.1, 0.0637783)
                (0.2, 0.269792) += (0.2, 0.125828) -= (0.2, 0.125828)
                (0.3, 0.442708) += (0.3, 0.173907) -= (0.3, 0.173907)
                (0.4, 0.95) += (0.4, 0.264929) -= (0.4, 0.264929)
                (0.5, 0.803125) += (0.5, 0.28492) -= (0.5, 0.28492)
                (0.6, 0.627083) += (0.6, 0.278244) -= (0.6, 0.278244)
                (0.7, 0.20625) += (0.7, 0.235489) -= (0.7, 0.235489)
                (0.8, -0.339583) += (0.8, 0.131128) -= (0.8, 0.131128)
                (0.9, -0.727083) += (0.9, 0.0454223) -= (0.9, 0.0454223)
                (1, -0.9) += (1, 1.33255e-16) -= (1, 1.33255e-16)
            };

            \addplot[smooth, mark=triangle, error bars/.cd, y dir=both, y explicit] plot coordinates {
                (0, 0.0270833) += (0, 0.0584986) -= (0, 0.0584986)
                (0.1, 0.29375) += (0.1, 0.166863) -= (0.1, 0.166863)
                (0.2, 0.5875) += (0.2, 0.353052) -= (0.2, 0.353052)
                (0.3, 1.4) += (0.3, 0.636476) -= (0.3, 0.636476)
                (0.4, 1.15625) += (0.4, 0.605772) -= (0.4, 0.605772)
                (0.5, 1.1125) += (0.5, 0.586575) -= (0.5, 0.586575)
                (0.6, 0.991667) += (0.6, 0.567831) -= (0.6, 0.567831)
                (0.7, 1.0125) += (0.7, 0.614696) -= (0.7, 0.614696)
                (0.8, 0.372917) += (0.8, 0.483237) -= (0.8, 0.483237)
                (0.9, -0.577083) += (0.9, 0.10909) -= (0.9, 0.10909)
                (1, -0.8) += (1, 3.10927e-16) -= (1, 3.10927e-16)
            };
            \end{axis}
            \end{tikzpicture}
            \caption{$\ell_n = n (\ln{n})^4$}\label{fig:Relative-accuracy-vs-perturbation-probability-dense}
    \end{subfigure}

    \caption{%
    These plots display the relative accuracies $(\hat{K} - K)/K$ attained versus the perturbation probability $\varepsilon$.
    Here, $n = 1000$, and each sample mean is the result of $96$ independent replications.
    }
    \label{fig:Relative-accuracy-vs-perturbation-probability}
\end{figure}

\paragraph{Robustness to external noise depends on the path length}
Observe firstly in \Cref{fig:Relative-accuracy-vs-perturbation-probability-sparse} that for $\ell_n/n = (\ln n)^3$, the performance of the method appears to be insensitive to the perturbation probability when $\varepsilon \lesssim 0.4$.
It obviously makes sense that the performance decreases as $\varepsilon \to 1$, because the dominant behavior is then of an $n$-dimensional \gls{MC}.
However, it is surprising that even for relatively large perturbation probabilities ($\varepsilon \lesssim 0.4$), the number of clusters $K$ could still be detected.
After all, the simulated \gls{MC} here is of rank $n$ with overwhelming probability and \emph{not} rank $K$ whenever $\varepsilon > 0$.

This is further illustrated in \Cref{fig:Relative-accuracy-vs-perturbation-probability-dense}, where we repeat the experiment with $\ell_n/n = (\ln n)^4$ instead.
In this case, the trajectory is relatively dense, and our method anticipates lower background noise levels.
Consequently, the method is more sensitive to the noise level $\varepsilon$, which begins to affect the results almost immediately.

\subsection{Comparison to alternative methods}
\label{sec:results_alternative_methods}

Next, we compare our method (ALG2) to five off--the--shelves alternatives.
While these methods are not tailored to \glspl{BMC} and consequently do not come with applicable performance guarantees, it will still be interesting to see how performant they are.
We refer to \cref{sec:alternative_methods} for details on each alternative method; and we refer to \Cref{sec:Generating_BMCs_with_random_parameters} for details on our test scenarios.

\subsubsection{Brief overview of the alternative methods studied}

We investigate the following alternative methods:

\begin{enumerate}
    \item[HDBS]
    This method uses HDBSCAN to cluster the points of the embedding $\hat{X}$ \cite{8215642}.
    The minimum cluster size is set to the neighborhood size threshold used by \cref{alg:kpost}; recall \eqref{eqn:paramter_parameterization}.

    \item[MEGH]
    This method uses the maximum eigengap heuristic to estimate $K$.

    \item[LLSC]
    This method maximizes a cross-validated log-likelihood.
    It uses half of the observations to estimate a model \gls{BMC} $(\hat{p}^{(K')},\hat{\sigma}^{(K')})$ for $K'=1,2,\ldots$, and uses the other half as held-out validation data to compute the log-likelihood \cite{smyth2000model}.
    Specifically, this method estimates $\hat{\sigma}^{(K')}$ using the spectral clustering algorithm described in \cite{ClusterBMC2017}.

    \item[LLCI]
    This method is a refined version of the previous.
    Specifically, it further refines $\hat{\sigma}^{(K')}$ via the iterative, likelihood-based improvement algorithm described in \cite{ClusterBMC2017}.

    \item[CAIC]
    This method uses all observations to estimate $\hat{p}^{(K')}$ and $\hat{\sigma}^{(K')}$ for $K'=1,2,\ldots$ using the improvement algorithm described in \cite{ClusterBMC2017}.
    Subsequently, for each $K'$, the \gls{CAIC} \cite{bozdogan1987model} is calculated and the minimizing $K'$ is returned.
\end{enumerate}

\subsubsection{Brief overview of the different test scenarios used}

We test all algorithms in the following scenarios:

\begin{enumerate}
    \item[Test 1]
    We sample ``assortative'' \glspl{BMC} with $p_{i,i}=p_0=0.8$ fixed for all $i\in [K]$, and off-diagonal elements $\{p_{i,j}\}_{i\neq j}$ and cluster fractions $\alpha$ sampled uniformly at random from the simplex.
    Details are in \Cref{sec:assortative_BMCs}.

    \item[Test 2]
    We sample \glspl{BMC} with $p$ and $\alpha$ sampled uniformly at random from the simplex.
    Details can be found in \Cref{sec:uniform_BMCs}.

    \item[Test 3]
    We sample from perturbed \glspl{BMC} as in \eqref{eqn:noisy_transition_matrix} with perturbation strength $\epsilon=0.2$.
    For the \gls{BMC} component we fix the cluster fractions $\alpha_k=1/K$ and sample the cluster transition matrix $p$ uniformly at random from the simplex.

    \item[Test 4]
    We sample \glspl{BMC} with fixed cluster sizes $\alpha_k=1/K$ and cluster transition matrices of rank $\lceil K/2\rceil$ according to \Cref{ex:dot_product_model}.
    For details, see \Cref{sec:low_rank_BMCs}.
\end{enumerate}

Essentially, these tests range from easy to more difficult.
In our experiment, we will also varying the path length and the number of clusters.

\subsubsection{Results}

The results are summarized in \cref{tab:Performance-of-the-different-methods}.

\begin{table}[htpb]
    \centering
    \setlength{\tabcolsep}{6pt}
    {
        \small
        \begin{tabular}{c@{\hspace{7pt}}c@{\hspace{7pt}}c|c@{\hspace{7pt}}cc@{\hspace{7pt}}cc@{\hspace{7pt}}c|c@{\hspace{7pt}}cc@{\hspace{7pt}}cc@{\hspace{7pt}}c}
            \toprule
            $l_n$     & $K$ & {Test} & \multicolumn{2}{c}{ALG2} & \multicolumn{2}{c}{HDBS}          & \multicolumn{2}{c|}{MEGH}      & \multicolumn{2}{c}{LLSC} & \multicolumn{2}{c}{LLCI} & \multicolumn{2}{c}{CAIC}\\
            &     &        & $\mu_{\hat{K}}$ & $\sigma_{\hat{K}}$ & $\mu_{\hat{K}}$ & $\sigma_{\hat{K}}$ & $\mu_{\hat{K}}$ & $\sigma_{\hat{K}}$ & $\mu_{\hat{K}}$ & $\sigma_{\hat{K}}$ & $\mu_{\hat{K}}$ & $\sigma_{\hat{K}}$ & $\mu_{\hat{K}}$ & $\sigma_{\hat{K}}$\\
            \midrule

            \multirow{4}{*}{$n\ln(n)^2$} & \multirow{4}{*}{$3$}
              & 1 & $ \emph{3.0} $ & $0.0 $ & $ \emph{3.0} $ & $0.0 $ & $ 3.9 $ & $ 0.3 $ & $ 3.7 $ & $ 2.0 $ & $ \emph{3.0} $ & $0.0 $ & $ 9.4 $ & $ 1.4 $ \\
            & & 2 & $ 1.3 $ & $ 1.3 $ & $ 1.0 $ & $ 1.3 $ & $ 2.3 $ & $ 1.1 $ & $ 4.9 $ & $ 3.1 $ & $ \emph{3.5} $ & $1.4 $ & $ 9.0 $ & $ 1.4 $ \\
            & & 3 & $ 2.3 $ & $ 0.8 $ & $ 2.4 $ & $ 0.6 $ & $ 2.4 $ & $ 1.1 $ & $ 3.7 $ & $ 2.2 $ & $ \emph{3.5} $ & $1.6 $ & $ 8.8 $ & $ 1.6 $ \\
            & & 4 & $ 1.2 $ & $ 1.1 $ & $ 1.0 $ & $ 1.1 $ & $ 2.5 $ & $ 1.4 $ & $ 3.5 $ & $ 2.0 $ & $ \emph{2.9} $ & $0.9 $ & $ 9.1 $ & $ 1.3 $ \\ [0.5em]

            \multirow{4}{*}{$n\ln(n)^2$} & \multirow{4}{*}{$6$}
              & 1 & $ \emph{6.0} $ & $ 0.3 $ & $ \emph{6.0} $ & $0.2 $ & $ 7.0 $ & $ 1.2 $ & $ 9.8 $ & $ 5.1 $ & $ 6.8 $ & $ 2.2 $ & $ 8.3 $ & $ 1.1 $ \\
            & & 2 & $ 1.1 $ & $ 1.6 $ & $ 0.8 $ & $ 1.5 $ & $ 2.3 $ & $ 0.6 $ & $ 10.9 $ & $ 4.2 $ & $ \emph{7.4} $ & $2.2 $ & $ 9.7 $ & $ 1.7 $ \\
            & & 3 & $ 3.8 $ & $ 1.1 $ & $ 2.8 $ & $ 1.0 $ & $ 2.5 $ & $ 1.0 $ & $ \emph{7.2} $ & $3.2 $ & $ 8.0 $ & $ 2.7 $ & $ 8.7 $ & $ 1.6 $ \\
            & & 4 & $ 0.9 $ & $ 1.2 $ & $ 0.5 $ & $ 1.0 $ & $ 2.5 $ & $ 0.7 $ & $ 7.1 $ & $ 3.9 $ & $ \emph{6.6} $ & $2.2 $ & $ 9.8 $ & $ 1.7 $ \\ [0.5em]

            \multirow{4}{*}{$n\ln(n)^2$} & \multirow{4}{*}{$10$}
              & 1 & $ \emph{10.0} $ & $ 0.2 $ & $ \emph{10.0} $ & $0.1 $ & $ 11.0 $ & $ 0.6 $ & $ 23.0 $ & $ 3.0 $ & $ 17.5 $ & $ 3.3 $ & $ 10.6 $ & $ 0.6 $ \\
            & & 2 & $ 0.5 $ & $ 0.9 $ & $ 0.1 $ & $ 0.5 $ & $ 2.3 $ & $ 0.6 $ & $ 15.3 $ & $ 3.9 $ & $ \emph{11.9} $ & $3.9 $ & $ 11.9 $ & $ 1.9 $ \\
            & & 3 & $ 4.9 $ & $ 1.4 $ & $ 2.0 $ & $ 1.2 $ & $ 2.6 $ & $ 1.1 $ & $ \emph{9.9} $ & $4.3 $ & $ 13.4 $ & $ 3.1 $ & $ 11.7 $ & $ 2.3 $ \\
            & & 4 & $ 0.8 $ & $ 0.9 $ & $ 0.4 $ & $ 0.8 $ & $ 2.8 $ & $ 1.1 $ & $ \emph{10.1} $ & $4.9 $ & $ 11.5 $ & $ 3.8 $ & $ 12.2 $ & $ 2.0 $ \\ [0.5em]
            \midrule

            \multicolumn{3}{l|}{Avg. $|\mu_{\hat{K}}-K|/K$} & $ 0.5$ & & $ 0.5$ & & $ 0.4$ & & $ 0.4$ & & $ \mathbf{0.2} $ & & $ 0.9$ & \\ [0.1em]
            \multicolumn{3}{l|}{Avg. $\sigma_{\Khat}/K$} & & $ 0.18 $ & & $ 0.16 $ & & $ 0.19 $ & & $ 0.62 $ & & $ 0.36 $ & & $ 0.44 $ \\ [0.1em]

            \midrule
            \multirow{4}{*}{$n^2$} & \multirow{4}{*}{$3$}
              & 1 & $ \emph{3.0} $ & $0.0 $ & $ \emph{3.0} $ & $ 0.1 $ & $ 2.9 $ & $ 0.3 $ & $ \emph{3.0} $ & $0.0 $ & $ \emph{3.0} $ & $0.0 $ & $ 9.0 $ & $ 1.4 $ \\
            & & 2 & $ 3.6 $ & $ 3.9 $ & $ 7^* $ & $ 19 $ & $ 2.1 $ & $ 0.3 $ & $ 3.2 $ & $ 0.8 $ & $ \emph{3.1} $ & $0.7 $ & $ 8.7 $ & $ 1.4 $ \\
            & & 3 & $ 2.6 $ & $ 0.6 $ & $ 9^* $ & $ 24 $ & $ 2.3 $ & $ 1.3 $ & $ 3.1 $ & $ 0.4 $ & $ \emph{3.0} $ & $0.1 $ & $ 8.4 $ & $ 1.6 $ \\
            & & 4 & $ 3.4 $ & $ 3.7 $ & $ 31^* $ & $ 46 $ & $ 2.5 $ & $ 3.3 $ & $ 3.4 $ & $ 1.8 $ & $ \emph{3.0} $ & $0.3 $ & $ 9.1 $ & $ 1.3 $ \\ [0.5em]

            \multirow{4}{*}{$n^2$} & \multirow{4}{*}{$6$}
              & 1 & $ \emph{6.0} $ & $0.1 $ & $ \emph{6.0} $ & $0.1 $ & $ 4.9 $ & $ 1.0 $ & $ 6.0 $ & $ 0.1 $ & $ 6.0 $ & $ 0.3 $ & $ 8.3 $ & $ 0.8 $ \\
            & & 2 & $ 7.3 $ & $ 3.8 $ & $ \emph{5.9} $ & $0.3 $ & $ 2.2 $ & $ 0.5 $ & $ 9.4 $ & $ 3.4 $ & $ 6.4 $ & $ 1.4 $ & $ 9.1 $ & $ 1.5 $ \\
            & & 3 & $ 5.2 $ & $ 0.9 $ & $ \emph{6.0} $ & $ 0.1 $ & $ 2.4 $ & $ 0.6 $ & $ 6.0 $ & $ 0.1 $ & $ \emph{6.0} $ & $0.0 $ & $ 7.6 $ & $ 0.8 $ \\
            & & 4 & $ \emph{6.6} $ & $ 3.6 $ & $ \emph{6.6} $ & $ 8.6 $ & $ 2.2 $ & $ 0.7 $ & $ 9.3 $ & $ 3.7 $ & $ 6.7 $ & $ 1.5 $ & $ 9.2 $ & $ 1.6 $ \\ [0.5em]

            \multirow{4}{*}{$n^2$} & \multirow{4}{*}{$10$}
              & 1 & $ \emph{10.0} $ & $0.0 $ & $ \emph{10.0} $ & $ 0.1 $ & $ 3.2 $ & $ 0.7 $ & $ 13.0 $ & $ 1.0 $ & $ 11.4 $ & $ 1.3 $ & $ 11.2 $ & $ 0.7 $ \\
            & & 2 & $ 12.0 $ & $ 4.1 $ & $ \emph{9.8} $ & $0.5 $ & $ 2.4 $ & $ 0.7 $ & $ 17.2 $ & $ 2.6 $ & $ 12.7 $ & $ 4.0 $ & $ 12.7 $ & $ 2.6 $ \\
            & & 3 & $ 9.0 $ & $ 0.9 $ & $ \emph{10.0} $ & $0.0 $ & $ 2.6 $ & $ 0.7 $ & $ 13.5 $ & $ 1.6 $ & $ 10.1 $ & $ 0.3 $ & $ 10.2 $ & $ 0.4 $ \\
            & & 4 & $ 11.2 $ & $ 4.2 $ & $ \emph{9.3} $ & $0.9 $ & $ 2.3 $ & $ 0.6 $ & $ 16.2 $ & $ 3.1 $ & $ 13.3 $ & $ 3.8 $ & $ 13.2 $ & $ 2.8 $ \\ [0.5em]
            \midrule
            \multicolumn{3}{l|}{Avg. $|\mu_{\hat{K}}-K|/K$} & $ \mathbf{0.1} $ & & $ 1.1$ & & $ 0.5$ & & $ 0.3$ & & $ \mathbf{0.1} $ & & $ 0.8$ & \\
            \multicolumn{3}{l|}{Avg. $\sigma_{\Khat}/K$} & & $ 0.41 $ & & $ 2.61 $ & & $ 0.22 $ & & $ 0.26 $ & & $ 0.16 $ & & $ 0.28 $ \\ [0.1em]
            \midrule
            \vspace{0pt}
        \end{tabular}
    }
    \caption{%
        Sample means and standard deviations of the estimated number of clusters for the different methods described in the main text with $n=1000$.
        The aggregate scores show the average relative error and average normalized standard deviation for each method across all scenarios at a given path length.
        For each combination of method, test, $\ell_n$ and $K$ we estimate a $95\%$-confidence interval for the sample mean $\mu_{\Khat}$ as if taking a fixed number of samples, and sample until the margin of error is at most $0.15$ with a minimum of $250$ samples.
        We report the number of samples and the margin of error in \Cref{tab:confidence-bounds-for-different-methods}.
        *Due to the large variance of the estimator in these scenarios, these sample means have wider estimated confidence intervals.
    }
    \label{tab:Performance-of-the-different-methods}
\end{table}

\paragraph{Crossvalidated log-likelihood with cluster improvement performs best}
At the shorter path length $\ell_n=n\ln(n)^2$, LLCI performs best on average.
The other methods only recover $K$ reasonably well in the relatively easy Test 1.

At the longer path length $\ell_n=n^2$, methods ALG2, HDBS and LLSC also perform well.
Focusing on ALG2, we see that it now matches the average performance of LLCI, but with more variance.
By comparison, the MEGH and CAIC methods essentially fail to detect $K$ in all but a few cases.

The solid performance of the log-likelihood methods, especially when using the improvement step of \cite{ClusterBMC2017}, presents an interesting avenue for future research.
However, from a practical point of view, we should note that LLCI is computationally more involved than ALG2.
LLCI estimates clusters for multiple candidates $K'$, and the improvement algorithm can take some time when $n$ is large.

\paragraph{Alternative density-based clustering methods can improve numerical performance}
At the shorter path length $\ell_n=n\ln(n)^2$, we find that HDBS outperforms ALG2 for $K = 6, 10$ but not $K = 3$.
Some general advantages of HDBS over other density-based clustering methods are that it can handle clusters of irregular shapes, and varying densities.
This makes the method flexible when perturbations or imbalanced visitation rates distort the embedding $\hat{X}$.

At the longer path length $\ell_n=n^2$, we see at the same time excellent performance of HDBS in many individual tests but also very poor performance in a few individual tests.
This is reflected in the decrease of average relative accuracy, and increase of average variance, at the bottom of the table.
On the contrary, ALG2 performed better on average in these tests cases due to HDBS's few outliers.

\section*{Acknowledgements}

This work is part of \emph{Valuable AI}, a research collaboration between the Eindhoven University of Technology and the Koninklijke KPN N.V.
Parts of this research have been funded by the \gls{EAISI}'s IMPULS program.

\bibliographystyle{unsrt}
\bibliography{refs}

\appendix
\section{Remaining proofs}
\label{sec:proof}

The proof of \cref{thm:Khat=K} adapts that of \cite[Theorem 2]{ClusterBMC2017} to account for the unknown value of $K$.
Our adaptations draw upon the proof techniques of \cite{NIPS2016_a8849b05}, which presents a spectral clustering algorithm that estimates the number of clusters in \glspl{SBM}.
However, different from both of these works, we first generalize the tight spectral norm bounds of \cite{Bounds2023}, and then use this to establish an explicit convergence result for \Cref{alg:kpre}.
At the same time, we improve the bound on the clustering error of \Cref{alg:kmeans} by an additional log-factor compared to \cite[Theorem 2]{ClusterBMC2017}.

\subsection{Analyzing \texorpdfstring{\Cref{alg:kpre}}{Algorithm 1}}
\label{sec:proofskpre}

Let $N$ denote the matrix with elements
\begin{equation}
    N_{x,y}\eqdef \mbbE[\hat{N}_{x,y}]=\ell_n\Pi_xP_{x,y}\quad \textnormal{for }x,y\in[n].
\end{equation}
We first generalize the results of \cite{Bounds2023} so that they hold under the weaker \cref{asm:assumption1,asm:assumption2}.
Compared to \cite{Bounds2023} we allow the cluster transition matrix $p$ to have zero entries, and to be rank-deficient.
Specifically, we will confirm the following:

\begin{theorem}
    [{Adapted from \cite[Theorem 3]{Bounds2023}}]
    \label{thm:spectral_norm_bound}

    Presume \cref{asm:assumption1,asm:assumption2}.
    The following holds:
    \begin{enumerate}
        \item[(a)] If $\ell_n=\Omega(n \ln n)$, then
              \begin{equation}
                  \sigma_1(\hat{N}-N)=O_{\mbbP}(\sqrt{\ell_n/n}).
              \end{equation}
        \item[(b)] If $\ell_n = \omega(n)$ and $\Gamma^{\mathrm{c}}$ is a set of size $\floor{ne^{-\ell_n/n}}$ containing the states with the highest number of visits, then
              \begin{equation}
                  \sigma_1(\hat{N}_{\Gamma}-N)=O_{\mbbP}(\sqrt{\ell_n/n}).
              \end{equation}
    \end{enumerate}
\end{theorem}
In addition, we will prove the following corollary:
\begin{corollary}
    [{Adapted from \cite[Corollary 4]{Bounds2023}}]
    \label{cor:singvalscalingrankp}
    Presume \cref{asm:assumption1,asm:assumption2}.
    Let $\hat{N}_{\Gamma}$ be as in \cref{alg:kpre}, where $\Gamma\subset [n]$ is the set of states obtained after removing the $\floor{n\exp(-\ell_n/n)}$ most visited states.
    If $\ell_n = \omega(n)$, then
    \begin{equation}
        \label{eqn:singvalscalingrankp}
        \sigma_i(\hat{N}_\Gamma)=
        \begin{cases}
            \Theta_{\mathbb{P}}\lb\frac{\ell_n}{n}\rb   & i\leq \rank(p), \\
            O_{\mathbb{P}}\lb\sqrt{\frac{\ell_n}{n}}\rb & i> \rank(p).    \\
        \end{cases}
    \end{equation}
\end{corollary}
Observe that the statement of \cref{thm:spectral_norm_bound} is identical to that of \cite[Theorem 3]{Bounds2023}, up to \cref{asm:assumption1,asm:assumption2}.
Similarly, \cref{cor:singvalscalingrankp} is analogous to \cite[Corollary 4]{Bounds2023} up to the stronger implication that $\sigma_i(\hat{N}_{\Gamma})=\Theta_{\mbbP}(\ell_n/n)$ for $i\leq \rank(p)$ rather than $i\leq K$, and \cref{asm:assumption1,asm:assumption2}.

Given \Cref{cor:singvalscalingrankp}, the proof of \Cref{prop:limit_kpre} is now straightforward.

\subsubsection{Proof of \texorpdfstring{\Cref{prop:limit_kpre}}{Proposition 3.1} assuming \texorpdfstring{\Cref{cor:singvalscalingrankp}}{Corollary A.2}}
\label{sec:proof_of_limit_kpre_assuming_cor_singvalscalingrankp}

We introduce the shorthand notation $r_p= \rank(p)\leq K$.
Observe that the assumptions of \cref{prop:limit_kpre} ensure that the consequences of  \cref{cor:singvalscalingrankp} hold.

We first show that $\Ktilde\leq r_p$ with high probability as $n \rightarrow \infty$.
By definition \eqref{eqn:kpredef}, $\Ktilde\leq r_p$ if and only if $\sigma_{r_p+1}(\hat{N}_{\Gamma})<\gamma_n$ so that
\begin{equation}
    \label{eqn:Kspec_upper_bound}
    \mbbP\bkt{\Ktilde\leq r_p} = \mbbP\bkt{\sigma_{r_p+1}(\hat{N}_{\Gamma})<\gamma_n} = 1- \mbbP\bkt{\sigma_{r_p+1}(\hat{N}_{\Gamma})\geq \gamma_n}.
\end{equation}
It therefore suffices to check that $\sigma_{r_p+1}(\hat{N}_{\Gamma})=o_{\mbbP}(\gamma_n)$, in which case the right-hand side of \eqref{eqn:Kspec_upper_bound} converges to $1$ as $n \rightarrow \infty$.
This is the case because \cref{cor:singvalscalingrankp} implies that
\begin{equation}
    \label{eqn:singvalscaling1}
    \sigma_{r_p+1}(\hat{N}_\Gamma)=O_{\mbbP}(\sqrt{\ell_n/n}).
\end{equation}
Since by assumption $\gamma_n=\omega(\sqrt{\ell_n/n})$, it follows from \eqref{eqn:singvalscaling1} that $\sigma_{r_p+1}(\hat{N}_{\Gamma})=o_{\mbbP}(\gamma_n)$.

Similarly, we show that $\Ktilde\geq r_p$ with high probability.
Again, by definition \eqref{eqn:kpredef} $\Ktilde\geq r_p$ if and only if $\sigma_{r_p}(\hat{N}_{\Gamma})\geq \gamma_n$ so that
\begin{equation}
    \mbbP\bkt{\Ktilde\geq r_p} = \mbbP\bkt{\sigma_{r_p}(\hat{N}_{\Gamma})\geq\gamma_n} = 1 - \mbbP\bkt{\sigma_{r_p+1}(\hat{N}_{\Gamma})< \gamma_n}.
\end{equation}
It now suffices to show that $\sigma_{r_p}(\hat{N}_{\Gamma})=\omega_{\mbbP}(\gamma_n)$. From \cref{cor:singvalscalingrankp} it follows that
\begin{equation}
    \label{eqn:singvalscaling2}
    \sigma_{r_p}(\hat{N}_\Gamma)=\Theta_{\mbbP}\left(\ell_n/n\right),
\end{equation}
and by assumption $\gamma_n=o(\ell_n/n)$ so that $\sigma_{r_p}(\hat{N}_{\Gamma})=\omega_{\mbbP}(\gamma_n)$.
That is it.\qed

\subsubsection{Proof of \texorpdfstring{\cref{thm:spectral_norm_bound}}{Theorem A.1}}
\label{sec:proof_of_spectral_norm_bound}
It remains to prove \Cref{thm:spectral_norm_bound,cor:singvalscalingrankp}.
Inspecting the proof of \cite[Theorem 3]{Bounds2023} carefully, we find that \cite[Assumption 1]{Bounds2023} is used only via its implication \cite[Lemma 5]{Bounds2023}.
It states that for all $x,y\in[n]$, $N_{x,y}=\Theta(\ell_n/n^2)$ and $P_{x,y}=\Theta(1/n)$.
Moreover, the lower bounds implied by these two statements are used only in the proof of the intermediate result \cite[Proposition 9]{Bounds2023}.
Consequently, to prove \cref{thm:spectral_norm_bound} it suffices to prove the upper bounds of \cite[Lemma 5]{Bounds2023}; and to then prove \cite[Proposition 9]{Bounds2023} using \emph{only} these upper bounds.
Our assumptions are chosen precisely so as to ensure the upper bounds:
\begin{lemma}
    \label{lem:elementwise_bd_N}
    Presume \cref{asm:assumption1,asm:assumption2}.
    Then there exist constants $\mf{n}_2,\mf{p}_2>0$ independent of $n$, such that for all $x,y\in[n]$ and for sufficiently large $n$, $N_{x,y}\leq \mf{n}_2\ell_n/n^2$ and $P_{x,y}\leq \mf{p}_2/n$.
\end{lemma}
\begin{proof}
    The proof of \cref{lem:elementwise_bd_N} follows from that of \cite[Lemma 5]{Bounds2023} by noting that \cref{asm:assumption2} guarantees the existence of a nonvanishing stationary distribution $\pi$, and that \cref{asm:assumption1} ensures a nonvanishing lower bound on the $\alpha_k$.
    Together, these two facts allow the proof of the upper bound in \cite[Lemma 5]{Bounds2023} to go through unchanged.
\end{proof}
Thus, to prove \cref{thm:spectral_norm_bound} it remains only to prove a variant of \cite[Proposition 9]{Bounds2023} that uses \cref{lem:elementwise_bd_N} instead.
To this aim, we first recall the following definitions from \cite{Bounds2023}: for any two subsets $\mcA,\mcB\subseteq [n]$, define
\begin{equation}
    e(\mcA,\mcB)\eqdef \sum_{x\in\mcA}\sum_{y\in\mcB}\hat{N}_{x,y}\quad \mathrm{and} \quad \mu(\mcA,\mcB)\eqdef \mbbE\bkt{e(\mcA,\mcB)}.
\end{equation}
Similarly, for any $\mcA,\mcB\subseteq[n]$, define $e_{\Gamma}(\mcA,\mcB)\eqdef \sum_{x\in\mcA}\sum_{y\in\mcB}(\hat{N}_{\Gamma})_{x,y}$.
The following is \cite[Definition 8]{Bounds2023}:
\begin{definition}[Discrepancy property]\label{def:discrepancy_property}
    Let $\mf{d}_1,\mf{d}_2>0$ be two constants independent of $n$.
    We say that $\hat{N}$ is $(\mf{d}_1,\mf{d}_2)$-discrepant if for every pair $(\mcA,\mcB)\subset[n]^2$ one of the following holds:
    \begin{enumerate}
        \item[(i)]\label{eqn:discrepancy_i} $\frac{e(\mcA,\mcB)n^2}{\ell_n|\mcA||\mcB|}\leq \mf{d}_1$,
        \item[(ii)]\label{eqn:discrepancy_ii} $e(\mcA,\mcB)\ln \frac{e(\mcA,\mcB)n^2}{|\mcA||\mcB|\ell_n}\leq \mf{d}_2(|\mcA|\vee |\mcB|)\ln\frac{n}{|\mcA|\vee |\mcB|}$.
    \end{enumerate}
\end{definition}
We next prove that $\hat{N}$ ($\hat{N}_{\Gamma}$) satisfies \cref{def:discrepancy_property} with high probability whenever $\ell_n=\Omega(n\ln n)$ ($\ell_n = \omega(n)$) and \cref{asm:assumption1,asm:assumption2} hold.
\begin{proposition}
    \label{prop:discrepancy_property}
    For any \gls{BMC} satisfying \cref{asm:assumption1,asm:assumption2}, and any constant $c>0$, there exist sufficiently large constants $\mf{b}_3,\mf{b}_4,\mf{d}_1,\mf{d}_2>0$ independent of $n$ such that the following hold:
    \begin{enumerate}
        \item[(a)] If $\ell_n=\Omega(n\ln n)$ and $\max_{y\in[n]}\{e([n],y),e(y,[n])\}\leq \mf{b}_3\ell_n/n $, then for sufficiently large $n$, $\hat{N}$ is $(\mf{d}_1,\mf{d}_2)$-discrepant with probability at least $1-1/n^c$.
        \item[(b)] If $\ell_n=\Omega(n\ln n)$, $\Gamma^{\mathrm{c}}$ is a set of size $\floor{ne^{-\ell_n/n}}$ containing the states with the highest number of visits, and moreover $\max_{y\in[n]}\{e(\Gamma,y),e(y,\Gamma)\}\leq \mf{b}_4\ell_n/n $, then for sufficiently large $n$, $\hat{N}_{\Gamma}$ is $(\mf{d}_1,\mf{d}_2)$-discrepant with probability at least $1-1/n^c$.
    \end{enumerate}
\end{proposition}
\cref{prop:discrepancy_property} replaces \cite[Proposition 9]{Bounds2023}, and does not use \cite[Assumption 1]{Bounds2023}.
We have also improved the bound on the tail probability to $1/n^c$ for arbitrary $c>0$.
Its proof follows that of \cite[Proposition 9]{Bounds2023}, which relies on the auxiliary results \cite[Corollary 19]{Bounds2023} and \cite[Lemma 20]{Bounds2023}.
We therefore start by using our \cref{lem:elementwise_bd_N} to prove analogs of these two results that hold under our weaker \cref{asm:assumption1,asm:assumption2}.
This requires a more careful analysis of the concentration of $\hat{N}$.

The following is an immediate corollary of \cref{lem:elementwise_bd_N} which replaces \cite[Corollary 19]{Bounds2023}.
\begin{corollary}
    \label{cor:bound_mu_AB}
    Presume \cref{asm:assumption1,asm:assumption2}.
    There exists a constant $\mf{m}_2>0$ independent of $n$ such that for all $\mcA,\mcB\subseteq[n]$ and for sufficiently large $n$, $\mu(\mcA,\mcB)\leq \mf{m}_2|\mcA||\mcB|\ell_n/n^2$.
\end{corollary}
Similarly, the following \cref{lem:concentration_e_AB} will replace \cite[Lemma 20]{Bounds2023}.
\begin{lemma}
    [{Adapted from \cite[Lemma 20]{Bounds2023}}]
    \label{lem:concentration_e_AB}
    Presume \cref{asm:assumption1,asm:assumption2}.
    Then there exists a constant $k_0>0$ independent of $n$ and an integer $m>0$ such that for all $n\geq m$, all $k\geq k_0$, and all $\mcA,\mcB\subseteq [n]$,
    \begin{equation}
        \label{eqn:new_poisson_bound}
        \mbbP\bkt{e(\mcA,\mcB)\geq k\mu(\mcA,\mcB)}\leq 2\exp\lb-\frac{1}{4}\mu(\mcA,\mcB)k\ln k\rb.
    \end{equation}
\end{lemma}
\begin{proof}
    The proof follows that of \cite[Lemma 20]{Bounds2023} closely.
    Let $e_0(\mcA,\mcB)\eqdef \sum_{t=0}^{\ceil{\ell_n/2}-1}\mathbbm{1}\{X_{2t}\in \mcA, X_{2t+1}\in \mcB\}$ refer to all even-numbered transitions in the sum $e(\mcA,\mcB)$, and let $e_1(\mcA,\mcB)$ refer to all odd-numbered transitions in the sum $e(\mcA,\mcB)$.
    Then, \cite[(A.27)--(A.29)]{Bounds2023} carry over unchanged, and we find that
    \begin{equation}
        \mbbP\bkt{2e_0(\mcA,\mcB)-\mu(\mcA,\mcB)\geq a}\leq \inf_{h> 0}\bigl\{\exp\bigl(-h(a+\mu(\mcA,\mcB))\mbbE\bigl[e^{2he_0(\mcA,\mcB)}\bigr]\bigr)\bigr\}.
    \end{equation}
    A similar result holds for $e_1$.
    We next prove \cref{lem:concentration_e_AB} by giving a tighter bound on the moment generating function $\mbbE\bkt{e^{2he_0(\mcA,\mcB)}}$.
    In particular, we follow the proof of \cite[Lemma 20]{Bounds2023} until \cite[(A.31)]{Bounds2023}, but improve the final inequality therein.

    The key observation that allows us to do this is that for some $\mf{p}_3>0$ independent of $n$, and for sufficiently large $n$,
    \begin{equation}
        \label{eqn:ratio_bound_P_Pi}
        P_{x,y}/\Pi_y\leq \mf{p}_3, \quad \forall x,y\in[n].
    \end{equation}
    This follows because $\Pi_x=\Theta(1/n)$ (recall the discussion following \cref{asm:assumption2}), and for sufficiently large $n$ $P_{x,y}\leq \mf{p}_2/n$ as per \cref{lem:elementwise_bd_N}.
    Consequently, there exists some $\mf{p}_3>0$ such that for sufficiently large $n$, $P_{x,y}/\Pi_y\leq \mf{p}_3$ for all $x,y\in[n]$.

    Using \eqref{eqn:ratio_bound_P_Pi} and that $\Pi_x>0$ we can improve the final bound of \cite[(A.31)]{Bounds2023} for sufficiently large $n$ to
    \begin{equation}
        \label{eqn:MGF_term_bound}
        \begin{split}
            1+\sum_{x\in\mcA}\sum_{y\in\mcB}e^{2h}
            P_{X_{\ell_n-2},x}P_{x,y} & = 1+\sum_{x\in\mcA}\sum_{y\in\mcB}e^{2h}\frac{P_{X_{\ell_n-2},x}}{ \Pi_x}N_{x,y}/\ell_n \\
                                      & \leq 1+e^{2h}\mf{p}_3\mu(\mcA,\mcB)/\ell_n.
        \end{split}
    \end{equation}
    Observe that \eqref{eqn:MGF_term_bound} is tighter than \cite[(A.31)]{Bounds2023} precisely when $\mu(\mcA,\mcB)$ is $o(\ell_n|\mcA||\mcB|/n^2)$, which is avoided by the lower bound of \cite[Corollary 19]{Bounds2023}.
    Resuming the proof of \cite[Lemma 20]{Bounds2023} with \eqref{eqn:MGF_term_bound} replacing \cite[(A.31)]{Bounds2023}, we then obtain the following refinement of \cite[(A.33)]{Bounds2023}:
    \begin{equation}\label{eqn:refinement_of_A33}
        \mbbP\bkt{2e_0(\mcA,\mcB)-\mu(\mcA,\mcB)\geq a}
        \leq
        \inf_{h\geq 0}\bigl\{\exp\bigl(-h(a+\mu(\mcA,\mcB))+\frac{1}{2}e^{2h}\mf{p}_3\mu(\mcA,\mcB)\bigr)\bigr\}.
    \end{equation}
    The proof concludes as in \cite{Bounds2023}.
    Namely, we specify $a=(k-1)\mu(\mcA,\mcB)$ and optimize the right-hand side of \eqref{eqn:refinement_of_A33} with respect to $h$ to find the optimum
    \begin{equation}
        h_n^{\textnormal{opt}}=\frac{1}{2}\ln(k/\mf{p}_3),
    \end{equation}
    such that for all sufficiently large $n$,
    \begin{equation}
        \label{eqn:final_bound_e0}
        \begin{split}
            \mbbP\bkt{2e_0(\mcA,\mcB)\geq k\mu(\mcA,\mcB)}
            &
            \leq
            \exp\biggl( \frac{1}{2}k\mu(\mcA,\mcB)\bigl( 1-\ln(k/\mf{p}_3) \bigr)\biggr)
            \\
            &
            \leq
            \exp\biggl( -\frac{1}{4}\mu(\mcA,\mcB)k \ln k \biggr).
        \end{split}
    \end{equation}
    Here, the final bound holds for all $k\geq \exp\lb2+2\ln\mf{p}_3 \rb$.
    We can obtain the same bound as in \eqref{eqn:final_bound_e0} for $e_1$.
    Substituting these into \cite[(110)]{Bounds2023} gives \eqref{eqn:new_poisson_bound}.
\end{proof}

Finally, we use \cref{cor:bound_mu_AB,lem:concentration_e_AB} to prove \cref{prop:discrepancy_property}.
\begin{proof}[Proof of \cref{prop:discrepancy_property}.]
    We first prove the part without trimming, i.e., part $(a)$ of \cref{prop:discrepancy_property}.
    To this aim, we begin by making a preliminary reduction to the case that $(\mcA,\mcB)$ is such that $\mu(\mcA,\mcB)>0$.
    This is possible because $e(\mcA,\mcB)$ is nonnegative almost surely.
    Therefore $\mu(\mcA,\mcB)=\mbbE\bkt{e(\mcA,\mcB)}=0$ would imply that $e(\mcA,\mcB)=0$ almost surely.
    As a result, for any such pair $(\mcA,\mcB)$ and constant $\mf{d}_1>0$, \cref{def:discrepancy_property} $(i)$ holds almost surely and we are done.
    Therefore, proceed by assuming that $(\mcA,\mcB)$ are such that $\mu(\mcA,\mcB)>0$.

    The remainder of the proof then starts out similarly to that of \cite[Proposition 9]{Bounds2023}.
    In particular, we assume without loss of generality that $|\mcA|\leq|\mcB|$.
    Then, when $|\mcB|>n/\mathrm{e}$, $e(\mcA,\mcB)$ satisfies \cref{def:discrepancy_property} $(i)$ with $\mf{d}_1\geq \mathrm{e}\mf{b}_3$ as follows from the proof of \cite[Proposition 9]{Bounds2023} (using our \cref{cor:bound_mu_AB}).
    Hence, we further restrict to $0<|\mcA|\leq |\mcB|\leq n/e$.

    From \cref{lem:concentration_e_AB}, for all sufficiently large $n$, there exists $k_0>0$ independent of $n$ such that $\mbbP[e(\mcA,\mcB)\geq \allowbreak k\mu(\mcA,\mcB)]\leq 2\exp\lb-\mu(\mcA,\mcB)k\ln (k)/4\rb$ for all $k\geq k_0$.
    As in \cite{Bounds2023} we are free to presume $k_0\geq \max\{1,1/\mf{m}_2,\mf{m}_2\}$ where $\mf{m}_2$ is as in \cref{cor:bound_mu_AB}.

    Next, let $\mf{c}_2>0$ be some constant that is to be fixed later.
    Denote by $t^*(\mcA,\mcB)$ the unique solution to $t\ln t= (1/2)\mf{c}_2(|\mcB|/\mu(\mcA,\mcB))\ln (n/|\mcB|)>0$, which exists and is unique because the function $t\ln t$ is monotonically increasing for $t \geq1$ and $\mu(\mcA,\mcB)>0$ by our preliminary reduction.
    Observe that our definition of $t^*(\mcA,\mcB)$ differs slightly from that above \cite[(A.17)]{Bounds2023}.
    We claim that with our redefinition of $t^*(\mcA,\mcB)$, $\hat{N}$ is discrepant whenever the event
    \begin{equation}
        \label{eqn:event_E_def}
        \mcE\eqdef \bigcap_{\mcA,\mcB\subseteq[n]:|\mcA|\leq|\mcB|\leq n/e}\{e(\mcA,\mcB)\leq \mu(\mcA,\mcB)k^*(\mcA,\mcB)\} \textnormal{ with } k^*(\mcA,\mcB)\eqdef \max\{k_0,t^*(\mcA,\mcB)\}
    \end{equation}
    holds.

    When $(\mcA,\mcB)$ is such that $k^*(\mcA,\mcB)=k_0$, it is immediate from \cite[(A.18)]{Bounds2023} and \cref{cor:bound_mu_AB} that $e(\mcA,\mcB)$ satisfies \cref{def:discrepancy_property} $(i)$ when the event $\mcE$ holds:
    \begin{equation}
        \label{eqn:discrepancy_case_B_small_t_small}
        \frac{e(\mcA,\mcB)n^2}{\ell_n|\mcA||\mcB|}\overset{\eqref{eqn:event_E_def}}{\leq} k_0\frac{\mu(\mcA,\mcB)n^2}{\ell_n|\mcA||\mcB|}\leq k_0\mf{m}_2.
    \end{equation}

    Our proof that \cref{def:discrepancy_property} $(ii)$ holds for $(\mcA,\mcB)$ such that $k^*(\mcA,\mcB)=t^*(\mcA,\mcB)$ differs from the related part in \cite{Bounds2023}.
    Namely, because we lack the lower bound of \cite[Corollary 19]{Bounds2023} we must avoid prematurely upper bounding $\mu(\mcA,\mcB)$.
    Begin by observing that for any such pair $(\mcA,\mcB)$ and for all sufficiently large $n$, we have by \cref{cor:bound_mu_AB} the event $\mcE$ that
    \begin{equation}
        t^*(\mcA,\mcB)\geq \frac{e(\mcA,\mcB)}{\mu(\mcA,\mcB)} \geq\frac{n^2 e(\mcA,\mcB)}{\mf{m}_2\ell_n |\mcA||\mcB|}.
    \end{equation}
    Next, observe that by monotonicity, that $\ln t\leq \ln \mf{m}_2t^*(\mcA,\mcB)$ for all $t\leq \mf{m}_2t^*(\mcA,\mcB)$.
    Moreover, from the case distinction that $k^*(\mcA,\mcB)=t^*(\mcA,\mcB)$ and the lower bound on $k_0$, $\mf{m}_2t^*(\mcA,\mcB)\geq \mf{m}_2k_0\geq \mf{m}_2\max\{1,1/\mf{m}_2,\mf{m}_2\}\geq 1$, and thus $\ln \mf{m}_2t^*(\mcA,\mcB)\geq 0$.
    Consequently, for all $t\leq t^*(\mcA,\mcB)$, $t\ln \mf{m}_2t^*(\mcA,\mcB)\leq t^*(\mcA,\mcB)\ln \mf{m}_2t^*(\mcA,\mcB)$.
    From this, it follows that
    \begin{equation}
        \label{eqn:discrepancy_ii_bound}
        \frac{e(\mcA,\mcB)}{\mu(\mcA,\mcB)}\ln \frac{n^2e(\mcA,\mcB)}{\ell_n |\mcA||\mcB|}\leq t^*(\mcA,\mcB)\ln \mf{m}_2t^*(\mcA,\mcB)\overset{(ii)}{\leq} 2t^*(\mcA,\mcB)\ln t^*(\mcA,\mcB) \overset{(iii)}{=} \mf{c}_2\frac{|\mcB|}{\mu(\mcA,\mcB)}\ln \frac{n}{|\mcB|}.
    \end{equation}
    Here, $(ii)$ holds because $t^*(\mcA,\mcB)\geq \mf{m}_2$, and $(iii)$ because of the definition of $t^*(\mcA,\mcB)$.
    Rearranging then shows that \cref{prop:discrepancy_property} $(ii)$ holds with constant $\mf{d}_2=\mf{c}_2$.

    The proof is finalized by showing that for any $c>0$, there exists a constant $\mf{c}_2>0$ that is independent of $n$, such that for sufficiently large $n$ the event $\mcE$ holds with probability at least $1-1/n^{-c}$.
    As in \cite[(A.23)]{Bounds2023}, using a union bound and using instead our \cref{lem:concentration_e_AB}, it follows that for sufficiently large $n$,
    \begin{equation}
        \label{eqn:bound_on_prob_of_E_event1}
        \begin{split}
             & \mbbP\bkt{\bigcup_{\mcA,\mcB\subseteq[n]:|\mcA|\leq|\mcB|\leq n/e}\{e(\mcA,\mcB)>\mu(\mcA,\mcB)k^*(\mcA,\mcB)\}}                 \\
             & \leq \sum_{\mcA,\mcB\subseteq[n]:|\mcA|\leq|\mcB|\leq n/e}2\exp\lb-\frac{1}{4}\mu(\mcA,\mcB)k^*(\mcA,\mcB)\ln k^*(\mcA,\mcB)\rb.
        \end{split}
    \end{equation}
    From the lower bound on $k_0$, it follows that $k^*(\mcA,\mcB)=\max\{t^*(\mcA,\mcB),k_0\}\geq \max\{t^*(\mcA,\mcB),1\}$.
    Since for $t \geq 1$ and $t \geq t'$, $t\ln t \geq t'\ln t'$, it follows furthermore that $k^*(\mcA,\mcB)\ln k^*(\mcA,\mcB)\geq t^*(\mcA,\mcB)\ln t^*(\mcA,\mcB)$.
    Using this fact and our definition of $t^*(\mcA,\mcB)$ then gives that
    \begin{equation}
        \label{eqn:bound_on_prob_of_E_event2}
        \eqref{eqn:bound_on_prob_of_E_event1}\leq \sum_{\mcA,\mcB\subseteq[n]:|\mcA|\leq|\mcB|\leq n/e}2\exp\lb -\frac{\mf{c}_2}{8}|\mcB|\ln\frac{n}{|\mcB|} \rb.
    \end{equation}
    Then, from \cite[(A.24)]{Bounds2023} (with $\mf{m}_1\mf{c}_2/\mf{m}_2$ replaced by $\mf{c}_2$):
    \begin{equation}
        \eqref{eqn:bound_on_prob_of_E_event2}\leq n^{-\frac{\mf{c}_2}{8}+7} \overset{(iv)}{\leq} n^{-c},
    \end{equation}
    where $(iv)$ holds for all $\mf{c}_2\geq 8(c+7)$.
    It follows that $\hat{N}$ is $(\mf{d}_1,\mf{d}_2)$-discrepant if we choose $\mf{d}_2=\mf{c}_2\geq 8(c+7)$.
    Moreover, from \cite[(A.16)]{Bounds2023} and \eqref{eqn:discrepancy_case_B_small_t_small} we require that $\mf{d}_1\geq \max\{\mathrm{e}\mf{b}_3,k_0\mf{m}_2\}$.

    As in \cite{Bounds2023}, the case with trimming is obtained by the same argument.
    Notice that for any $\mcA,\mcB\subseteq[n]$, $e_{\Gamma}(\mcA,\mcB)\leq e(\mcA,\mcB)$ almost surely.
    It again follows from the proof of \cite[Proposition 9]{Bounds2023} that when $|\mcB|>n/\mathrm{e}$, $\hat{N}_{\Gamma}$ satisfies \cref{def:discrepancy_property} $(i)$, but now with paramter $\mf{d}_1\geq \mathrm{e}\mf{b}_4$.
    When $|\mcB|\leq n/\mathrm{e}$, equations \eqref{eqn:event_E_def}--\eqref{eqn:discrepancy_ii_bound} hold \textit{mutatis mutandis} after replacing $e(\mcA,\mcB)$ by $e_{\Gamma}(\mcA,\mcB)$, specifically because $e_{\Gamma}(\mcA,\mcB)\leq e(\mcA,\mcB)$ almost surely.
\end{proof}

This completes the proof of \Cref{thm:spectral_norm_bound}.
We have now proven \Cref{prop:discrepancy_property} which replaces \cite[Proposition 9]{Bounds2023}; recall our proof outline at the start of \Cref{sec:proof_of_spectral_norm_bound}.

\subsubsection{Proof of \texorpdfstring{\cref{cor:singvalscalingrankp}}{Corollary A.2}}\label{sec:proof_of_cor_singvalscalingrankp}

In order to prove \Cref{cor:singvalscalingrankp} we first prove \Cref{lem:singular_values_N} below characterizing the singular values of $N$.
The proof of \Cref{cor:singvalscalingrankp} follows afterwards.
\begin{lemma}
    \label{lem:singular_values_N}
    Presume \cref{asm:assumption1,asm:assumption2}.
    Then
    \begin{equation}
        \sigma_i(N)=
        \begin{cases}
            \Theta(\ell_n/n) & i\leq \rank(p) \\
            0                & i>\rank(p).
        \end{cases}
    \end{equation}
\end{lemma}
\begin{proof}
    It will prove convenient to introduce the matrix $C\in \mbbR^{n\times K}$ defined as $C_{x,i}\eqdef \mathbbm{1}\{\sigma(x)=i\}$.
    It encodes cluster membership and satisfies for all $v\in\mbbR^{K}$, $(Cv)_x = v_{\sigma(x)}$ and $(v C^\intercal)_x = v_{\sigma(x)}$.
    Moreover, define the vector $V\in \mbbR^{K}$ with elements $V_i\eqdef |\mcV_i|$.
    Using these, we write
    \begin{equation}
        \begin{split}
            N_{x,y}/\ell_n
            &
            =
            \Pi_xP_{x,y}
            =
            \frac{\pi_{\sigma(x)}p_{\sigma(x),\sigma(y)}}{|\mcV_{\sigma(x)}||\mcV_{\sigma(y)}|}
            \\
            &
            =
            (C \,\diag(V)^{-1}\diag(\pi) p \,\diag(V)^{-1} C^\intercal)_{x,y}
            \defeq
            (\tilde{C} \,\diag(\pi) p \,\tilde{C}^\intercal)_{x,y}.
        \end{split}
    \end{equation}
    Here, we have further defined the matrix $\tilde{C}\eqdef C\,\diag(V)^{-1}$ whose transpose is equal to $\diag(V)^{-1}C^{\intercal}$.

    Since $|\mcV_i|>0$ and $\pi_i>0$ for all $i\in [K]$, it follows that $\tilde{C}$ and $\text{diag}(\pi)$ are both full-rank.
    Consequently, $\rank(N)=\rank(p)$ and $\sigma_i(N)=0$ for $i>\rank(p)$.
    To prove that $\sigma_i(N)=\Theta(\ell_n/n)$ for $i\leq \rank(p)$ we further use the inequalities \cite{wang1997some}
    \begin{equation}
        \begin{split}
            \sigma_i(A)\sigma_1(B)\geq \sigma_i(AB)\geq \sigma_i(A)\sigma_{n}(B),
        \end{split}
    \end{equation}
    for $A\in \mbbC^{p\times n}$ and $B\in \mbbC^{n\times m}$, to find that for $i\leq \rank(p)$
    \begin{equation}
        \sigma_i(\tilde{C} \,\diag(\pi)p\,\tilde{C}^\intercal) \geq \sigma_i(\tilde{C} \,\diag(\pi)p)\sigma_{K}(\tilde{C})\geq \sigma_i(\diag(\pi)p)(\sigma_{K}(\tilde{C}))^2
    \end{equation}
    and
    \begin{equation}
        \sigma_i(\tilde{C}\, \diag(\pi)p \,\tilde{C}^\intercal) \leq \sigma_i(\tilde{C}\,\diag(\pi)p)\sigma_1(\tilde{C})\leq \sigma_i(\diag(\pi)p)(\sigma_1(\tilde{C}))^2.
    \end{equation}
    Finally, observe that the singular values of $\tilde{C}$ are readily computed to be $\sigma_i(\tilde{C})=\sqrt{1/|\mcV_i|}=\Theta(1/\sqrt{n})$ for $i\in [K]$, and that $\sigma_i(\diag(\pi)p)=\Theta(1)$ for all $i\leq \rank(p)$.
    Putting everything together, we find that for all $i\leq \rank(p)$, $\sigma_i(N) =\Theta(\ell_n/n)$.
\end{proof}

We next use \cref{thm:spectral_norm_bound,lem:singular_values_N} to prove \cref{cor:singvalscalingrankp}.
As in the proof of \cite[Corollary 4]{Bounds2023}, our first step is to use \cite[Lemma 17]{Bounds2023} (Weyl's inequality) to relate the singular values of $\hat{N}_{\Gamma}$ to those of $N$.
Applied to our setting, \cite[Lemma 17]{Bounds2023} implies that for all $i\in [n]$,
\begin{equation}
    \label{eqn:weylsineq}
    |\sigma_i(\hat{N}_\Gamma)-\sigma_i(N)|\leq \sigma_1(\hat{N}_\Gamma-N),
\end{equation}
almost surely. It then follows from \cref{thm:spectral_norm_bound} that for all $i\in [n]$
\begin{equation}
    \label{eqn:weylandnormbd}
    |\sigma_{i}(\hat{N}_\Gamma)-\sigma_{i}(N)|\overset{\eqref{eqn:weylsineq}}{\leq} \sigma_{1}(\hat{N}_\Gamma-N)=\OrderP{\sqrt{\ell_n/n}}.
\end{equation}
We now distinguish between $i>\rank(p)$ and $i\leq \rank(p)$.

To prove that $\sigma_i(\hat{N}_{\Gamma})=O_{\mbbP}(\sqrt{\ell_n/n})$ for $i>\rank(p)$, observe that from \cref{lem:singular_values_N}, $\sigma_i(N)=0$.
Consequently,
\begin{equation}
    \sigma_{i}(\hat{N}_\Gamma) = |\sigma_{i}(\hat{N}_\Gamma)- \sigma_{i}(N)|  \overset{\eqref{eqn:weylandnormbd}}{=} \OrderP{\sqrt{\ell_n/n}}.
\end{equation}
Meanwhile, for $i\leq \rank(p)$ it follows from \eqref{eqn:weylandnormbd} that $\sigma_{i}(\hat{N}_\Gamma)-\sigma_{i}(N)=O_{\mbbP}(\sqrt{\ell_n/n})$, and from \cref{lem:singular_values_N} that $\sigma_i(N)=\Theta(\ell_n/n)$.
We conclude using \cref{lem:asymptoticrelation}, which is proved next, that $\sigma_i(\hat{N}_\Gamma)=\Theta_{\mbbP}(\sigma_i(N))$ and therefore that $\sigma_i(\hat{N}_\Gamma)=\Theta_{\mbbP}(\ell_n/n)$.
\qed

\begin{lemma}
    \label{lem:asymptoticrelation}
    Let $\{X_n\}_{n=1}^{\infty}$ a sequence of random variables with the properties that $X_n\geq 0$, and that $X_n-y_n = \OrderP{x_n}$, where $\{x_n\}_{n=1}^{\infty}$ and $\{y_n\}_{n=1}^{\infty}$ denote two deterministic sequences with $x_n,y_n\in(0,\infty)$ and $x_n = o(y_n)$.
    Then $X_n=\Theta_{\mbbP}(y_n)$.
\end{lemma}
\begin{proof}
    The assumption that $X_n-y_n=\OrderP{x_n}$ ensures that for any $\varepsilon>0$, we can find constants $\delta_{\varepsilon}$ and $N_{\varepsilon}$ such that for all $n>N_{\varepsilon}$
    \begin{equation}
        \mbbP\bkt{|X_n-y_n|\leq \delta_{\varepsilon}x_n}\geq 1-\varepsilon.
    \end{equation}
    Then, given $\varepsilon>0$, fix $\kappa^+_{\varepsilon}>1>\kappa^-_{\varepsilon}>0$ and choose $N_{\varepsilon}$ large enough that for all $n>N_{\varepsilon}$, $\kappa_{\varepsilon}^-\leq 1-\delta_{\varepsilon}x_n/y_n$ and $\kappa_{\varepsilon}^+\geq 1+\delta_{\varepsilon}x_n/y_n$.
    This is possible because $x_n=o(y_n)$.
    Then, for all $n>N_{\varepsilon}$,
    \begin{equation}
        \begin{split}
            \mbbP\bkt{\{\kappa_{\varepsilon}^-y_n\leq X_n\} \cap \{X_n\leq \kappa_{\varepsilon}^+y_n\}}
            &
            \geq
            \mbbP\bkt{\{(1-\delta_{\varepsilon}x_n/y_n)y_n\leq X_n\}\cap\{X_n\leq (1+\delta_{\varepsilon}x_n/y_n)y_n\}}
            \\
            &
            =
            \mbbP\bkt{\{-\delta_{\varepsilon}x_n\leq X_n-y_n\} \cap \{X_n-y_n\leq \delta_{\varepsilon}x_n\}}
            \\
            &
            =
            \mbbP\bkt{|X_n-y_n|\leq \delta_{\varepsilon}x_n} \geq 1-\varepsilon.
        \end{split}
    \end{equation}
    That is it.
\end{proof}

\subsection{Analyzing \texorpdfstring{\Cref{alg:kpost,alg:kmeans}}{Algorithms 2 and 3}}
\label{sec:proofskpost}

To facilitate the analysis of \Cref{alg:kpost,alg:kmeans}, we introduce some additional notation.
Let
\begin{equation}\label{eqn:rank_r_approximation_def}
    \hat{R}
    \eqdef
    U_{1:n,1:r} \allowbreak \diag(\sigma_1,\ldots,\sigma_r) \allowbreak (V_{1:n,1:r})^\intercal
\end{equation}
denote the rank-$r$ approximation of $\hat{N}_{\Gamma}$, where $2r$ is the embedding dimension used in \Cref{alg:kpost}.
Define also $\hat{R}^0=[\hat{R},\hat{R}^\intercal]$, $N^0=[N,N^\intercal]$, and $\hat{N}^0=[\hat{N}_\Gamma,\hat{N}_\Gamma^\intercal]$.
Moreover, let $\bar{N}_k^0\eqdef (1/|\mcV_k|)\sum_{z\in\mcV_k}N_{z,\cdot}^0$ for $k=1,\ldots,K$.

\subsubsection{An equivalence of norms}
\label{sec:equivalence_of_norms}

We start by showing the following equivalence of norms:

\begin{lemma}
    \label{lem:equivalent_spectral_embedding}

    For $x,y \in [n]$,
    $
        \lVert \hat{R}^0_{x,\cdot}-\hat{R}^0_{y,\cdot} \rVert_2
        =
        \lVert \hat{X}_{x,\cdot}-\hat{X}_{y,\cdot} \rVert_2
    $
    almost surely.
\end{lemma}
\begin{proof}
    \Cref{lem:equivalent_spectral_embedding}'s proof is as follows.
    We first express $\hat{X}$ using matrix notation as
    \begin{equation}
        \hat{X} = [U_{1:n,1:r}\diag(\sigma_1,\ldots,\sigma_r), V_{1:n,1:r}\diag(\sigma_1,\ldots,\sigma_r)] \defeq [\tilde{U},\tilde{V}].
    \end{equation}
    We will next leverage that
    \begin{equation}
        \hat{R}
        =
        \tilde{U}(V_{1:n,1:r})^\intercal
        =
        U_{1:n,1:r}\tilde{V}^\intercal
        \label{eqn:Representation-of-Rhat-in-U-and-V}
    \end{equation}
    almost surely, and that
    \begin{equation}
        U^\intercal U
        =
        \mathbbm{1}_n
        \quad
        \textnormal{and}
        \quad
        V^\intercal V
        =
        \mathbbm{1}_{n}
        \label{eqn:Left-and-right-singular-vectors-form-two-orthonormal-sets}
    \end{equation}
    almost surely.
    Here, $\mathbbm{1}_n$ denotes the $n$-dimensional identity matrix.
    Observe that \eqref{eqn:Left-and-right-singular-vectors-form-two-orthonormal-sets} says that the left and right singular vectors of an \gls{SVD} form two orthonormal sets.

    Specifically: for $x, y \in [n]$,
    \begin{align}
        \lVert
            \hat{R}_{x,\cdot} - \hat{R}_{y,\cdot}
        \rVert^2_2
        &
        \eqcom{eqn:Representation-of-Rhat-in-U-and-V}
        =
        \lVert
            (\tilde{U}_{x,\cdot}-\tilde{U}_{y,\cdot})(V_{1:n,1:r})^\intercal
        \rVert_2^2
        =
        (\tilde{U}_{x,\cdot}-\tilde{U}_{y,\cdot})(V_{1:n,1:r})^\intercal V_{1:n,1:r}(\tilde{U}_{x,\cdot}-\tilde{U}_{y,\cdot})^\intercal
        \nonumber \\ &
        \eqcom{eqn:Left-and-right-singular-vectors-form-two-orthonormal-sets}
        =
        (\tilde{U}_{x,\cdot}-\tilde{U}_{y,\cdot})\mathbbm{1}_r(\tilde{U}_{x,\cdot}-\tilde{U}_{y,\cdot})^\intercal
        =
        \lVert
            \tilde{U}_{x,\cdot}-\tilde{U}_{y,\cdot}
        \rVert_2^2
    \end{align}
    and
    \begin{align}
        \lVert\hat{R}_{\cdot,x}-\hat{R}_{\cdot,y}\rVert^2_2
        &
        \eqcom{eqn:Representation-of-Rhat-in-U-and-V}
        =
        \lVert U_{1:n,1:r}((\tilde{V}^\intercal)_{\cdot,x}-(\tilde{V}^\intercal)_{\cdot,y}) \rVert^2_2
        =
        \lVert (\tilde{V}_{x,\cdot}-\tilde{V}_{y,\cdot})(U_{1:n,1:r})^\intercal \rVert^2_2
        \nonumber \\ &
        =
        (\tilde{V}_{x,\cdot}-\tilde{V}_{y,\cdot})(U_{1:n,1:r})^\intercal U_{1:n,1:r}(\tilde{V}_{x,\cdot}-\tilde{V}_{y,\cdot})^\intercal
        \nonumber \\ &
        \eqcom{eqn:Left-and-right-singular-vectors-form-two-orthonormal-sets}
        =
        (\tilde{V}_{x,\cdot}-\tilde{V}_{y,\cdot})\mathbbm{1}_r(\tilde{V}_{x,\cdot}-\tilde{V}_{y,\cdot})^\intercal
        =
        \lVert \tilde{V}_{x,\cdot}-\tilde{V}_{y,\cdot}\rVert^2_2
    \end{align}
    almost surely.
    Consequently
    \begin{align}
        \lVert\hat{R}^0_{x,\cdot}-\hat{R}^0_{y,\cdot}\rVert^2_2
        &
        =
        \lVert\hat{R}_{x,\cdot}-\hat{R}_{y,\cdot}\rVert^2_2+\lVert\hat{R}_{\cdot,x}-\hat{R}_{\cdot,y}\rVert^2_2 \\ &
        =
        \lVert \tilde{U}_{x,\cdot}-\tilde{U}_{y,\cdot}\rVert_2^2+\lVert \tilde{V}_{x,\cdot}-\tilde{V}_{y,\cdot}\rVert_2^2
        =
        \lVert \hat{X}_{x,\cdot}-\hat{X}_{y,\cdot}\rVert^2_2
    \end{align}
    almost surely, which completes the proof.
\end{proof}

Because the output of \Cref{alg:kpost} only depends on the neighborhoods $\mcN_x$ calculated therein, any embedding that leads to the same neighborhoods also leads to the same estimator $\Khat$.
Consequently, \cref{lem:equivalent_spectral_embedding} shows that we can replace $\hat{X}$ by $\hat{R}^0$ in the subsequent analysis, which simplifies the proofs.

\subsubsection{Concentration of the low rank approximation}
\label{sec:concentration_of_low_rank_approximation}

Our next step is to prove a lemma characterizing the concentration of this rank-$\Ktilde$ approximation.
The following is a slight generalization of \cite[Lemma 5]{ClusterBMC2017}.

\begin{lemma}
    \label{lem:rank_Ktilde_frobenius_bound}
    Let $\Ktilde$ be the output of \cref{alg:kpre} run with threshold $\gamma_n$.
    Then,
    \begin{equation}
        \lVert\hat{R}^0-N^0\rVert_{\mathrm{F}} \leq \sqrt{2(\Ktilde + K)}\bigl( \lVert \hat{N}_{\Gamma}-N\rVert + \min\{\gamma_n, \lVert \hat{N}_{\Gamma}-N\rVert + \sigma_{\Ktilde +1}(N)\} \bigr),
        \label{eqn:frobenius-norm-bound-1}
    \end{equation}
    where $\lVert\cdot\rVert$ denotes the spectral norm.
    Under the assumptions of \Cref{prop:limit_kpre} we have in particular that
    \begin{equation}
        \lVert \hat{R}-N\rVert_{\mathrm{F}}=O_{\mbbP}(\sqrt{\ell_n/n}).
        \label{eqn:frobenius-norm-bound-2}
    \end{equation}
\end{lemma}

\begin{proof}
    We first prove \eqref{eqn:frobenius-norm-bound-1}.
    From \cite[SM4.3 (69)]{ClusterBMC2017} and the observation that $\hat{R}-N$ is of rank at most $\Ktilde+\rank(p)$, it follows that
    \begin{equation}
        \label{eqn:frobenius_norm_bound}
        \lVert \hat{R}^0-N^0\rVert_{\mathrm{F}} \leq \sqrt{2(\Ktilde+\rank(p))}\big(\lVert\hat{R}-\hat{N}_{\Gamma}\rVert+\lVert\hat{N}_{\Gamma}-N\rVert \big)
    \end{equation}
    almost surely.
    Since $\hat{R}$ is a rank $\Ktilde$-approximation of $\hat{N}_{\Gamma}$, it follows that $\lVert\hat{R}-\hat{N}_{\Gamma}\rVert = \sigma_{\Ktilde+1}(\hat{N}_{\Gamma})$ almost surely.
    Now, by definition of $\Ktilde$, $\sigma_{\Ktilde+1}(\hat{N}_{\Gamma})< \gamma_n$ almost surely.
    Moreover, from \eqref{eqn:weylsineq} it follows that $\sigma_{\Ktilde+1}(\hat{N}_{\Gamma})\leq \sigma_{\Ktilde+1}(N)+\lVert\hat{N}_{\Gamma}-N\rVert$, which is finite almost surely.
    Together, these two bounds imply that
    \begin{equation}
        \label{eqn:low_rank_approx_bound}
        \lVert\hat{R}-\hat{N}_{\Gamma}\rVert \leq \min\{\gamma_n,\sigma_{\Ktilde+1}(N)+\lVert\hat{N}_{\Gamma}-N\rVert\}.
    \end{equation}
    Substituting \eqref{eqn:low_rank_approx_bound} into \eqref{eqn:frobenius_norm_bound} proves \eqref{eqn:frobenius-norm-bound-1}.

    To prove \eqref{eqn:frobenius-norm-bound-2}, note that by assumption \cref{prop:limit_kpre} applies such that $\Ktilde= \rank(p)\leq K$ with high probability.
    Consequently, with high probability also $\sigma_{\Ktilde+1}(N)=0$ so that \eqref{eqn:frobenius-norm-bound-1} implies that with high probability,
    \begin{equation}
        \lVert \hat{R}-N\rVert_{\mathrm{F}}\leq \sqrt{2K}\lb \lVert\hat{N}_{\Gamma}-N\rVert +\min\{\gamma_n,\lVert\hat{N}_{\Gamma}-N\rVert\}\rb.
    \end{equation}
    Moreover, recall from \cref{thm:spectral_norm_bound} that $\lVert \hat{N}_{\Gamma}-N\rVert=O_{\mbbP}(\sqrt{\ell_n/n})$.
    Together with the fact that $\min\{\gamma_n,\lVert \hat{N}_{\Gamma}-N\rVert\}\leq \lVert \hat{N}_{\Gamma}-N\rVert$ a.s., it follows that
    \begin{equation}
        \lVert \hat{R}-N\rVert_{\mathrm{F}}=O_{\mbbP}(\sqrt{\ell_n/n}).
        \label{eqn:frobenius-norm-bound}
    \end{equation}
\end{proof}

\subsubsection{Analyzing the geometry of the spectral embedding}
\label{sec:density-based_clustering_recovers_the_true_number_of_clusters}

\cref{lem:rank_Ktilde_frobenius_bound} will be used to control the concentration of the rows of $\hat{R}^0=[\hat{R},\hat{R}^\intercal]$ as points in $\mbbR^{2n}$.
We now show that if these rows are concentrated sufficiently tightly around their expectated value, the density-based clustering will recover the true number of clusters.
Moreover, the centers selected by \cref{alg:kpost} will lie close to the expected values $\bar{N}^0_k=(1/|\mcV_k|)\sum_{x\in \mcV_k}N_x^0$.
This is made precise by the following \cref{lem:cluster_centers_characterization}.

\begin{lemma}
    \label{lem:cluster_centers_characterization}
    Presume \cref{asm:assumption1,asm:assumption2}, $I(\alpha,p)>0$ and $\ell_n = \omega(n)$.
    Assume moreover that $\lVert \hat{R}-N\rVert_{\mathrm{F}}=O_{\mbbP}(f_n)$ for some sequence $f_n=o(\ell_n/n)$.
    If \cref{alg:kpost} is run with
    thresholds satisfying
    \begin{equation}
        \label{eqn:hnrhonconditions}
        \omega(f^2_n/n)=h^2_n=o(\ell_n^2/n^{3}),\quad \omega(f_n^2/h_n^2)=\rho_n=o(n),
    \end{equation}
    then the output $\Khat$ of \Cref{alg:kpost} satisfies $\Khat=K$ with high probability as $n\rightarrow\infty$.
    Moreover, the centers centers $z_1^*,\ldots,\allowbreak z_{\Khat}$ satisfy
    \begin{equation}
        \label{eqn:center_characterization}
        \lVert\hat{R}^0_{z_k^*,\cdot}-\bar{N}^0_{\gamma(k)}\rVert_2\leq bh_n \quad \textnormal{ for } k=1,\ldots,K
    \end{equation}
    for some permutation $\gamma$ of $1,\ldots,K$ and any $3/2<b$.
\end{lemma}
The proof of \Cref{lem:cluster_centers_characterization} relies on the following separability property of $N^0$, which follows from \cite[Lemma 4]{ClusterBMC2017}.
\begin{lemma}[Adapted from {\cite[Lemma 4]{ClusterBMC2017}}]
    \label{lem:separability_property}
    Presume \cref{asm:assumption1,asm:assumption2} and $I(\alpha,p)>0$.
    Then, for $x,y\in[n]$ such that $\sigma(x)\neq\sigma(y)$, $\lVert N^0_{x,\cdot}-N^0_{x,\cdot}\rVert_2=\Omega(\ell_n/n^{3/2})$.
\end{lemma}
\begin{proof}
    \Cref{lem:asymptoticrelation} follows from \cite[Lemma 4]{ClusterBMC2017} and the discussion in \cite[Section 6.2]{ClusterBMC2017}.
    In particular, the latter states that the quantity $D(\alpha,p)$ appearing in \cite[Lemma 4]{ClusterBMC2017} is a positive constant that is independent of $n$ whenever $I(\alpha,p)>0$.
    One checks carefully that their argument remains valid, even under our weaker \Cref{asm:assumption2}.
    One also checks that the proof of \cite[Lemma 4]{ClusterBMC2017} remains valid under \Cref{asm:assumption1,asm:assumption2}.
    Given these facts, the conclusion of \Cref{lem:equivalent_spectral_embedding} then follows from \cite[Lemma 4]{ClusterBMC2017}, provided that $I(\alpha,p)>0$.
\end{proof}
We now prove \Cref{lem:cluster_centers_characterization} using \Cref{lem:separability_property}.
\begin{proof}[Proof of \Cref{lem:cluster_centers_characterization}]
Recall from \Cref{sec:equivalence_of_norms} that the output of \Cref{alg:kpost} is the same if we replace the spectral embedding $\hat{X}$ by the matrix $\hat{R}^0$, which we will subsequently analyze.

Then, fix $0<a<1/2$ and $1+a<b<\infty$.
To prove \cref{lem:cluster_centers_characterization}, we first define the set of \textit{cores}
\begin{equation}
    C_k\eqdef \big\{x\in\mcV_k \,\big| \lVert \hat{R}^0_{x,\cdot}-\bar{N}_k^0\rVert_2\leq ah_n\big\} \textnormal{ for } k=1,\ldots,K,
\end{equation}
i.e., states from cluster $k$ for which $\hat{R}_{x,\cdot}^0$ is close to their expected value $\bar{N}_k^0$.
As we shall see, these states will make up the majority of all states and they will be well-behaved with respect to the density-based clustering algorithm we deploy on the rows of $\hat{R}^0$.
Moreover, we define the set of \textit{outliers}
\begin{equation}
    \mathcal{O}\eqdef\big\{x\in[n]\,\big|\lVert \hat{R}^0_{x,\cdot}-\bar{N}_k^0\rVert_2\geq bh_n\textnormal{ for all } k=1,\ldots,K,\big\}.
\end{equation}
These states are far from \textit{any} cluster's expected center.
Finally, recall also the definition of the neighborhood of a state $x\in [n]$ from \eqref{eqn:nbrhooddef}:
\begin{equation}
    \mcN_x= \big\{y\in [n]\,\big|\lVert\hat{R}^0_{x,\cdot}-\hat{R}^0_{y,\cdot}\rVert_2\leq h_n\big\}.
\end{equation}
Notice the following four properties of these sets.
\begin{enumerate}
    \item $\left|\left(\bigcup_{k=1}^K C_k\right)\cap \mathcal{N}_x\right|=0$ for all $x\in\mathcal{O}$.
            \begin{proof}
                Let $x\in\mathcal{O}$.
                Then by definition $\|\hat{R}_{x,\cdot}^0-\Bar{N}_k^0\|_2\geq bh_n$ for all $k\in[K]$.
                Moreover, it holds that $\|\hat{R}_{x,\cdot}^0-\hat{R}_{y,\cdot}^0\|_2\leq h_n$ for all $y\in\mathcal{N}_x$. It then follows that for $l\in[K]$,
                \begin{equation}
                    \lVert \hat{R}^0_{y,\cdot}-\bar{N}_l^0 \rVert_2\geq \lVert \hat{R}^0_{x,\cdot}-\bar{N}_l^0\rVert_2-\lVert\hat{R}^0_{y,\cdot}-\hat{R}^0_{x,\cdot}\rVert_2 \geq (b-1)h_n>ah_n
                \end{equation}
                and thus $y\notin C_l$.
            \end{proof}
    \item $\left|\left(\bigcup_{k=1}^KC_k\right)^{\mathrm{c}}\right|=o(f_n^2/h_n^2)\overset{\eqref{eqn:hnrhonconditions}}{=} o_{\mbbP}(n)$, and $\left|C_k^{\mathrm{c}}\cap\mcV_k\right|=o(f_n^2/h_n^2)\overset{\eqref{eqn:hnrhonconditions}}{=}\orderP{n}$ for $k\in[K]$.
            \begin{proof}
                We have that
                \begin{equation}
                    \label{eqn:frobenius_lower_bound_1}
                    \lVert\hat{R}^0-N^0\rVert_{\mathrm{F}}^2 = \sum_{x\in[n]}\lVert\hat{R}^0_{x,\cdot}-\bar{N}^0_{\sigma(x)}\rVert_2^2\geq \big|\left(\cup_{k=1}^K C_k\right)^{\mathrm c}\big|\min_{x\in\left(\cup_{k=1}^K C_k\right)^{\mathrm c}}\lVert\hat{R}^0_{x,\cdot}-\bar{N}^0_{\sigma(x)}\rVert_2^2
                \end{equation}
                almost surely.
                Then, observe that $x\notin C_k$ for all $k\in [K]$ implies that in particular, $\lVert\hat{R}^0_{x,\cdot}-\bar{N}^0_{\sigma(x)}\rVert_2^2>a^2h_n^2$ such that $\lVert\hat{R}^0-N^0\rVert_{\mathrm F}^2>a^2h_n^2 \left|\left(\cup_{k=1}^K C_k\right)^{\mathrm{c}}\right|$ almost surely.
                Rearranging and using that $\lVert\hat{R}^0-N^0\rVert_{\mathrm{F}}^2=O_{\mbbP}(f^2_n)$ gives the first statement.

                The second statement follows by almost the same argument: one replaces in \eqref{eqn:frobenius_lower_bound_1} the set $(\cup_{k=1}^K C_k)^{\mathrm{c}}$ by $C_k^{\mathrm{c}}\cap \mcV_k$ and uses that $\lVert\hat{R}^0_{x,\cdot}-\bar{N}_k^0\rVert^2_2>a^2h_n^2$ still holds for all $x\in C_k^{\mathrm{c}}\cap \mathcal{V}_k$.
            \end{proof}
    \item $C_{\sigma(x)}\subseteq \mcN_{x}$ for all $x\in\bigcup_{k=1}^K C_k$.
            \begin{proof}
                Let $x\in\bigcup_{k=1}^K C_k$.
                Since the $C_k$ are disjoint by construction, there exists a unique $l\in\{1,\ldots,K\}$ such that $x\in C_l$.
                Let $y\in C_l$.
                Then
                \begin{equation}
                    \lVert\hat{R}_{x,\cdot}^0-\hat{R}_{y,\cdot}^0\rVert_2 \leq \lVert\hat{R}_{x,\cdot}^0-\bar{N}_{l}^0\rVert_2+\lVert\hat{R}_{y,\cdot}^0-\bar{N}_{l}^0\rVert_2 \leq 2ah_n.
                \end{equation}
                Since we chose $a<1/2$, we have that $y\in\mathcal{N}_x$ because $\|\hat{R}_{x,\cdot}^0-\hat{R}_{y,\cdot}^0\|_2<h_n$.
            \end{proof}
    \item If $|\mathcal{N}_x\cap C_k|\geq 1$ for some $x\in[n]$ and $k\in [K]$, then $|\mathcal{N}_x\cap C_l|=0$ for any $l\neq k$.
            \begin{proof}
                Suppose that $|\mcN_x\cap C_k|\geq 1$ for some $x\in[n]$ and $k\in [K]$.
                Let $y\in \mcN_x$ and $l\neq k$.
                Then by propety (1)
                \begin{equation}
                    \lVert\hat{R}_{y,\cdot}^0-\bar{N}_k^0 \rVert_2\leq \lVert\hat{R}_{y,\cdot}^0 -\hat{R}_{x,\cdot}^0\rVert_2 + \lVert\hat{R}_{x,\cdot}^0-\bar{N}_k^0 \rVert_2\leq (a+b)h_n
                \end{equation}
                almost surely.
                Since $h_n=o(\ell_n/n^{3/2})$ by assumption, it follows that $\lVert \hat{R}_{y,\cdot}^0-\bar{N}_k^0 \rVert_2=o(\ell_n/n^{3/2})$.
                Together with \Cref{lem:separability_property}, this implies that
                \begin{equation}
                    \lVert\hat{R}_{y,\cdot}^0-\bar{N}_l^0\rVert_2 \geq \left|\lVert\bar{N}_l^0-\bar{N}_k^0\rVert_2 -\lVert\hat{R}_{y,\cdot}^0 -\bar{N}_k^0\rVert_2  \right|=\Omega\bigg(\frac{\ell_n}{n^{3/2}}-(a+b)h_n\bigg) = \Omega\bigg(\frac{\ell_n}{n^{3/2}}\bigg).
                \end{equation}
                Since $h_n=o(\ell_n/n^{3/2})$, $y$ does not belong to $C_l$.
            \end{proof}
\end{enumerate}
In the following, let $z_k^*$ be the $k$th center selected by \cref{alg:kpost}, and let $\hat{\mcV}_k$ be the $k$th cluster identified by \cref{alg:kpost}, i.e., $\hat{\mcV}_k=\mcN_{z_k^*}\setminus \cup_{l=1}^{k-1}\hat{\mcV}_l$.
Note that $\cup_{l=1}^k\hat{\mcV}_{l}=\cup_{l=1}^k\mcN_{z_l^*}$.
We will use properties (1)--(4) to show that for each $k=1,\ldots,K$,
\begin{equation}
    \label{eq:coreconstruction}
    \exists \,x\in \big(\cup_{k=1}^KC_k\big)\setminus\big(\cup_{l=1}^{k-1}\hat{\mcV}_l\big) :\quad \big|\mathcal{N}_x\setminus\cup_{l=1}^{k-1}\hat{\mcV}_l\big|\geq |C^{(k)}|=n\alpha^{(k)}(1-\orderP{1}),
\end{equation}
where $C^{(k)}$ is the $k$th largest core and $\alpha^{(k)}$ is the size of the $k$th largest cluster.
Using \eqref{eq:coreconstruction} we can prove \cref{lem:cluster_centers_characterization} as follows.

Firstly, observe that \eqref{eq:coreconstruction} implies that none of the $z_k^*$ are outliers.
Indeed, each of the centers will satisfy $(i)$ $|\mcN_{z_k^*}\cap \lb \cup_{l=1}^KC_l\rb|\geq 1$, since otherwise
\begin{equation}
    \label{eq:nonoutlier}
    \big|\mcN_{z_k^*}\setminus \cup_{l=1}^{k-1}\hat{\mcV}_l\big|\leq |\mcN_{z_k^*}|\overset{\neg (i)}{\leq }\big|\big(\cup_{k=1}^KC_k\big)^{\mathrm c}\big|\overset{(2)}{=}\orderP{n},
\end{equation}
and the algorithm would have selected a different center as guaranteed by \eqref{eq:coreconstruction}.
In particular, this means that $z_k^*$ is not an outlier by property (1), such that $z_k^*$ is at least $bh_n$ close to some core center $\bar{N}^0_{c_k}$ with $c_k\in1,\ldots,K$.
Moreover, it follows from property (4) that $(ii)$ $|\mcN_{z_k^*}\cap C_l|=0$ for $l\neq c_k$ such that the core $C_{c_k}$ is unique.

To prove \eqref{eqn:center_characterization} we also require that the $c_k$ are distinct.
In fact, we claim that $C_{c_k}=C^{(k)}$ with high probability.
This follows from the bound
\begin{equation}
    \label{eqn:neighborhood_size_bound}
    \begin{split}
        |\mcN_{z_k^*}|= |\mcN_{z_k^*}\cap (\cup_{k=1}^KC_k)|+|\mcN_{z_k^*}\cap (\cup_{k=1}^KC_k)^{\mathrm c}| & \overset{(ii)}{=}  |\mcN_{z_k^*}\cap C_{c_k}|+|\mcN_{z_k^*}\cap (\cup_{k=1}^KC_k)^{\mathrm c}|   \\
                                                                                                                & \leq |C_{c_k}|+|(\cup_{k=1}^KC_k)^{\mathrm c}| \overset{(iii)}{=} n\alpha_{c_k}(1+o_{\mbbP}(1)).
    \end{split}
\end{equation}
Here, $(iii)$ is a consequence of the fact that by construction, there are at most $|C_k|\leq |\mcV_k|=n\alpha_k(1+o(1))$ states in each core, and property (2).
Together, \eqref{eq:coreconstruction} and \eqref{eqn:neighborhood_size_bound} imply that the algorithm would have selected a different center with high probability unless $\alpha_{c_k}=\alpha^{(k)}$.
Since \eqref{eq:coreconstruction} further implies that with high probability the $k$th largest core corresponds to the $k$th largest cluster, it follows that $C_{c_k}=C^{(k)}$ with high probability.

Finally, \eqref{eq:coreconstruction} shows that \cref{alg:kpost} will select \textit{at least} $K$ centers, since for $k\leq K$, \eqref{eq:coreconstruction} ensures that there exists a center $z_k^*$ such that $|\mcN_{z_k^*}\setminus \cup_{l=1}^{k-1}\hat{\mcV}_l|=\Omega_{\mbbP}(n)=\omega_{\mbbP}(\rho_n)$.
Together, these facts prove \eqref{eqn:center_characterization}.
Moreover, since the $\hat{\mcV}_k$ are all disjoint by construction and so too are the $C_k$, it follows that after selecting $K$ centers there remain at most
\begin{equation}
    \left|\lb \cup_{l=1}^{K}\hat{\mcV}_l\rb^{\mathrm c}\right| \overset{(iv)}{=} n - \sum_{k=1}^K|\hat{\mcV}_k| \overset{\eqref{eq:coreconstruction}}{\leq } n-\sum_{k=1}^K|C_k| = |(\cup_{k=1}^K C_k)^c| \overset{(2)}{=} o_{\mbbP}(f_n^2/h_n^2)\overset{\eqref{eqn:hnrhonconditions}}{=}\orderP{\rho_n}.
\end{equation}
states.
This implies that $\Khat\leq K$ with high probability according to \cref{alg:kpost}'s prescription, which completes the proof. \qed

\paragraph{Proof of \eqref{eq:coreconstruction}}
We first prove a lower bound on the size of the cores.
By property (2), we have that for $k=1,\ldots,K$,
\begin{equation}
    |C_k|=|\mcV_k|-|C_k^c\cap\mcV_k|\overset{(2)}{=}n\alpha_k(1-\orderP{1}).
\end{equation}
It follows that with high probability
\begin{equation}
    \label{eq:coresizebd}
    |C^{(k)}|= n\alpha^{(k)}(1-\orderP{1}).
\end{equation}
In particular, \eqref{eq:coresizebd} implies that the cores $C_k$ are non empty with high probability.

We will now prove \eqref{eq:coreconstruction} by induction.
For $k=1$, property (3) implies that $|\mcN_x|\geq |C^{(1)}|$ for any $x\in C^{(1)}$.
\eqref{eq:coresizebd} then ensures that \eqref{eq:coreconstruction} holds for $k=1$.

Assume then that \eqref{eq:coreconstruction} holds for each $l\in [k-1]$ for some $k\leq K$.
It follows from \eqref{eq:nonoutlier} that $\big|\mcN_{z_l^*}\cap \allowbreak \lb \cup_{m=1}^KC_m\rb\big| \geq 1$ for each $l\in [k-1]$.
Given that the neighborhoods $\mcN_{z_l^*}$ of the centers $z_l^*$ non trivially intersect the cores, property (4) then ensures that each of these neighborhoods $\mcN_{z_l^*}$ intersect exactly one such core.
Let $m_l$ denote the index of this core for each $l\in [k-1]$.

We now distinguish two cases.
Suppose that the $C_{m_l}$ are such that $C^{(k)}\neq C_{m_l}$ for any $l\in [k-1]$.
Then
\begin{equation}
    \label{eqn:empty_core_intersection}
    |C^{(k)}\cap \cup_{l=1}^{k-1}\mcN_{z_l^*}|=0,
\end{equation}
since each $\mcN_{z_l^*}$ only intersects $C_{m_l}$.
Moreover, by property (3), $C^{(k)}\subseteq \mcN_x$ for any $x\in C^{(k)}$.
Together with \eqref{eqn:empty_core_intersection}, this implies that
\begin{equation}
    \label{eqn:neighborhood_lower_bound}
    |\mcN_x\setminus \cup_{l=1}^{k-1}\mcN_{z_l^*}|\geq |C^{(k)}|.
\end{equation}
Since $C^{(k)}$ is non empty, and \eqref{eqn:empty_core_intersection} implies that $x\in C^{(k)} \Leftrightarrow x\notin \cup_{l=1}^{k-1}\mcN_{z_l^*}$, \eqref{eqn:neighborhood_lower_bound} ensures that \eqref{eq:coreconstruction} holds at $k$.

If instead, the $k$th largest core $C^{(k)}$ is among the $C_{m_l}$ for $l\in [k-1]$, then by the pigeonhole principle there must be some larger core $C^{(k')}$ with $k'\leq k$ that is not.
Replacing $C^{(k)}$ by $C^{(k')}$ in the preceding argument and using that $|C^{(k')}|\geq |C^{(k)}|$ then proves that \eqref{eq:coreconstruction} still holds at $k$.
That is it.% \qed
\end{proof}

Finally, we combine \cref{lem:rank_Ktilde_frobenius_bound} and \cref{lem:cluster_centers_characterization} to prove \cref{thm:Khat=K,thm:kmeanserror}.

\subsubsection{Proof of \texorpdfstring{\cref{thm:Khat=K}}{Theorem 3.3}}
Given \Cref{lem:rank_Ktilde_frobenius_bound,lem:cluster_centers_characterization} the proof of \Cref{thm:Khat=K} is almost immediate.
Indeed, the assumptions of \Cref{thm:Khat=K} ensure that \Cref{prop:limit_kpre} and therefore \Cref{lem:rank_Ktilde_frobenius_bound} apply.
Consequenly, $\lVert \hat{R}-N\rVert_{\mathrm{F}}=O_{\mbbP}(f_n)$ with $f_n=\sqrt{\ell_n/n}$.
\Cref{thm:Khat=K} then follows if we can verify that \eqref{eqn:hnrhonconditions} holds with $f_n=\sqrt{\ell_n/n}$, which indeed is immediate from \eqref{eqn:parameter_conditions}. \qed

\subsubsection{Proof of \texorpdfstring{\cref{thm:kmeanserror}}{Theorem 3.4}}

The assumptions of \cref{thm:kmeanserror} include those of \cref{thm:Khat=K}, such that the proof of \cref{thm:Khat=K} therefore shows that \cref{lem:cluster_centers_characterization} applies.
Upon replacing \cite[(72)]{ClusterBMC2017} by \eqref{eqn:center_characterization}, the proof of \cite[Lemma 6]{ClusterBMC2017} shows that for misclassified states $x\in\mcE$, the corresponding rows $\hat{R}_{x,\cdot}^0$ are far from their true value $N^0_{x,\cdot}$:
\begin{equation}
    \label{eqn:R0x-N0x}
    \|\hat{R}_{x,\cdot}^0-N_{x,\cdot}^0\|_2=\Omega_{\mbbP}\bigg(\frac{\ell_n}{n^{3/2}}\bigg) \quad \textnormal{for all }x\in\mathcal{E}.
\end{equation}
Given \eqref{eqn:R0x-N0x}, the rest of the proof that $|\mathcal{E}|/n=\OrderP{n/\ell_n}$ proceeds as follows.

We start by writing
\begin{equation}
    \label{eqn:frobeniuserrorrel}
    \lVert \hat{R}^0-N^0\rVert^2_{\mathrm{F}}=\sum_{x\in[n]} \lVert\hat{R}_{x,\cdot}^0-N_{x,\cdot}^0\rVert_2^2 \geq \sum_{x\in\mathcal{E}}\lVert\hat{R}_{x,\cdot}^0-N_{x,\cdot}^0\rVert_2^2 = |\mcE|\min_{x\in \mcE}\lVert \hat{R}^0_{x,\cdot}-N^0_{x,\cdot}\rVert_2^2.
\end{equation}
Together with \eqref{eqn:R0x-N0x} this then implies that
\begin{equation}
    \lVert \hat{R}^0-N^0\rVert^2_{\mathrm{F}}/(|\mcE|/n)\geq n\min_{x\in \mcE}\lVert \hat{R}^0_{x,\cdot}-N^0_{x,\cdot}\rVert_2^2=\Omega_{\mbbP}\lb \ell_n^2/n^2\rb.
\end{equation}
Moreover, the proof of \cref{thm:Khat=K} shows that under the assumptions of \cref{thm:kmeanserror} $\lVert \hat{R}^0-N^0\rVert^2_{\mathrm{F}}=O_{\mbbP}(\ell_n/n)$.
The result of \cref{thm:kmeanserror} then follows by applying \cite[Lemma 22]{ClusterBMC2017} with $Y_n=|\mathcal{E}|/n$, $y_n=\ell_n^2/n^2$, and $x_n=\ell_n/n$.\qed

\section{Further experimental details}

\subsection{Key performance indices}
\label{sec:Numerical-key-performance-indices}

\subsubsection{Relative accuracy}

To evaluate the performance of an algorithm that estimates the number of clusters $K$, we typically report the \emph{relative accuracy}
\begin{equation}
    \Delta \hat{K}^{\mathrm{rel}}
    \eqdef
    \frac{\Khat-K}{K}
    .
\end{equation}
Here, $\hat{K}$ denotes the estimate of $K$ obtained from $X_1, \ldots, X_\ell$ using some algorithm.

Note that the relative accuracy can actually be calculated by us in \Cref{sec:numerical-experiments}, because the true value of $K$ is known in every experiment.

\subsubsection{Entropy of a partition}
\label{sec:Entropy-of-a-partition}
To characterize cluster size variability, we rely on the notion of the \emph{entropy of a partition}.
This is defined for any partition $\{ \mathcal{A}_k \}_{k=1}^K$ of $[n]$ by
\begin{equation}
    \text{H}( \{ \mathcal{A}_k \}_{k=1}^K )
    \eqdef
    -\sum_{k=1}^K
    \frac{ | \mathcal{A}_k | }{n} \ln{ \frac{ | \mathcal{A}_k | }{n} }
    ,
    \label{eqn:clustering_entropy_def}
\end{equation}
where it is understood that $0 \ln{0} \equiv 0$.

For the same reason we also rely on the notion of the \emph{normalized entropy of a partition}.
It is defined for any partition $\{ \mathcal{A}_k \}_{k=1}^K$ of $[n]$ by
\begin{equation}
    \bar{\text{H}}(\{ \mathcal{A}_k \}_{k=1}^K)
    \eqdef
    \frac{1}{\ln{K}}
    \text{H}(\{ \mathcal{A}_k \}_{k=1}^K)
    .
    \label{eqn:normalized_clustering_entropy_def}
\end{equation}
The normalized entropy of a partition has the useful properties that
(i) $\bar{\text{H}}(\{ \mathcal{A}_k \}_{k=1}^K)=0$ whenever $| \mathcal{A}_k | = n$ for some $k\in [K]$,
and
(ii) $\bar{\text{H}}(\{ \mathcal{A}_k \}_{k=1}^K)=1$ whenever $| \mathcal{A}_k | = n/K$ for all $k\in [K]$.

\subsubsection{Adjusted Mutual Information (AMI)}
\label{sec:Adjusted-Mutual-Information}

If an algorithm also estimates the clusters' constituents, then we sometimes also report the \emph{\gls{AMI} index}.
The \gls{AMI} index can indicate how close an approximate clustering is to the ground truth labeling $\sigma$, when available.
For further background, we refer to \cite{vinh2009information}.

The \gls{AMI} is derived from the \emph{\gls{MI} between two partitions}.
For any two partitions $\{ \mathcal{A}_k \}_{k=1}^{K_{\mathrm{a}}}, \allowbreak \{ \mathcal{B}_l \}_{l=1}^{K_{\mathrm{b}}}$ of $[n]$,
\begin{equation}
    \text{MI}
    \bigl(
        \{ \mathcal{A}_{ k } \}_{ k = 1 }^{ K_{\mathrm{a}} }
        ,
        \{ \mathcal{B}_{ l } \}_{ l = 1 }^{ K_{\mathrm{b}} }
    \bigr)
    \eqdef
    \sum_{ k = 1 }^{ K_{\mathrm{a}} }
    \sum_{ l = 1 }^{ K_{\mathrm{b}} }
    \frac{| \mathcal{A}_{ k } \cap \mathcal{B}_{ l }| }{n}
    \ln{
        \frac{n| \mathcal{A}_{ k } \cap \mathcal{B}_{ l }| }{| \mathcal{A}_{ k } || \mathcal{B}_{ l }| }
        }
\end{equation}
denotes their \gls{MI}.
Recall also the definition \eqref{eqn:clustering_entropy_def} of the entropy of a partition.
Then, for any two partitions $\{ \mathcal{A}_k \}_{k=1}^{K_{\mathrm{a}}}, \allowbreak \{ \mathcal{B}_l \}_{l=1}^{K_{\mathrm{b}}}$ of $[n]$,
\begin{equation}
    \textnormal{AMI}
    \bigl(
        \{ \mathcal{A}_k \}_{k=1}^{K_{\mathrm{a}}}
        ,
        \{ \mathcal{B}_l \}_{l=1}^{K_{\mathrm{b}}}
    \bigr)
    \eqdef
    \frac{
        \text{MI}
        \bigl(
            \{ \mathcal{A}_{k} \}_{k=1}^{K_{\mathrm{a}}}
            ,
            \{ \mathcal{B}_{l} \}_{l=1}^{K_{\mathrm{b}}}
        \bigr)
        -
        \mbbE
        \bigl[
            \text{MI}
            (
                \{ \mathcal{U}_{k}^{(a)} \}_{k=1}^{K_{\mathrm{a}}}
                ,
                \{ \mathcal{U}_{l}^{(b)} \}_{l=1}^{K_{\mathrm{b}}}
            )
        \bigr]
    }
    {
        \sqrt{
            \text{H}
            \bigl(
                \{ \mathcal{A}_{k} \}_{k=1}^{K_{\mathrm{a}}}
            \bigr)
            \text{H}
            \bigl(
                \{ \mathcal{B}_{l} \}_{l=1}^{K_{\mathrm{b}}}
            \bigr)
        }
        -
        \mbbE
        \bigl[
            \text{MI}
            (
                \{ \mathcal{U}_{k}^{(a)} \}_{k=1}^{K_{\mathrm{a}}}
                ,
                \{ \mathcal{U}_{l}^{(b)} \}_{l=1}^{K_{\mathrm{b}}}
            )
        \bigr]
    }
    \label{eqn:AMI_definition}
\end{equation}
denotes the \gls{AMI} of the two partitions.
Here, $\{ \mathcal{U}^{(a)}_{k} \}_{k = 1}^{K_a}, \allowbreak \{ \mathcal{U}^{(b)}_{l} \}_{l = 1}^{K_b}$ denote two independent random partitions of $[n]$ of sizes $K_a, K_b$ satisfying $|\mathcal{U}_k^{(a)}|=|\mcA_k|$ for $k\in[K_a]$ and $|\mathcal{U}_k^{(b)}|=|\mcB_k|$ for $k\in[K_b]$, respectively.
Each is chosen uniformly at random from the set of all possible such partitions.

With a slight abuse of notation, we will usually report $\mathrm{AMI}( \sigma, \hat{\sigma} )$ in our experiments.

\subsection{Characteristics of the \texorpdfstring{\gls{BMC}}{BMC} ensembles numerically studied}
\label{sec:Generating_BMCs_with_random_parameters}

For some of our experiments, we generate random \glspl{BMC}.
This section contains details about the ensembles used and their implementation.
We also characterize these ensembles by the average values of the information quantity $I(\alpha, p)$, cluster entropies $\bar{\mathrm{H}}(\sigma)$, and their mixing times $t_{\mix}$.
When calculating $\bar{\mathrm{H}}(\sigma)$ for \glspl{BMC} with randomly sampled cluster fractions $\alpha$, we reject samples such that $\min_k n\alpha_k<1$ and sample \glspl{BMC} with $n=1000$.
All averages are calculated using $1000$ independent samples from each ensembles.

\subsubsection{Uniformly sampled \texorpdfstring{\glspl{BMC}}{BMCs}}
\label{sec:uniform_BMCs}

Given $K \in \mathbb{N}_+$, we generate random \gls{BMC} parameters $(\alpha, p)$
that are distributed uniformly at random within the $n$-dimensional simplex
\begin{equation}
    \Delta_K
    :=
    \Bigl\{
    ( x_1, \ldots, x_K )
    \Big|
    \sum_{i=1}^K x_i = 1
    \Bigr\}
    \subset
    \mathbb{R}^K
    .
\end{equation}
To do so, we exploit the fact that if for $i = 1, \ldots, n$, $X_i \sim \mathrm{Exp}(1)$, then the random vector
\begin{equation}\label{eqn:dirichlet_sampling}
    \frac{ ( X_1, X_2, \ldots, X_K ) }{ \sum_{i=1}^K X_i }
    \sim
    \mathrm{Unif}( \Delta_K )
    .
\end{equation}

\begin{table}[ht]
	\centering
	\setlength{\tabcolsep}{4pt}
	{
		\small
		\begin{tabular}{l l c c c c}
			\toprule
			{Ensemble $p$} & {Ensemble $\alpha$} & $K$ & $I(\alpha, p)$ & $\bar{\mathrm{H}}(\sigma)$ & $t_{\mix}$  \\
			\midrule
			Uniform  & Uniform & $5$ & $0.63 \pm 0.02$ & $0.798 \pm 0.007$ & $2.10\pm 0.02$ \\
            Uniform  & Uniform & $10$ & $0.51 \pm 0.01$ & $0.840 \pm 0.004$ & $2.01 \pm 0.01$ \\
            \midrule
			Uniform & Const. $1/K$ & $5$ & $0.53 \pm 0.01$ & $1\pm 0$ & $2.00 \pm 0.02$ \\
            Uniform & Const. $1/K$ & $10$ & $0.60 \pm 0.01$ & $1 \pm 0$ & $2.00 \pm 0.01$ \\
            \bottomrule
			\vspace{0pt}
		\end{tabular}
	}
	\caption{%
    Parameters characterizing \gls{BMC} ensembles with uniformly sampled cluster transition matrices $p$.
    }
    \label{tab:ensemble_characterizations_uniform}
\end{table}

\subsubsection{Low rank \texorpdfstring{\glspl{BMC}}{BMCs}}
\label{sec:low_rank_BMCs}

Fix $d<K$.
We randomly sample $K$ vectors $v_i\in \mbbR_{\geq 0}^d$ for $i\in [K]$, and define the cluster transition matrix as in \cref{ex:dot_product_model}.
In our examples, we sample $v_i\in\mbbR^d$ for $i\in [K]$ according to $v_i\sim \textnormal{Dir}((1/d,\ldots,1/d))$ where Dir denotes the Dirichlet distribution.
The cluster transition matrix is then defined as in \cref{ex:dot_product_model}.

\begin{table}[ht]
	\centering
	\setlength{\tabcolsep}{4pt}
	{
		\small
		\begin{tabular}{l l c c c c}
			\toprule
			{Ensemble $p$} & {Ensemble $\alpha$} & $K$ & $I(\alpha, p)$ & $\bar{\mathrm{H}}(\sigma)$ & $t_{\mix}$  \\
			\midrule
			Low rank & Uniform & $5$ & $0.29\pm 0.02$ & $0.805 \pm 0.007$ & $1.34 \pm 0.03$  \\
            Low rank & Uniform & $10$ & $0.189\pm 0.009$ & $0.840 \pm 0.004$ & $1.27 \pm 0.02$  \\
			\bottomrule
			\vspace{0pt}
		\end{tabular}
	}
	\caption{%
    Parameters characterizing different \gls{BMC} ensembles with rank deficient cluster transition matrices $p$.
    }
    \label{tab:ensemble_characterizations_low_rank}
\end{table}

\subsubsection{Reversible \texorpdfstring{\glspl{BMC}}{BMCs}}
\label{sec:reversible_BMCs}

Given the number of clusters $K$, we first sample a symmetric random matrix $W\in \mbbR^{K\times K}$.
To maintain consistency with our previous ensemble, we sample $W$ by sampling each row uniformly at random from the simplex, and symmetrizing the resulting matrix via the transformation $W\rightarrow (W+W^\intercal)/2$.
For simplicity, we further normalize $W\rightarrow W/\sum_{i,j} W_{i,j}$ without loss of generality.
We then set $p_{i,j}= W_{i,j}/(\sum_{j \in [K]}W_{i,j})$.
One checks that $\pi_j=\sum_{j \in [K]}W_{i,j}$.

Ensembles of random matrices $p$ generated as above have previously been studied for large $\text{dim}(p)$ \cite{bordenave2008spectrum} when the entries of $W$ are independent and identically distributed.
Our primary interest here, however, is the case that $\text{dim}(p)=K$ is small.
Note also that $W$ can be viewed as the weight matrix of an undirected graph, in which case $p$ is the transition matrix of a random walk on this graph.

\begin{table}[ht]
	\centering
	\setlength{\tabcolsep}{4pt}
	{
		\small
		\begin{tabular}{l l c c c c}
			\toprule
			{Ensemble $p$} & {Ensemble $\alpha$} & $K$ & $I(\alpha, p)$ & $\bar{\mathrm{H}}(\sigma)$ & $t_{\mix}$  \\
			\midrule
            Reversible & Const. $\pi$ & $5$ & $0.26 \pm 0.01$ & $0.990 \pm 0.001$ & $1.87 \pm 0.03$ \\
            Reversible & Const. $\pi$ & $10$ & $0.259 \pm 0.006$ & $0.9956 \pm 0.0001$ & $1.99 \pm 0.01$ \\
			\bottomrule
			\vspace{0pt}
		\end{tabular}
	}
	\caption{%
    Parameters characterizing different \gls{BMC} ensembles with reversible cluster transition matrices $p$.
    }
    \label{tab:ensemble_characterizations_reversible}
\end{table}

\subsubsection{Assortative \texorpdfstring{\glspl{BMC}}{BMCs}}
\label{sec:assortative_BMCs}

Fix $p_0\in(1/2, 1)$.
To generate ``assortative'' \glspl{BMC}, we fix the diagonal values of $p_{i,i}=p_0$ for $i\in [K]$.
The off diagonal elements $p_{i,j}$ for $j\neq i$ are obtained by first sampling a vector $q=(q_j)_{j \in [K]\setminus\{i\}}$ uniformly at random from the $(K-1)$-dimensional simplex, and setting $p_{i,j}=(1-p_0)q_j$ for $j\neq i$.

\begin{table}[ht]
	\centering
	\setlength{\tabcolsep}{4pt}
	{
		\small
		\begin{tabular}{l l c c c c}
			\toprule
			{Ensemble $p$} & {Ensemble $\alpha$} & $K$ & $I(\alpha, p)$ & $\bar{\mathrm{H}}(\sigma)$ & $t_{\mix}$  \\
			\midrule
            Assortative & Const. $1/K$ & $5$ & $2.79\pm 0.03$ & $1 \pm 0$ & $8.8 \pm 0.1$ \\
            Assortative & Const. $1/K$ & $10$ & $3.40\pm 0.03$ & $1 \pm 0$ & $9.06 \pm 0.06$ \\
			\bottomrule
			\vspace{0pt}
		\end{tabular}
	}
	\caption{%
    Parameters characterizing different \gls{BMC} ensembles with assortative cluster transition matrices $p$.
    }
    \label{tab:ensemble_characterizations_assortative}
\end{table}

\subsection{Supporting numerical investigations}

\subsubsection{Visualization of data and clustering}

\Cref{fig:observation_matrix_visualization} visualizes the matrix $\hat{N}$ for the \gls{BMC} of \Cref{fig:Proof_of_concept__From_easy_to_hard}(c).
It highlights the difficulty of task of detecting the number of clusters.

\begin{figure}[hbtp]
    \centering
    \begin{subfigure}{0.48\linewidth}
        \includegraphics[width=\linewidth]{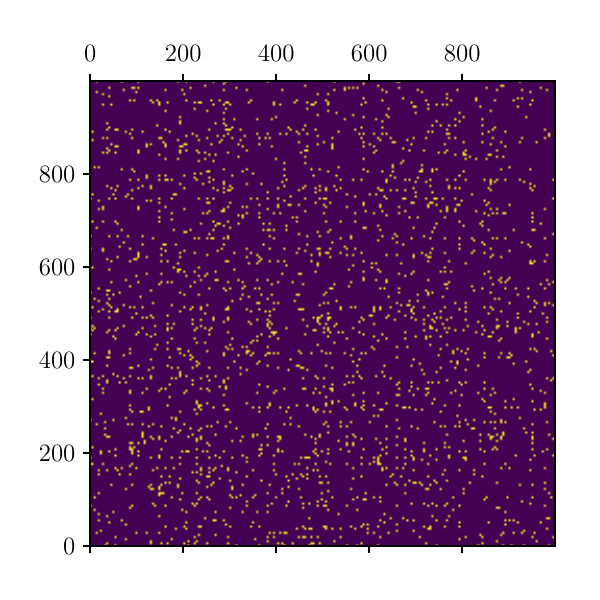}
        \caption{States sorted randomly}
        \label{fig:observation_matrix_visualization_random}
    \end{subfigure}%
    \hfill
    \begin{subfigure}{0.48\linewidth}
        \includegraphics[width=\linewidth]{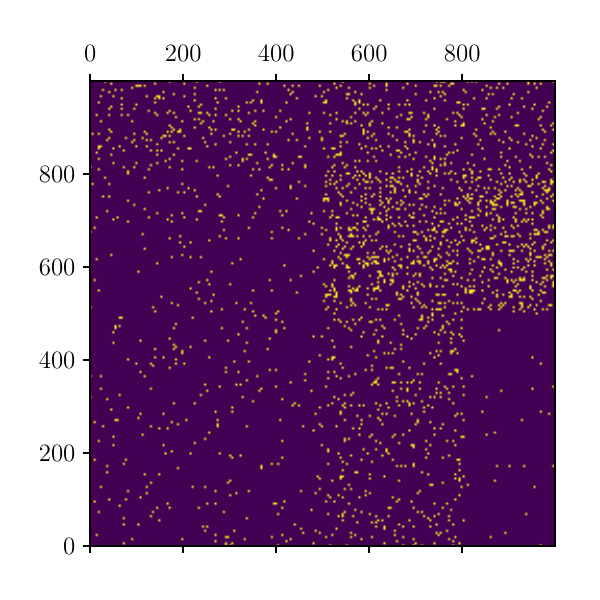}
        \caption{States sorted according to estimate clustering}
        \label{fig:observation_matrix_visualization_sorted}
    \end{subfigure}%
    \caption{%
    Heatmap of $\mathbbm{1}_{\hat{N}_{x,y}>0}$ for the \gls{BMC} from \Cref{fig:Proof_of_concept__From_easy_to_hard}(c) at $\ell_n/n=(\ln n)^2$ and $n=1000$.
    In this case, $\ell_n/n^2 \approx 0.0047$.
    }
    \label{fig:observation_matrix_visualization}
\end{figure}

\subsubsection{Visualization of the embedding \texorpdfstring{$X$}{X}}

\Cref{fig:embedding_visualization} visualizes the embedding $\hat{X}$ for a \gls{BMC} with parameters
\begin{equation}
    \label{eqn:pair_plot_BMC_def}
    p
    =
    \begin{pmatrix}
        0.1 & 0.3 & 0.6\\
        0.1 & 0.3 & 0.6\\
        0.8 & 0.1 & 0.1
    \end{pmatrix}
    ,
    \quad
    \alpha
    =
    \pi
    =
    (0.38 , 0.22 , 0.40)^\intercal
    .
\end{equation}
Observe that this \gls{BMC} has $\rank(p)=2<K$; which is why we created the embedding with $r=2$.
Note also that the clusters $1$ and $2$ have identical outgoing transition probabilities, and that clusters $2$ and $3$ have identical ingoing transition probabilities.
Consequently, we require both left and right singular vectors to distinguish all three clusters in the embedding.
This can be seen in the plot by considering the subplots corresponding to $(\tilde{U}_{\cdot,1},\tilde{U}_{\cdot,2})$, and $(\tilde{V}_{\cdot,1},\tilde{V}_{\cdot,2})$.

\input{images/TikZ__embedding_visualization.tex}

\subsubsection{Relative accuracy for different algorithm settings}
\label{sec:Relative_accuracy_vs_parameter_settings}

\Cref{alg:kpre,alg:kpost} have several parameters for which explicit choices have to be made when implementing it.
These include the radius threshold $h_n$, neighborhood size threshold $\rho_n$, and singular value threshold $\gamma_n$: recall \eqref{eqn:paramter_parameterization}.

We may wonder about the sensitivity of the algorithm to these settings, as well as whether certain choices of parameters are better than others.
In the following experiment, we therefore vary the exponents $a, b, c$ in \eqref{eqn:paramter_parameterization}.
We limit our investigation to ranges of $a, b, c$ for which correct operation of our algorithm is guaranteed in the limit $n \to \infty$.
The results are shown in \Cref{fig:Relative-accuracy-vs-singular-value-threshold-exponent}.

\input{images/TikZ__Relative_accuracy_vs_singular_value_threshold_exponent.tex}

\Cref{fig:Relative-accuracy-vs-singular-value-threshold-exponent}(a) suggests that any value $a \gtrsim 0.85$ works well for the random \glspl{BMC} considered in these experiments.
Fixing $a=0.9$, \Cref{fig:Relative-accuracy-vs-singular-value-threshold-exponent}(b) suggests a strictly decreasing relationship between the relative accuracy and the exponent $b$.
Fixing $b=0.1$, \Cref{fig:Relative-accuracy-vs-singular-value-threshold-exponent}(c) exhibits a clear peak of the relative accuracy as we vary the singular value threshold's exponent $c$ at around $0.75$.

\subsubsection{Effect of embedding dimension on cluster density}
\label{sec:cluster-density-vs-rank}

\Cref{fig:Cluster-density-vs-rank} shows a plot of the average intracluster distance $\bar{d}_{\text{intra}}$, and the average intercluster distance $\bar{d}_{\text{inter}}$, both as a function of the embedding rank $r$ for some \gls{BMC} sampled from the same ensemble as was used in \Cref{fig:Relative_accuracy_vs_rank_reversible}.
We define these distances as
\begin{equation}
    \bar{d}_{\textnormal{intra}}
    \eqdef
    \frac{1}{K}\sum_{k=1}^K\frac{1}{|\mcV_k|}\sum_{x\in\mcV_k}\lVert \hat{X}_{x,\cdot} - \mu_{k}\rVert_2
    ,
    \quad
    \bar{d}_{\textnormal{inter}}
    \eqdef
    \frac{1}{2K(K-1)}\sum_{k=1}^K\sum_{k'\neq k}\lVert\hat{\mu}_k - \hat{\mu}_{k'}\rVert_2
    ,
\end{equation}
where $\hat{\mu}_k\eqdef (1/|\mcV_k|)\sum_{x\in \mcV_k} \hat{X}_{x,\cdot}$ denote the cluster centroids, and $\lVert\cdot\rVert_2$ is the usual Euclidean norm.
Note that $\bar{d}_{\textnormal{inter}}$ is defined as half the average distance between cluster centroids.

Observe that for each $\beta$, the intra- and intercluster distances both increase as $r$ increases.
However, for $r \gtrsim K$, the $\bar{d}_{\textnormal{inter}}$ plateaus whereas $\bar{d}_{\textnormal{inter}}$ continues to increase.
Also, observe that $\bar{d}_{\textnormal{intra}}$ grows slower for $r \lesssim K$ than for $r \gtrsim K$.

\begin{figure}[htb]
    \centering
    \begin{subfigure}{0.32\linewidth}
        \centering
        \begin{tikzpicture}[scale = 0.95]
            \begin{axis}[
                    width=\linewidth,
                    xmin = 0.45, xmax = 21,
                    ymin = 0, ymax=2.5,
                    % ymin = -1.15, ymax = 0.25,
                    %	xmode = log, ymode = log,
                    xlabel = {Embedding parameter $r$},
                    ylabel = {Distance (a.u.)},
                    % label style = {font = \normalsize},
                    % tick label style = {font = \normalsize},
                ]

                \addplot[smooth, mark=o] plot coordinates {
                    (1, 0.2534382735006543)	(2, 0.4365734007092351)	(3, 0.5882336157952974)	(4, 0.7097731501734378)	(5, 0.8316902708103857)	(6, 0.9728298097940083)	(7, 1.1259456638948788)	(8, 1.267241209043378)	(9, 1.3974681068156856)	(10, 1.513321842492302)	(11, 1.6197880931701438)	(12, 1.7178795415669317)	(13, 1.8146706089010063)	(14, 1.9028425630054628)	(15, 1.9891803418263645)	(16, 2.0715439322303943)	(17, 2.149216759737471)	(18, 2.2264601311798202)	(19, 2.298771242431418)	(20, 2.368968555205761)	
                    };
                % \addlegendentry{$\bar{d}_{\textnormal{intra}}$}

                \addplot[smooth, mark=square] plot coordinates {
                    (1, 0.04558253365255963)	(2, 0.4724229163100582)	(3, 0.7051319397533161)	(4, 0.8630220512917886)	(5, 0.9556396380085502)	(6, 0.9961246318520337)	(7, 1.008068973600641)	(8, 1.0152991950812786)	(9, 1.0179430710440633)	(10, 1.0217374311010663)	(11, 1.0267484201453616)	(12, 1.0331515128317146)	(13, 1.0342802581549793)	(14, 1.0385928897494119)	(15, 1.0416156354858102)	(16, 1.0435602756227358)	(17, 1.0473215731560874)	(18, 1.0480766783851092)	(19, 1.0496281647149064)	(20, 1.051696677512466)	
                    };

                \path [draw=black, dashed]
                (axis cs:10,0)
                --(axis cs:10,2.5);
                \addplot[dashed, domain = 0:21, samples = 2] {1.24377};
            \end{axis}
        \end{tikzpicture}
        \caption{$\beta=2$}
    \end{subfigure}
    \begin{subfigure}{0.32\linewidth}
        \centering
        \begin{tikzpicture}[scale = 0.95]
            \begin{axis}[
                    width=\linewidth,
                    reverse legend,
                    xmin = 0.0, xmax = 21,
                    ymin=0, ymax = 8.5,
                    % ymin = -1.15, ymax = 0.25,
                    %	xmode = log, ymode = log,
                    xlabel = {Embedding parameter $r$},
                    % ylabel = {Relative accuracy},
                    % yticklabel = \empty,
                    % label style = {font = \normalsize},
                    % tick label style = {font = \normalsize},
                    legend style={at={(0.96, 0.4)}, anchor=south east}
                ]

                \addplot[smooth, mark=o] plot coordinates {
                    (1, 0.6537861038516688)	(2, 1.0479824788660554)	(3, 1.34463447794203)	(4, 1.588576606010708)	(5, 1.8129825311595298)	(6, 2.0145361475516586)	(7, 2.225224235624717)	(8, 2.453963672699128)	(9, 2.7707664702043617)	(10, 3.18099541137069)	(11, 3.550627716835095)	(12, 3.886236624856098)	(13, 4.18470551473881)	(14, 4.46262179341002)	(15, 4.726550695316772)	(16, 4.978634056549948)	(17, 5.2136309872892905)	(18, 5.4334315789506125)	(19, 5.646266231133554)	(20, 5.851950016472231)	
                    };
                % \addlegendentry{\scriptsize $K = 5$}

                \addplot[smooth, mark=square] plot coordinates {
                    (1, 0.12177997533222151)	(2, 3.004749001355959)	(3, 4.634195994349103)	(4, 5.621765293645451)	(5, 6.238266736893478)	(6, 6.606144587104451)	(7, 6.862694290281483)	(8, 7.003076010036935)	(9, 7.047163943263831)	(10, 7.051080225362808)	(11, 7.051849315481258)	(12, 7.052293627466049)	(13, 7.053233075759307)	(14, 7.054682163885493)	(15, 7.054973609994127)	(16, 7.055804686371432)	(17, 7.057469651760398)	(18, 7.059268438138523)	(19, 7.060510678548668)	(20, 7.060717937238043)	
                    };
                % \addlegendentry{\scriptsize $K = 10$}
                % \addlegendentry{\scriptsize $K = 20$}

                \path [draw=black, dashed]
                (axis cs:10,0)
                --(axis cs:10,8);
                \addplot[dashed, domain = 0:21, samples = 2] {7.8};
            \end{axis}
        \end{tikzpicture}
        \caption{$\beta=3$}
    \end{subfigure}
    \begin{subfigure}{0.32\linewidth}
        \centering
        \begin{tikzpicture}[scale = 0.95]
            \begin{axis}[
                    width=\linewidth,
                    reverse legend,
                    xmin = 0, xmax = 21,
                    ymin = 0, ymax = 55,
                    %	xmode = log, ymode = log,
                    xlabel = {Embedding parameter $r$},
                    % ylabel = {Relative accuracy},
                    % yticklabel = \empty,
                    % label style = {font = \normalsize},
                    % tick label style = {font = \normalsize},
                    legend style={at={(0.95, 0.4)}, anchor=south east}
                ]
                \path [draw=black, dashed]
                (axis cs:10,0)
                --(axis cs:10,50);
                \addplot[dashed, domain = 0:21, samples = 2] {48.9193};

                \addplot[smooth, mark=o] plot coordinates {
                    (1, 1.7030841544496433)	(2, 2.7118565248009205)	(3, 3.4374575458483485)	(4, 4.028011841808668)	(5, 4.569354064375862)	(6, 5.030784561520017)	(7, 5.458405099031516)	(8, 5.875193819508323)	(9, 6.309490055774749)	(10, 6.960727405290628)	(11, 8.104779331822723)	(12, 9.089249541147767)	(13, 9.992374788305323)	(14, 10.811439466462286)	(15, 11.559193246004089)	(16, 12.26026276572195)	(17, 12.924888565745203)	(18, 13.556485802408352)	(19, 14.169194887554568)	(20, 14.734863875206239)	
                    };
                \addlegendentry{$\bar{d}_{\textnormal{intra}}$}

                \addplot[smooth, mark=square] plot coordinates {
                    (1, 0.8709501899150046)	(2, 20.45381429891016)	(3, 31.51722904900621)	(4, 38.19862075890427)	(5, 42.393432622767676)	(6, 44.84072572440697)	(7, 46.61348414844114)	(8, 47.51729144594294)	(9, 47.870051960180454)	(10, 47.9578423463667)	(11, 47.957999467619125)	(12, 47.958188075464996)	(13, 47.958397099477004)	(14, 47.958599657312305)	(15, 47.96036599878089)	(16, 47.96047892121786)	(17, 47.96053558313451)	(18, 47.9605937774653)	(19, 47.96063725494504)	(20, 47.961030610697826)	
                    };
                \addlegendentry{$\bar{d}_{\textnormal{inter}}$}

            \end{axis}
        \end{tikzpicture}
        \caption{$\beta=4$}
    \end{subfigure}

    \caption{%
        The average intra- and intercluster distances as a function of the embedding dimension $r$ for some \gls{BMC} sampled from the reversible ensemble described in \Cref{sec:reversible_BMCs}.
        Here, $K=10$ and $n=1000$ are fixed, and the pathlength $\ell_n/n=(\ln n)^{\beta}$ is varied with $\beta\in\{2,3,4\}$.
        The horizontal dashed line indicates the radius threshold $h_n$, while the vertical dashed line indicates the true number of clusters $K$, which here equals the rank of $p$.
    }
    \label{fig:Cluster-density-vs-rank}
\end{figure}
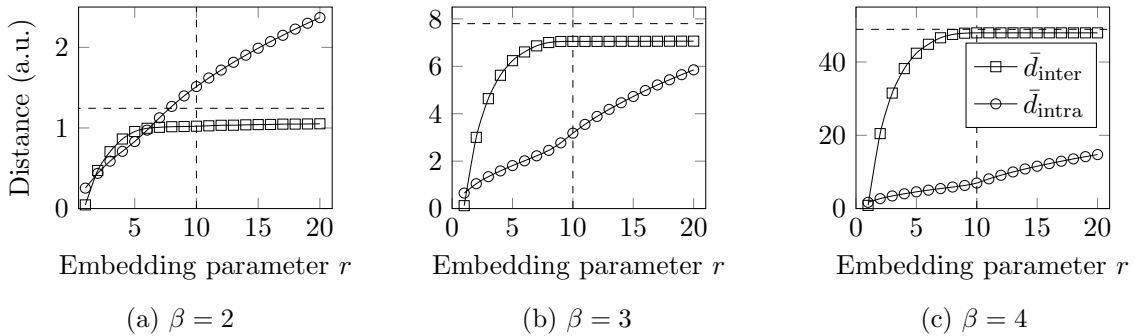

\subsubsection{Singular values in sparse regimes}\label{sec:singular-values-in-sparse-regimes}
\Cref{fig:Sparse-example-singular-values} shows how the singular values $\{ \sigma_i \}_{i=1, 2, 3}$ of $\hat{N}_{\Gamma}$ depend on the parameters $p_0$ and $\beta$ from \Cref{sec:sparse_example}.

\input{images/TikZ__Sparse_example_singular_values.tex}

Observe in \Cref{fig:Sparse-example-singular-values-vs-p0} that $\sigma_1$ and $\sigma_3$ are insensitive to $p_0$; while on the contrary, $\sigma_2$ depends approximately linearly on $p_0$.
A brief computation shows that the expectation $N=\mbbE[\hat{N}]$ has nonzero singular values $(n/\ell_n)\sigma_{1,2}(N)=(1,2p_0-1)$.
The proof of \Cref{cor:singvalscalingrankp} in \Cref{sec:proof_of_cor_singvalscalingrankp} furthermore shows that $(n/\ell_n)\sigma_{1,2}(\hat{N}_{\Gamma})$ converge in probability to $(n/\ell_n)\sigma_{1,2}(N)=(1,2p_0-1)$ as $n \rightarrow \infty$.
However, to understand this phenomena completely, a more detailed analysis would be needed to describe the concentration of $\sigma_{1,2}(\hat{N}_{\Gamma})$ at finite $n$.

This caveat is underscored by \Cref{fig:Sparse-example-singular-values-vs-beta,fig:Sparse-example-singular-values-vs-n}.
These figures demonstrate how $\sigma_{2}(\hat{N}_{\Gamma})$ deviates from the asymptotic scaling predicted by \eqref{eqn:singvalscaling}.
Specifically, since $\gamma_n=(\ell_n/n)^{0.95}=o(\ell_n/n)$, \eqref{eqn:singvalscaling} predicts that $\sigma_2(\hat{N}_{\Gamma})/\gamma_n=\omega_{\mbbP}(1)$; but on the contrary, observe that $\sigma_2(\hat{N}_{\Gamma})/\gamma_n$ is a decreasing function of $\ell_n$ and $n$.
The same is to a lesser extent true for $\sigma_1(\hat{N}_{\Gamma})$.

\subsubsection{Confidence intervals for \texorpdfstring{\Cref{tab:Performance-of-the-different-methods}}{Table 1}}
\label{sec:confidence-interval-table}

\Cref{tab:confidence-bounds-for-different-methods} reports the number of repetitions $N_{\text{rep}}$ and the width of the estimated $95\%$-condfidence interval for the sample means $\mu_{\hat{K}}$ reported in \Cref{tab:Performance-of-the-different-methods}.
Note that all but three confidence intervals are at most $0.15$.

\begin{table}[htpb]
    \centering
    \setlength{\tabcolsep}{5pt}
    {
        \small
        \begin{tabular}{c@{\hspace{3pt}}c@{\hspace{3pt}}c|c@{\hspace{3pt}}cc@{\hspace{3pt}}cc@{\hspace{3pt}}c|c@{\hspace{3pt}}cc@{\hspace{3pt}}cc@{\hspace{3pt}}c}
            \toprule
            $l_n$     & $K$ & {Test} & \multicolumn{2}{c}{ALG2} & \multicolumn{2}{c}{HDBS}          & \multicolumn{2}{c}{MEGH}      & \multicolumn{2}{c}{LLSC} & \multicolumn{2}{c}{LLCI} & \multicolumn{2}{c}{CAIC}\\
            &     &        & $N_{\text{rep}}$ & MoE & $N_{\text{rep}}$ & MoE & $N_{\text{rep}}$ & MoE & $N_{\text{rep}}$ & MoE & $N_{\text{rep}}$ & MoE & $N_{\text{rep}}$ & MoE\\
            \midrule

            \multirow{4}{*}{$n\ln(n)^2$} & \multirow{4}{*}{$3$}
              & 1 & $250$ & $ 0.00 $ & $250$ & $ 0.00 $ & $250$ & $ 0.04 $ & $700$ & $ 0.15 $ & $250$ & $ 0.00 $ & $350$ & $ 0.15 $ \\
            & & 2 & $300$ & $ 0.14 $ & $300$ & $ 0.14 $ & $250$ & $ 0.14 $ & $1550$ & $ 0.15 $ & $350$ & $ 0.15 $ & $350$ & $ 0.15 $ \\
            & & 3 & $250$ & $ 0.10 $ & $250$ & $ 0.07 $ & $250$ & $ 0.14 $ & $900$ & $ 0.14 $ & $450$ & $ 0.15 $ & $450$ & $ 0.15 $ \\
            & & 4 & $250$ & $ 0.14 $ & $250$ & $ 0.14 $ & $350$ & $ 0.14 $ & $750$ & $ 0.15 $ & $250$ & $ 0.12 $ & $350$ & $ 0.14 $ \\ [0.5em]

            \multirow{4}{*}{$n\ln(n)^2$} & \multirow{4}{*}{$6$}
              & 1 & $250$ & $ 0.04 $ & $250$ & $ 0.03 $ & $250$ & $ 0.15 $ & $4450$ & $ 0.15 $ & $850$ & $ 0.15 $ & $250$ & $ 0.14 $ \\
            & & 2 & $450$ & $ 0.15 $ & $400$ & $ 0.14 $ & $250$ & $ 0.08 $ & $2950$ & $ 0.15 $ & $900$ & $ 0.14 $ & $500$ & $ 0.15 $ \\
            & & 3 & $250$ & $ 0.13 $ & $250$ & $ 0.13 $ & $250$ & $ 0.13 $ & $1800$ & $ 0.15 $ & $1250$ & $ 0.15 $ & $500$ & $ 0.14 $ \\
            & & 4 & $300$ & $ 0.14 $ & $250$ & $ 0.13 $ & $250$ & $ 0.09 $ & $2500$ & $ 0.15 $ & $850$ & $ 0.15 $ & $500$ & $ 0.15 $ \\ [0.5em]

            \multirow{4}{*}{$n\ln(n)^2$} & \multirow{4}{*}{$10$}
              & 1 & $250$ & $ 0.03 $ & $250$ & $ 0.01 $ & $250$ & $ 0.07 $ & $1550$ & $ 0.15 $ & $1850$ & $ 0.15 $ & $250$ & $ 0.08 $ \\
            & & 2 & $250$ & $ 0.11 $ & $250$ & $ 0.07 $ & $250$ & $ 0.08 $ & $2500$ & $ 0.15 $ & $2550$ & $ 0.15 $ & $650$ & $ 0.14 $ \\
            & & 3 & $350$ & $ 0.14 $ & $300$ & $ 0.14 $ & $250$ & $ 0.14 $ & $3150$ & $ 0.15 $ & $1700$ & $ 0.15 $ & $950$ & $ 0.15 $ \\
            & & 4 & $250$ & $ 0.11 $ & $250$ & $ 0.10 $ & $250$ & $ 0.13 $ & $4000$ & $ 0.15 $ & $2500$ & $ 0.15 $ & $700$ & $ 0.15 $ \\ [0.5em]

            \midrule

            \multirow{4}{*}{$n^2$} & \multirow{4}{*}{$3$}
              & 1 & $250$ & $ 0.00 $ & $250$ & $ 0.01 $ & $250$ & $ 0.04 $ & $250$ & $ 0.00 $ & $250$ & $ 0.00 $ & $350$ & $ 0.14 $ \\
            & & 2 & $2600$ & $ 0.15 $ & $9850$ & $ 0.38^* $ & $250$ & $ 0.04 $ & $250$ & $ 0.10 $ & $250$ & $ 0.09 $ & $350$ & $ 0.15 $ \\
            & & 3 & $250$ & $ 0.07 $ & $9650$ & $ 0.47^* $ & $350$ & $ 0.14 $ & $250$ & $ 0.05 $ & $250$ & $ 0.01 $ & $450$ & $ 0.15 $ \\
            & & 4 & $2450$ & $ 0.15 $ & $9850$ & $ 0.91^* $ & $1750$ & $ 0.15 $ & $550$ & $ 0.15 $ & $250$ & $ 0.04 $ & $300$ & $ 0.14 $ \\ [0.5em]

            \multirow{4}{*}{$n^2$} & \multirow{4}{*}{$6$}
              & 1 & $250$ & $ 0.01 $ & $250$ & $ 0.01 $ & $250$ & $ 0.12 $ & $250$ & $ 0.01 $ & $250$ & $ 0.03 $ & $250$ & $ 0.10 $ \\
            & & 2 & $2450$ & $ 0.15 $ & $250$ & $ 0.04 $ & $250$ & $ 0.06 $ & $2000$ & $ 0.15 $ & $350$ & $ 0.14 $ & $450$ & $ 0.14 $ \\
            & & 3 & $250$ & $ 0.11 $ & $250$ & $ 0.01 $ & $250$ & $ 0.07 $ & $250$ & $ 0.01 $ & $250$ & $ 0.00 $ & $250$ & $ 0.10 $ \\
            & & 4 & $2250$ & $ 0.15 $ & $12000$ & $ 0.15 $ & $250$ & $ 0.08 $ & $2350$ & $ 0.15 $ & $400$ & $ 0.14 $ & $450$ & $ 0.15 $ \\ [0.5em]

            \multirow{4}{*}{$n^2$} & \multirow{4}{*}{$10$}
              & 1 & $250$ & $ 0.00 $ & $250$ & $ 0.01 $ & $250$ & $ 0.08 $ & $250$ & $ 0.13 $ & $350$ & $ 0.14 $ & $250$ & $ 0.08 $ \\
            & & 2 & $3000$ & $ 0.15 $ & $250$ & $ 0.06 $ & $250$ & $ 0.09 $ & $1150$ & $ 0.15 $ & $2650$ & $ 0.15 $ & $1200$ & $ 0.15 $ \\
            & & 3 & $250$ & $ 0.12 $ & $250$ & $ 0.00 $ & $250$ & $ 0.09 $ & $450$ & $ 0.15 $ & $250$ & $ 0.04 $ & $250$ & $ 0.05 $ \\
            & & 4 & $3150$ & $ 0.15 $ & $250$ & $ 0.11 $ & $250$ & $ 0.07 $ & $1600$ & $ 0.15 $ & $2450$ & $ 0.15 $ & $1350$ & $ 0.15 $ \\ [0.5em]
            \vspace{0pt}
        \end{tabular}
    }
    \caption{
        This table shows the margin of error (MoE) of the estimated $95\%$-confidence intervals for the sample means of the different estimators tested in \Cref{tab:Performance-of-the-different-methods}, as well as the number of replications $N_{\text{rep}}$ that were used in each case.
        }
    \label{tab:confidence-bounds-for-different-methods}
\end{table}

\section{Further algorithmic details}

\subsection{Singular value thresholding via Sylvester's law}
\label{sec:On-singular-value-thresholding-via-Sylvesters-law}

\Cref{alg:kpre} combines Sylvester's law of inertia with a shift property for eigenvalues of matrices to compute the number of singular values of $\hat{N}_\Gamma$ that lie above the threshold $\gamma_n$.

The argument as to why this works is as follows.
Recall first that for $A \in \mathbb{C}^{n \times n}$, $\gamma \in \mathbb{C}$, if $\sigma^2$ is an eigenvalue of $A^\intercal A$, then $\sigma^2 - \gamma^2$ is an eigenvalue of $A^\intercal A - \gamma^2 I_{n \times n}$.
Recall second that Sylvester's law of inertia says that for any nonsingular matrix $S$, the inertia of a symmetric matrix such as $B := \hat{N}_\Gamma^\intercal \hat{N}_\Gamma - \gamma_n^2 I_{n \times n}$ equals the inertia of $C := S B S^\intercal$.
Inertia, here, is defined as the triple of number of positive, zero, and negative eigenvalues.
Note finally that the number of positive eigenvalues of $C$ coincides with the number of singular values of $\hat{N}_\Gamma$ that are larger than $\gamma_n$.
Now, all of this is especially useful to us because the $LDL^\intercal$-decomposition happens to be such a candidate transformation.
Moreover, the inertia of the diagonal matrix in the $LDL^\intercal$-decomposition is easy to determine (one just needs to look at the signs of $D$'s diagonal entries), and the $LDL^\intercal$-decomposition is computationally easier than e.g.\ computing a full \gls{SVD}.

\subsection{Alternative algorithms, applied off--the--shelve, to try and estimate \texorpdfstring{$K$}{K}}
\label{sec:alternative_methods}

This section present details on alternative methods that may be used to try and detect the number of clusters $K$ of a \gls{BMC}.
These were studied numerically in \cref{sec:results_alternative_methods}.
Contrary to \cref{alg:kpre,alg:kpost}, these algorithms do not come with consistency guarantees when applied to a trajectory of a \gls{BMC}.

\subsubsection{Via HDBSCAN}

Given the embedding $\hat{X}$, with dimension set by \Cref{alg:kpre}, we also use the popular HDBSCAN algorithm to detect the number of clusters via a density-based clustering of the embedding $\hat{X}$.
For our simulations we use the scikit-learn package \cite{scikit-learn}, whose implementation of HDBSCAN is based on \cite{8215642}.
HDBSCAN is an extension of DBSCAN which we briefly describe here.

DBSCAN takes two parameterrs which are analogous to the parameters used in \Cref{alg:kpost}.
In particular, DBSCAN first identifies states with dense neighborhoods by calculating the distance to each state's $\rho_n$th nearest neighbor.
If this distance is less than $h_n$, this state is identified as a core point.
All non-core points are labeled as noise.
To highlight the connection with our method, the set of core points can equivalently be defined as
\begin{equation}
    \mcC(\rho_n,h_n)\eqdef \{x\in[n]: |\mcN_x|>\rho_n\},
\end{equation}
where $\mcN_x$ is defined as in \eqref{eqn:nbrhooddef}.
A graph is then constructed with vertices $\mcC(\rho_n,h_n)$, and edges between all pairs $x,y\in\mcC(\rho_n,h_n)$ such that $\lVert \hat{X}_{x,\cdot}-\hat{X}_{y,\cdot}\rVert_2 < h_n$.
Clusters are defined as the connected components of this graph.

HDBSCAN extends DBSCAN by performing a sweep over all possible values of $h_n$, and identifying clusters which are stable across the largest range of values of $h_n$.
Such a sweep allows to identify regions of different density, corresponding to clusters which are stable across different ranges of $h_n$.
In our experiments, the remaining parameter $\rho_n$, corresponding to the minimum cluster size, is set to the same value as in \Cref{alg:kpost} (see \eqref{eqn:paramter_parameterization}).
We refer to \cite{8215642} for further details, as well as a statistical interpretation of this method.

\begin{algorithm}[H]
    \caption{Pseudocode for HDBSCAN}\label{alg:hdbscan}
    \KwData{A transition count matrix $\hat{N}=(\hat{N}_{x,y})_{x,y\in[n]}$, and neighborhood size threshold $\rho_n$}
    \KwResult{An estimated number of clusters $\hat{K}_{\mathrm{HDB}}$}
    Calculate \gls{SVD} $U\Sigma V^\intercal$ of $\hat{N}_{\Gamma}$;\\
    Construct spectral embedding $\hat{X}\gets [\tilde{U},\tilde{V}]$;\\
    Estimate $\hat{K}_{\mathrm{HDB}}$ using HDBSCAN with data matrix $\hat{X}$ and parameter $\rho_n$;\\
    \Return $\hat{K}_{\mathrm{HDB}}$
\end{algorithm}

\subsubsection{Via the maximum eigengap}
\label{sec:max_eigengap}

The next method likewise computes eigenvalues of a matrix associated to $\hat{N}$.
A common heuristic for estimating the number of clusters in data is to choose $\hat{K}$ as the location of the maximum gap between successive eigenvalues of some appropriate similarity matrix of the data points.
We implement this idea following the approach of \cite{BudelGabriel2020}.

Define the ingoing and outgoing degrees of state $x\in[n]$ as $d_x^{\mathrm{in}}\eqdef\sum_{y\in[n]}\hat{N}_{y,x}$ and $d_x^{\mathrm{out}}\eqdef \sum_{y\in[n]}\hat{N}_{x,y}$, corresponding to the number of transitions going into and out of state $x$, respectively.
Note that $|d_x^{\mathrm{in}}-d_x^{\mathrm{out}}|\leq 1$ because each transition into a state also has to leave it, except for the first and last time steps.
Then, define the modularity matrix $M$ elementwise as
\begin{equation}
    \hat{M}_{x,y}\eqdef\hat{N}_{x,y}-\frac{d_x^{\mathrm{in}} d_y^{\mathrm{out}}}{\ell}.
\end{equation}
Intuitively, the elements of $\hat{M}_{x,y}$ compare the observed transitions between $x$ and $y$ to the expected number of transitions if each transition was sampled indpendently with the probability of observing a transition given by $d_x^{\mathrm{in}}d_y^{\mathrm{out}}/\ell^2$.
The latter gives an expected in-degree equal to $d_x^{\mathrm{in}}$ and an expected out-degree equal to $d_x^{\mathrm{out}}$.

After $\hat{M}$ has been calculated we can define the method to estimate the number of clusters in a BMC.
Let $\lambda_1(\hat{M}),\allowbreak\lambda_2(\hat{M}),\allowbreak\ldots,\allowbreak\lambda_n(\hat{M})\allowbreak\in\mbb{C}$ be the eigenvalues of $\hat{M}$ such that $|\lambda_i(\hat{M})|\geq|\lambda_{i+1}(\hat{M})|$ for all $i\in\{1,\ldots,n-1\}$.
Then we define the estimated number of clusters $\hat{K}_M$ as
\begin{equation}
    \hat{K}_M=\argmax_{i=2,\ldots,n}\{|\lambda_{i-1}(\hat{M})|-|\lambda_{i}(\hat{M})|\}
\end{equation}
\begin{algorithm}[H]
    \caption{Pseudocode for maximum eigengap}\label{alg:max_eigengap}
    \KwData{A transition count matrix $\hat{N}=(\hat{N}_{x,y})_{x,y\in[n]}$}
    \KwResult{An estimate number of clusters $\hat{K}_M$}
    Compute in- and out-degrees $d_x^{\mathrm{in}}\gets \sum_{y\in[n]}\hat{N}_{y,x}$, $d_x^{\mathrm{out}}\gets \sum_{y\in[n]}\hat{N}_{x,y}$ for $x\in[n]$;\\
    Build modularity matrix $\hat{M}_{x,y}\gets \hat{N}_{x,y}-d_x^{\mathrm{in}}d_x^{\mathrm{out}}/\ell_n$ for $x,y\in[n]$;\\
    Calculate sorted eigenvalues $|\lambda_1(\hat{M})|,|\lambda_2(\hat{M})|,\ldots,|\lambda_n(\hat{M})|$;\\
    Calculate maximum eigengap $\hat{K}_M\gets \argmax_{i=2,\ldots,n}\{|\lambda_{i-1}(\hat{M})-|\lambda_i(\hat{M})|\}$;\\
    \Return $\hat{K}_M$
\end{algorithm}

\subsubsection{Via a crossvalidated log-likelihood}
\label{sec:loglikelihood}

Stepping away from spectral methods, a natural estimator for the number of clusters is to use a crossvalidated log-likelihood.
We consider \glspl{BMC} with different numbers of clusters as competing model classes and try to select the model which best explains the data.
The procedure is as follows.

First, split the observed sequence $X_0,\ldots,X_{\ell}$ into two parts, $X_{\text{train}}\eqdef X_0,\ldots,X_{\floor{\ell/2}}$, and $X_{\text{val}}\eqdef X_{\floor{\ell/2}+1},\ldots,X_{\ell}$.
For each $k=1,2,\ldots,K_{\max}$, where $K_{\max}\in\mbbN_+$ is some threshold value, we use the training data $X_{\text{train}}$ to estimate a model \gls{BMC} for the data with $k$ clusters.
Specifically, we use the \gls{BMC} clustering algorithm of \cite{ClusterBMC2017} to estimate a cluster assignment.
Let $\hat{Y}^{(k)}_t\eqdef \hat{\sigma}^{(k)}(X_t)$ for $t=0,\ldots,\ell$ denote the sequence of clusters under the estimated cluster assignment $\hat{\sigma}^{(k)}$.
We subsequently compute the maximum-likelihood estimator $\hat{p}^{(k)}=\{\hat{p}^{(k)}_{i,j}\}_{i,j \in [k]}$ of the cluster transition probabilities for the given cluster assignment $\hat{\sigma}^{(k)}$ as
\begin{equation}
    \hat{p}^{(k)}_{i,j}\eqdef\frac{\sum_{t=0}^{\floor{\ell/2}-1}\mathbbm{1}\{\hat{Y}^{(k)}_t=i, \hat{Y}^{(k)}_{t+1}=j\}}{\sum_{t=0}^{\floor{\ell/2}-1}\mathbbm{1}\{\hat{Y}^{(k)}_t=i\}}.
\end{equation}
Recall that we consider two estimators for $\sigma$, namely the spectral clustering algorithm \cite[Algorithm 1]{ClusterBMC2017}, and the refined estimate obtained from \cite[Algorithm 2]{ClusterBMC2017}.
The improvement algorithm in particular uses a local optimization of a log-likelihood function, starting from the initial estimate provided by \cite[Algorithm 1]{ClusterBMC2017}.
The combined model $(\hat{\sigma}^{(k)},\hat{p}^{(k)})$ is therefore an approximation of the maximum-likelihood estimator among all \glspl{BMC} with $k$ clusters.

To determine the appropriate value of $k=1,\ldots,K_{\max}$ we estimate for each candidate model $(\hat{\sigma}^{(k)},\hat{p}^{(k)})$ the maximum log-likelihood using the validation data $X_{\text{val}}$ as follows
\begin{equation}
    \hat{\mcL}^{(k)}\eqdef \frac{1}{\floor{\ell/2}}\sum_{t=\floor{\ell/2}+1}^{\ell-1}\ln \frac{\hat{p}^{(k)}_{\hat{Y}^{(k)}_t,\hat{Y}^{(k)}_{t+1}}}{|\hat{\mcV}^{(k)}_{\hat{Y}^{(k)}_{t+1}}|},
\end{equation}
Here, $\hat{\mcV}^{(k)}_{i}\eqdef \lb \hat{\sigma}^{(k)}\rb^{-1}(i)\subset[n]$ denotes the $i$'th cluster according to the estimate assignment $\hat{\sigma}^{(k)}$.
We then define the estimator $\hat{K}_{\mcL}$ the value of $k$ that maximizes log-likelihood ratio:
\begin{equation}
    \hat{K}_{\mcL}\eqdef\argmax_{k=1,\ldots,K_{\max}}\hat{\mcL}^{(k)}.
\end{equation}
By evaluating $\hat{\mcL}^{(k)}$ on a (nearly) independent set of validation data, we limit the bias of $\hat{K}_{\mcL}$, and reduce overfitting.

\begin{algorithm}[H]
    \caption{Pseudocode for crossvalidated log-likelihood}\label{alg:loglikelihood}
    \KwData{Training data $X_{\mathrm{train}}=X_0,\ldots,X_{\floor{\ell/2}}$, validation data $X_{\mathrm{val}}=X_{\floor{\ell/2}+1},\ldots,X_{\ell}$, maximum number of clusters $K_{\max}$}
    \KwResult{Estimated number of clusters $\hat{K}_{\mcL}$}
    \For{$k\in 1,\ldots,K_{\max}$}{
        Estimate \gls{BMC} $(\hat{\sigma}^{(k)},\hat{p}^{(k)})$ by running \cite[Algorithms 1 and 2]{ClusterBMC2017} with $X_{\mathrm{train}}$;\\
        Estimate cluster sequence $\hat{Y}_t^{(k)}\gets \hat{\sigma}^{(k)}(X_t)$ for $t=\floor{\ell/2}+1,\ldots,\ell$;\\
        Estimate log-likelihood $\hat{\mcL}^{(k)}\gets (1/\floor{\ell/2})\sum_{t=\floor{\ell/2}+1}^{\ell-1}\ln\bigl( \hat{p}^{(k)}_{\hat{Y}^{(k)}_t,\hat{Y}^{(k)}_{t+1}}/|\hat{\mcV}^{(k)}_{\hat{Y}^{(k)}_{t+1}}|\bigr)$;\\
    }
    Select maximum log-likelihood $\hat{K}_{\mcL}\gets \argmax_{k=1,\ldots,K_{\max}}\hat{\mcL}^{(k)}$;\\
    \Return $\hat{K}_{\mcL}$
\end{algorithm}

\subsubsection{Via a \texorpdfstring{\gls{CAIC}}{CAIC}}
\label{sec:AIC}

The holdout method used to minimize the bias of $\hat{K}_{\mcL}$ means that we can only use half of our data to cluster.
Information criterion methods attempt to avoid this issue by adding a correction term to the objective function, which penalizes model complexity and reduces overfitting.
Here we use the \gls{CAIC} \cite{bozdogan1987model}.
Let $(\hat{\sigma}^{(k)},\hat{p}^{(k)})$ be as in \cref{sec:loglikelihood}, except that we now use all of the data $X_0,\ldots,X_{\ell}$.
Let also $\hat{Y}_t^{(k)}=\hat{\sigma}^{(k)}(X_t)$.
Then, the \gls{CAIC} is defined as
\begin{equation}
    \text{CAIC}(k)\eqdef -2\sum_{t=0}^{\ell-1}\ln \frac{\hat{p}^{(k)}_{\hat{Y}_t^{(k)},\hat{Y}_{t+1}^{(k)}}}{|\hat{\mcV}^{(k)}_{\hat{Y}_{t+1}^{(k)}}|} + \text{DF}(k)(\ln {\ell} - 1),
\end{equation}
where $\text{DF}(k)$ denotes the number of degrees of freedom of the model with $k$ clusters.
Here, $\text{DF}(k) = DF(\sigma) + \text{DF}(p)=n + k(k-1)$, with $\text{DF}(\sigma)$ the degrees of freedom of the cluster assignment $\sigma$, and $\text{DF}(p)$ the degrees of freedom of the cluster transition matrix $p$.
Note that $p$ has only $k(k-1)$ degrees of freedom, because each row is constrained to sum to one.
We finally define $\hat{K}_{\mathrm{CAIC}}$ as the value that minimizes the CAIC
\begin{equation}
    \hat{K}_{\mathrm{CAIC}}\eqdef \argmin_{k=1,\ldots,K_{\max}}\mathrm{CAIC}(k).
\end{equation}
\begin{algorithm}[H]
    \caption{Pseudocode for CAIC}\label{alg:caic}
    \KwData{A trajectory $X=X_0,X_1,\ldots,X_{\ell}$, and a maximum number of clusters $K_{\max}$}
    \KwResult{An estimate number of clusters $\hat{K}_{\mcL}$}
    \For{$k\in 1,\ldots,K_{\max}$}{
        Estimate \gls{BMC} $(\hat{\sigma}^{(k)},\hat{p}^{(k)})$ by running \cite[Algorithms 1 and 2]{ClusterBMC2017} with $X$;\\
        Estimate cluster sequence $\hat{Y}_t^{(k)}\gets \hat{\sigma}^{(k)}(X_t)$ for $t=0,\ldots,\ell$;\\
        Calculate $\mathrm{CAIC}(k)\gets -2\sum_{t=0}^{\ell-1}\ln \bigl(\hat{p}^{(k)}_{\hat{Y}_t^{(k)},\hat{Y}_{t+1}^{(k)}}/|\hat{\mcV}^{(k)}_{\hat{Y}_{t+1}^{(k)}}|\bigr) + \text{DF}(k)(\ln {\ell} - 1)$;\\
    }
    Select minimum \gls{CAIC} $\hat{K}_{\mathrm{CAIC}}\gets \argmin_{k=1,\ldots,K_{\max}}\mathrm{CAIC}(k)$;\\
    \Return $\hat{K}_{\mcL}$
\end{algorithm}

\end{document}